\documentclass[twoside,11pt, x11names, table]{article}

\usepackage{mathrsfs,amsmath,amsfonts,amssymb}
\usepackage{bm,bbm,dsfont,pifont,amscd,stmaryrd,euscript,amsthm,appendix,color,epsfig,xr}
\usepackage[round]{natbib}
\usepackage{authblk}
\usepackage{graphicx}
\usepackage{hyperref}
\usepackage[utf8]{inputenc} 
\usepackage{text comp} 
\usepackage{mathtools} 
\usepackage[inline]{enumitem} 
\usepackage{todonotes}
\usepackage[cal = esstix, scr = boondoxo, scrscaled = 1]{mathalfa} 
\usepackage[acronym, shortcuts, toc, sanitize={symbol=false}, nowarn,hyperfirst=false]{glossaries}
\usepackage{glossary-longragged}
\usepackage[multiple]{footmisc} 
\usepackage{multirow}
\usepackage{tabu}
\usepackage{booktabs}

\usepackage{tikz}
\usetikzlibrary{patterns}

\usepackage[all]{xy} 
\usepackage{undertilde} 
\usepackage{bbm} 

\usepackage{needspace} 
\newtheorem{theorem}{Theorem}
\newtheorem{lemma}[theorem]{Lemma}
\newtheorem{corollary}[theorem]{Corollary}
\newtheorem{proposition}[theorem]{Proposition}
\newtheorem{definition}[theorem]{Definition}

\theoremstyle{remark}
\newtheorem{remark}[theorem]{Remark}
\newtheorem{example}{Example}

\newcommand{\K}{k}
\newcommand{\Kc}{k_c}

\newcommand{\emb}{\Phi}
\newcommand{\embK}{\emb_{\K}}

\newcommand{\embC}{\emb_{\Kc}}

\newcommand{\f}{f}

\DeclareMathOperator{\lin}{\mathrm{span}}

\DeclareMathOperator{\supp}{supp}

\DeclareMathOperator{\Dp}{\ensuremath{\partial_{\it p}}}
\DeclareMathOperator{\sinc}{sinc}

\newcommand{\inputS}{\mathcal{X}}
\newcommand{\real}{\mathbb{R}}
\newcommand{\complex}{\mathbb{C}}
\newcommand{\nat}{\mathbb{N}}

\newcommand{\HH}{\mathcal{H}}
\newcommand{\HKpre}{\HH_{\K}^{\mathrm{\scriptscriptstyle{pre}}}}

\newcommand{\HK}{\HH_{\K}}

\newcommand{\Hop}{\HH_{\partial^{(0,p)} \K}}
\newcommand{\Hpp}{\HH_{\partial^{(p,p)} \K}}
\newcommand{\HC}{\HH_{\Kc}^{\scriptscriptstyle{0}}}
\newcommand{\HCpre}{\HH_{\Kc}^{\scriptscriptstyle{0, pre}}}

\newcommand{\F}{\mathcal{F}}
\newcommand{\Func}{\complex^\inputS}
\newcommand{\Lp}[1]{\mathrm{L}^{#1}}
\newcommand{\Cont}[2]{\mathscr{C}^{#1}_{#2}}

\newcommand{\Sobo}[2]{\mathcal{W}^{\,#1}_{#2}}

\newcommand{\M}{\mathcal{M}}

\newcommand{\Mf}{\M_{\! f}}
\newcommand{\Mfs}{\M_{\delta}}
\newcommand{\Mc}{\M_{c}}
\newcommand{\Mr}{\M_{r}}
\newcommand{\Mprob}{\mathcal{P}}
\newcommand{\Pc}{\mathcal{P}_{c}}

\newcommand{\D}{\mathcal{D}}
\newcommand{\Dcomp}[1]{\mathscr{E}^{#1}}
\newcommand{\DLone}[1]{\mathcal{D}^{#1}_{\scriptscriptstyle{L^1}}}

\newcommand{\Dall}[1]{\mathcal{D}^{#1}}

\newcommand{\B}{\mathcal{B}}
\renewcommand{\S}{\mathcal{S}}
\newcommand{\narrow}{\sigma}

\newcommand{\topo}{\mathfrak{T}}

\newcommand{\One}{\mathbbm{1}}

\newcommand{\diff}{\mathop{} \! d} 
\newcommand{\ipdK}[2]{\ensuremath{\left \langle #1 \, | \, #2 \right \rangle_{\K}}}
\newcommand{\ipdC}[2]{\ensuremath{\left \langle #1 \, | \, #2 \right \rangle_{\Kc}}}

\newcommand{\ipd}[2]{\ensuremath{\left \langle #1 \, | \, #2 \right \rangle}}

\newcommand{\normK}[1]{\ensuremath{\left \| #1 \right \|_{\K}}}
\newcommand{\normC}[1]{\ensuremath{\left \| #1 \right \|_{\Kc}}}

\newcommand{\norm}[1]{\ensuremath{\left \| #1 \right \|}}

\newcommand{\Four}[1]{\ensuremath{\mathop{\mathscr{F} #1}}}

\newcommand{\csubset}{\ensuremath{\hookrightarrow}}
\newcommand{\csupset}{\ensuremath{\hookleftarrow}}
\newcommand{\cupset}{\rotatebox[origin=c]{90}{$\csubset$}}
\newcommand{\cdownset}{\rotatebox[origin=c]{-90}{$\csubset$}}
\newcommand{\vequal}{\rotatebox[origin=c]{90}{$=$}}

\providecommand{\function}[5]{} 
\renewcommand{\function}[5]{
	\ensuremath{
	\mathchoice{
	\ifthenelse{\equal{#1}{}}
            	{
            	\begin{array}[t]{ccl}
            		\ifthenelse{{\equal{#2}{}}}{
            		#4 & \longmapsto & #5}{
            		#2 & \longrightarrow & #3
            		 \ifthenelse{\equal{#4}{}} {} {\\
            		#4 & \longmapsto & #5}}
            	\end{array} \!
            	}
            	{
            	\begin{array}[t]{lccl}
            	#1 : 
            		\ifthenelse{{\equal{#2}{}}}{
            		& #4 & \longmapsto & #5}{
            		& #2 & \longrightarrow & #3
            		 \ifthenelse{\equal{#4}{}} {} {\\
            		 & #4 & \longmapsto & #5}}
            		 \end{array} \!
            	}
	}
	{
		\ifthenelse{\equal{#1}{}}
		{
			\ifthenelse{{\equal{#2}{}}}{
			{#4 \mapsto \, #5}}{
			{#2 \rightarrow \, #3}
			 \ifthenelse{\equal{#4}{}} {} {
			, \; {#4 \mapsto \, #5}}}
		}
		{
		#1 : 
			\ifthenelse{{\equal{#2}{}}}{
			{#4 \mapsto  \, #5}}{
			{#2 \rightarrow \, #3}
			 \ifthenelse{\equal{#4}{}} {} {
			, \; {#4 \mapsto \, #5}}}
		}	
	}
{}{}}}

\setkeys{glslink}{hyper=false}

\renewcommand*{\CustomAcronymFields}{%
  name={\the\glsshorttok},
  description={\the\glslongtok},
  first={\noexpand{\the\glslongtok}\space(\the\glsshorttok)},%
  firstplural={\noexpand{\the\glslongtok\noexpand\acrpluralsuffix}\space(\the\glsshorttok)},%
  text={\the\glsshorttok},%
  plural={\the\glsshorttok\noexpand\acrpluralsuffix}%
}

\SetCustomStyle

\newacronym{ispd}{$\int$s.p.d.\@}{integrally strictly positive definite}
\newacronym{spd}{s.p.d.\@}{strictly positive definite}
\newacronym{cpd}{c.p.d.\@}{conditionally positive definite}
\newacronym{iff}{iff}{if and only if}
\newacronym{wrt}{w.r.t.\@}{with respect to}
\newacronym{lcv}{loc.\ cv.\@}{locally convex}

\def\iftodo{\iffalse} 
\def\ifcomments{\iftrue} 
\def\iflong{\iffalse} 
\def\ifshort{\iftrue} 

\newcommand{\Todo}[2][]{
	\iftodo
		\ifthenelse{\equal{#1}{}}{\todo{#2}}{\todo[#1]{#2}}
	\fi
}

\newcommand{\acks}[1]{
	\subsection*{Acknowledgements}
	#1
	}

\bibpunct{(}{)}{;}{a}{,}{,}

\title{Kernel Distribution Embeddings:\\
Universal Kernels, Characteristic Kernels and Kernel Metrics on Distributions}

  \author[1]{Carl-Johann {Simon-Gabriel}\thanks{cjsimon@tuebingen.mpg.de}}
  \author[1]{Bernhard Sch\"olkopf\thanks{bs@tuebingen.mpg.de}}
  \affil[1]{Max Planck Institute for Intelligent Systems\\ Department of Empirical Inference\\
Spemanstra{\ss}e 38, 72076 T{\"u}bingen\\ Germany}

\date{}

\begin{document}

\maketitle

\begin{abstract}


	Kernel mean embeddings have recently attracted the attention of the machine learning community. They map measures $\mu$ from some set $\M$ to functions in a reproducing kernel Hilbert space (RKHS) with kernel $\K$.
	The RKHS distance of two mapped measures is a semi-metric $d_\K$ over $\M$. 
	We study three questions.
	\begin{enumerate}[label = (\Roman*), noitemsep, leftmargin = 8em]
		\item For a given kernel, what sets $\M$ can be embedded? \label{MQ1}
		\item When is the embedding injective over $\M$ (in which case $d_\K$ is a metric)? \label{MQ2}
		\item How does the $d_\K$-induced topology compare to other topologies on $\M$? \label{MQ3}
	\end{enumerate}
	The existing machine learning literature has addressed these questions in cases where $\M$ is (a subset of) the finite regular Borel measures. We unify, improve and generalise those results. 
	Our approach naturally leads to continuous and possibly even injective embeddings of (\mbox{\emph{Schwartz-})}\emph{distributions}, i.e., generalised measures, but the reader is free to focus on measures only. 
	In particular, we systemise and extend various (partly known) equivalences between different notions of universal, characteristic and strictly positive definite kernels, 
	and show that on an underlying locally compact Hausdorff space, $d_\K$ metrises the weak convergence of probability measures if and only if $\K$ is continuous and characteristic. 

\end{abstract}

\setlength{\parskip}{4pt}

    
\setenumerate{label=(\roman*), noitemsep}
\setitemize{label=$\triangleright$, noitemsep}
\setdescription{labelindent=1em,leftmargin=4em, style=nextline}



\textcolor{red}{The following preprint is an older and longer version of the
paper with same title published in JMLR \citep{simon18kernel}. The JMLR version
is completely restructured and hopefully easier to follow. We leave this older
version online for citation consistency and some additional content, but we
advise to start with the JMLR version and refer to this older version only if
needed. Also, please note that the proofs of Theorems
\ref{theo:GeneralExistence} and \ref{theo:GeneralExistenceCompact} are flawed
and have not been fixed.}

\section{Introduction}
	During the past two decades, reproducing kernel Hilbert spaces (RKHS) have risen to a major tool in various areas of machine learning.

	They offer a variety of nice function spaces for dimensionality-reduction, regression and classification algorithms, such as kernel PCA, kernel regression and SVMs. For those algorithms, the function space should be sufficiently large to approximate well the unknown target function, and its norm sufficiently strong to avoid overfitting. These two requirements lead to the question how a given RKHS $\HK$ compares with other well-known function spaces~$\F$. The question is two-fold. When is $\HK$ \emph{continuously} contained in $\F$? And when is it dense in $\F$? During the past two decades, the machine learning community focused on the second, more difficult question. Kernels $\function{\K}{\inputS \times \inputS}{\complex}{}{}$ whose RKHS is dense in $\F$ were called \emph{universal}, with slight variations depending on $\F$. The latter was usually a space of continuous functions. But we will try to ``narrow the gap'' between $\HK$ and $\F$ by choosing a smaller space $\F$, containing for example only continuously differentiable functions; with far-reaching consequences for kernel mean embeddings (KME).

	These KMEs only recently caught the attention of the machine learning community. They embed finite Borel measures into $\HK$ via the map $\function{\embK}{}{}{\mu}{\int_\inputS \K(.,x) \diff \mu(x)}$. This gives access to the RKHS's computational handiness to handle measures numerically, 
and led to new homogeneity \citep{gretton07}, distribution comparison \citep{gretton07,gretton12} and (conditional) independence tests \citep{gretton05a, gretton08, fukumizu08, gretton10, lopez-paz13}. For those algorithms, it is interesting when $\embK$ is defined and injective over large sets $\M$, in which case the kernel is said \emph{characteristic} over $\M$. It turned out that many usual kernels are \emph{characteristic} over (subsets of) finite Borel measures \citep{micchelli06, fukumizu08,fukumizu09a,fukumizu09b, sriperumbudur08, sriperumbudur10a, sriperumbudur11}. Nevertheless, KMEs are in general not surjective, leaving ``a gap'' between $\M$ and $\HK$ that we will try to narrow by extending $\embK$ to generalised measures, i.e.\ Schwartz-distributions (thereafter simply called \emph{distributions}, as opposed to \emph{measures}). 

	Narrowing the gaps between $\HK$ and $\F$ and between $\M$ and $\HK$ may seem like unrelated problems, but they are actually linked by duality. If $\HK$ continuously embeds into $\F$, Corollary~\ref{cor:EmbedDuals} will show that the KME is well-defined over the dual $\F'$ and Theorem~\ref{theo:UniversalMeansCharacteristic}  that the kernel is universal over $\F$ if and only if (iff) it is characteristic to $\F'$. For wisely chosen spaces $\F$, $\F'$ identifies with a space of measures; for others, it becomes a space of distributions. Corollary~\ref{cor:EmbedDuals} and Theorem~\ref{theo:UniversalMeansCharacteristic} thus finally unify and systemise various known links between different notions of universal and characteristic kernels catalogued in an overview paper by \citet{sriperumbudur11}; and they give a straightforward method to extend KMEs from measures to distributions! Proposition~\ref{prop:SPD} complements these results by noticing that kernels are characteristic to a given vector space iff they are \gls{spd} over this space. This systemises the various connections (also reviewed by \citealp{sriperumbudur11}) between universal and characteristic kernels and different notions of \gls{spd} kernels, such as conditionally \gls{spd} or integrally \gls{spd} kernels. Table~\ref{tab:UCS} summarises these connections. 

	Using (integrally) strict positive definiteness, \citet{sriperumbudur10a} gave handy conditions to check whether a stationary (also known as translation-invariant) kernel is characteristic. We improve and extend these results to distributions. Surprisingly, it turns out that any smooth and stationary kernel that is characteristic to finite measures with compact support is also characteristic to the larger space of compactly supported distributions (Proposition~\ref{prop:CompactCharacteristicRd})!

	The RKHS distance between two mapped measures (or distributions) $\mu$ and $\nu$ defines a semi-metric $d_\K$ over $\M$: $d_\K(\mu,\nu) := \normK{\embK(\mu) - \embK(\nu)}$. It is only natural to compare it with other topologies on $\M$. Table~\ref{tab:TopoSummary} summarises these comparisons. They have important consequences. For one thing, they guarantee continuity of KMEs when $\M$ is a dual equipped with its strong dual topology, the most common topology on duals (Proposition~\ref{prop:ContinuousEmbedding}). For another, they prove that our extension of KMEs to distributions with compact support is \emph{the only} continuous linear extension of usual KMEs (Proposition~\ref{prop:UniqueExtension}). And most importantly, they finalise a series of recent results \citep{sriperumbudur10b, sriperumbudur13} by showing that on locally compact Hausdorff spaces, $d_\K$ metrises the narrow\footnote{The narrow convergence is also known as the weak or weak-* convergence in probability theory. Except for the abstract, we call it narrow to distinguish it from the weak (dual) topology from functional analysis.} convergence of probability measures iff $\K$ is continuous and characteristic to probability measures (Theorem~\ref{theo:MetrizationOfNarrow2}). Metrics metrising the narrow convergence are of prime importance for convergence results in probability theory. Many such metrics exist, for example the Dudley, Lévy-Prohorov and Wasserstein (or Kantorovich) metrics. But  $d_\K$ has numerous advantages, particularly for applications. First, it underlies various machine learning algorithms, including almost all applications of KMEs, as well as kernel independent component analysis \citep{bach02, gretton05b, shen09} and kernel based dimensionality reduction for supervised learning \citep{fukumizu04}. Second, it is easy to estimate, even with finite samples only, because $\normK{\embK(\mu)}^2 = \iint \K(x,y) \diff \mu(x) \diff \bar \mu(y)$ (which also holds for distributions!). And third, it depends only on the kernel and can thus be defined over arbitrary input domains~$\inputS$. For further details, see introduction of \citet{sriperumbudur10b}.
	
	Overall, our main contribution is to unify and finalise many results on KMEs that were discovered and used by the machine learning community during the past decade. Contrary to probability measures, which often need extra care and are given special attention in this paper, embeddings of distributions are rather a byproduct of our systematic unification. For generality we phrase many results in terms of distributions, but readers may focus on measures only, if they systematically: 
\begin{itemize}
	\item replace the word ``distribution'' by ``measure'';
	\item set the parameter $m$ to $0$ (see notations in Section~\ref{sec:Definitions});
	\item remember that $\DLone{0} = \Mf$ and $\Dcomp{0} = \Mc$.
\end{itemize}
	
	Section~\ref{sec:Differentiation} however will lose most of its substance, as it focuses on distributional derivatives: the essence of distribution theory. 
As an appetiser however, let us show how these derivatives may appear and relate to KMEs. Consider an input space $\inputS = \real$, and let $\vec \varphi_\K(x) := \embK(\delta_x) = \K(.,x)$. When $\K$ is continuously differentiable, then $\vec \varphi_\K$, as a function of $x$ with values in $\HK$, is differentiable and its derivative $\partial \vec \varphi_\K$ belongs to $\HK$ \citep[Lemma~4.34]{steinwart08}. It is thus tempting to write
\begin{equation}\label{eq:DiffIntro}
	[\partial \vec \varphi_\K](x) := \lim_{h \rightarrow 0} \embK \left(\frac{\delta_{x+h} - \delta_x}{h}\right) \overset{?}{=} \embK \left(\lim_{h \rightarrow 0} \frac{\delta_{x+h} - \delta_x}{h}\right) = \embK (- \partial \delta_x) = - \embK(\partial \delta_x) \, , 
\end{equation}
where $\partial \delta_x := \lim_{h \rightarrow 0} \frac{\delta_{x-h} - \delta_x}{h}$ would denote the distributional derivative of $\delta_x$. 
Thus, the function $\partial \vec \varphi_\K$ does not seem to be the image of a measure, but the image of the first-order Schwartz-distribution $- \partial \delta_x$, which in physics is called a dipole. Equation~\eqref{eq:DiffIntro} will be made precise in Section~\ref{sec:Differentiation}.

	The structure of this paper roughly follows questions~\ref{MQ1}--\ref{MQ3} of the abstract. After fixing definitions and notations in Section~\ref{sec:Definitions}, we define the KME of distributions in Section~\ref{sec:Duality}. This gives answers to~\ref{MQ1} and immediately yields Theorem~\ref{theo:UniversalMeansCharacteristic} which links universal and characteristic kernels. To answer~\ref{MQ2}, we first need basic calculus rules for embedded distributions, together with some results specific to KMEs of probability measures. Those are covered by Sections~\ref{sec:Calculus} and \ref{sec:ProbabilityMeasures} respectively. Section~\ref{sec:UniversalKernels} answers~\ref{MQ2} by giving necessary and sufficient conditions for kernels to be characteristic to some distribution spaces. In particular, it shows that kernels which injectively embed large spaces of distributions do exist and provides examples. Section~\ref{sec:Topology} addresses~\ref{MQ3} by focusing on the induced kernel semi-metric $d_\K$, with a special emphasis on the metrisation of the narrow convergence of finite measures. Section~\ref{sec:Conclusion} gives a brief overview of related work and concludes.

\section{Definitions and Notations\label{sec:Definitions}}

%
Most of our notations are fairly standard. An informed reader might skip to Section~\ref{sec:Duality} and refer to Section~\ref{sec:Definitions} only as needed. 
 
\paragraph{Input and output space.} Let $\inputS$ be the input set of all considered kernels and functions. Whenever referring to differentiable functions or to distributions of order $\geq 1$, we will \emph{implicitly} take $\inputS = \Omega \subset \real^d$, where d is a strictly positive integer, $\real$ the set of real numbers, and $\Omega$ an open subset of $\real^d$. Otherwise, when referring to functions which need not be differentiable or to distributions of order $0$ (i.e.\ measures), $\inputS$ will simply be a \emph{locally compact paracompact Hausdorff} set. Note that a Hausdorff set is paracompact \gls{iff} it admits partitions of unity subordinate to any open cover. For many results, the paracompactness assumption will be superfluous, but it is handy to keep it for ease of discussions. We will explicitly lift this assumptions when we think it is worth it. $\inputS$ will be equipped with its Borel $\sigma$-algebra.  All considered functions, measures and distributions will take their values in $\complex$, the complex numbers.
	
\paragraph{Kernel.} In this paper, a \emph{kernel} $\function{\K}{\inputS \times \inputS}{\complex}{}{}$ will be a positive definite function, meaning that for all $n \in \nat$, all $\lambda_1, \ldots, \lambda_n \in \complex$, and all $x_1, x_2, \ldots x_n \in \inputS$, 
\[
	\sum_{i,j=1}^{n} \lambda_i \K(x_i,x_j) \overline{\lambda_j} \geq 0.
\]
	
\paragraph{Differentiation.} Unless stated otherwise, $m$ will always designate a positive, possibly infinite integer: $m \in \nat \cup \{ \infty \}$. Let $p = (p_1, p_2, \ldots, p_d) \in \nat^d$, where $\nat$ is the set of non-negative integers. We note $|p| := \sum_{i = 1}^{d} p_i$. For a given function $\function{f}{\inputS}{\complex}{}{}$, we define $\partial^p f := \frac{\partial^{|p|} f}{\partial^p_1 x_1 \partial^p_2 x_2 \cdots \partial^p_d x_d}$, whenever the right-hand-side is well defined. $f$ will be said $m$-times continuously differentiable, if for any $p$ such that $|p| = m$, $\partial^p f$ exists and is continuous. Similarly, for kernels $\function{\K}{\Omega \times \Omega}{\complex}{}{}$, we will write $\partial^{(p,q)} \K$, where $p, q \in \nat^d$. A kernel will be said $(m,m)$-times continuously differentiable if, for all $p$ such that $|p| = m$, $\partial^{(p, p)} \K$ exists and is continuous. In particular, $\partial^{(p,q)} \K$ then exists and is continuous for any $|p|,|q| \leq m$. The same notations will be used for the distributional derivative.

\paragraph{Topological subsets.} Let $\mathcal S_1$ and $\mathcal S_2$ be two topological sets such that $\mathcal S_1 \subset \mathcal S_2$. If the canonical embedding of $\mathcal S_1$ into $\mathcal S_2$ is continuous, $\mathcal S_1$ is said to \emph{embed continuously into}  $\mathcal S_2$ or to \emph{be continuously contained in} $\mathcal S_2$. In that case, we write
\begin{equation*}
	\mathcal S_1 \csubset \mathcal S_2 \qquad \text{or} \qquad \mathcal S_2 \csupset \mathcal S_1 \, .
\end{equation*}
In other words, $\mathcal S_1 \csubset \mathcal S_2$ means: $\mathcal S_1$ is contained in $\mathcal S_2$ and carries a stronger  topology than the relative topology induced by $\mathcal S_2$. 

\paragraph{Spaces of functions.} The letter $\F$ will always designate a \gls{lcv} topological vector space (TVS) of functions. For (a few) reminders on \gls{lcv} TVSs, see Appendix~\ref{sec:Reminders}. \emph{All function spaces encountered in this paper are \gls{lcv} TVSs.}  Let $m$ be a possibly infinite, non-negative integer, and let $q \in \real$, with $1 \leq q \leq \infty$. Whenever it is defined, we note $\norm{f}_{\infty} := \sup_{x \in \inputS} |f(x)|$. We will consider the following spaces of functions over $\inputS$. Except $\Cont{m}{b}$, they will always be equipped with their usual, natural topology. The reader need not know all these topologies to understand the paper. They are stated here solely for completeness and rigour.
\begin{description}
	\item[$\Func$] the space of \emph{all} functions from $\inputS$ to $\complex$, equipped with the pointwise convergence topology.
	\item[$\Cont{m}{}$] the space of m-times continuously differentiable functions. It is equipped with the topology of uniform convergence on compact subsets of the functions and of their derivatives of order $\leq m$. This topology is generated by the family of semi-norms ${\{\norm{.}_{p,K} \, | \, |p| \leq m, \, K \subset \inputS, \, K \ \mathrm{compact} \}}$, where $\norm{f}_{p, K} := \max_{K} \partial^p |f|$.
	\item[$(\Cont{m}{b})_c$] the space of m-times continuously differentiable functions $f$, such that all their derivatives up to order $m$ be bounded. The natural topology would be the uniform convergence of the functions and of all their derivatives up to order $m$. This topology, however, will be too strong for our purposes. Instead, we will equip $\Cont{m}{b}$ with a weaker (LCv) topology (see Section~\ref{sec:MprobAndUniversality}) ---which we mark by the index $c$ in $(\Cont{m}{b})_c$--- such that the dual of $(\Cont{m}{b})_c$ be the same as the dual of $\Cont{m}{0}$ (see next item).
	\item[$\Cont{m}{0}$] the space of m-times continuously differentiable functions that vanish at infinity, as well as all their derivatives of order $\leq m$. Note that a function $f \in \Func$ is said to \emph{vanish at infinity}, if, for any $\epsilon>0$, there exists a compact $K \subset \inputS$ such that $|f(\inputS \backslash K)| \leq \epsilon$.  The space $\Cont{m}{0}$ is equipped with the topology of uniform convergence of the functions and of their derivatives of order $\leq m$. This topology is generated by the family of semi-norms ${\{\norm{.}_{p} \, | \, |p| \leq m\}}$, where $\norm{f}_{p} := \max \, \partial^p |f|$.
	\item[$\Cont{m}{c}$] the space of m-times continuously differentiable functions with \emph{compact} support. It will be equipped with its usual limit Fréchet topology \citep[see][]{schwartzTD, treves67}. In this topology, a sequence $(f_n)_n$ of functions converges to $f$ iff
	\begin{enumerate}
		\item there exists a compact $K \subset \inputS$ such that $f_n, f \in \Cont{m}{}(K)$ and 
		\item $f_n \rightarrow f$ in $\Cont{m}{}(K)$.
	\end{enumerate}
	\item[$\Lp{q}$] with its usual $\Lp{q}$-norm $\norm{.}_{L^q}$.
	\item[$\Sobo{m,q}{}:= \{ f \in \Lp{q}  \, | \, \forall \, |p| \leq m, \partial^p f \in \Lp{q}\}$] Here, $\partial^p f$ is the distributional derivative of $f$. The topology on these so-called \emph{Sobolev spaces} is generated by the family of semi-norms $\function{}{}{}{f}{\norm{\partial^p f}_{L^q}}$.
	\item[$\Sobo{m,q}{0}$] the closure of $\Cont{\infty}{c}$ in $\Sobo{m,q}{}$. When $q < \infty$ and $m=0$, or when $q < \infty$ and $\inputS = \real^d$, then it can be shown that $\Sobo{m,q}{0} = \Sobo{m,q}{}$. This does not hold for any arbitrary $\inputS = \Omega$.
\end{description}
Whenever $m=0$, we may drop the superscript $m$. Also, we may write $\Cont{}{0}(\real^d)$ if we want to specify that the input space $\inputS$ is specifically $\real^d$. We will write $\Cont{m}{*}$ if we want to designate either $\Cont{m}{}$ or $\Cont{m}{0}$ without specifying which one of them.
Note that $\Func \csupset \Cont{m}{} \csupset (\Cont{m}{b})_c \csupset \Cont{m}{0} \csupset \Cont{m}{c}$.

We will write $\K \in \Cont{(m,m)}{}$ (resp.\ $\K \in \Cont{(m,m)}{0}$) to express that $\K$ is $(m,m)$-times continuously differentiable (resp.\ $(m,m)$-times continuously differentiable and, for all $|p| \leq m$ and $x \in \inputS$, $\partial^{(p,p)} \K(.,x) \in \Cont{}{0}$ and $\sup_{x \in \inputS} \partial^{(p,p)} \K(x,x) < \infty$).

\paragraph{Spaces of measures.}  We will only consider \emph{regular Borel measures}. This won't be repeated in the sequel. The letter $\M$ will designate a generic set of measures. We will encounter the following spaces (or sets) of (signed) measures.
\begin{description}
	\item[$\Mfs$] the space of measures with \emph{finite} support: ${\Mfs := \lin \{\delta_x \, | \, x \in \inputS \}}$.
	\item[$\Mc$] the finite measures with \emph{compact} support.
	\item[$\Mf$] the \emph{finite} measures.
	\item[$\Mprob$] the set of probability measures. It is the set of \emph{positive} measures in $\Mr$ that sum to~$1$.
	\item[$\Mr$] the space of \emph{locally finite} measures, also known as \emph{Radon measures}. By locally finite, we mean that, on every compact $K \subset \inputS$, the measure takes finite values.
\end{description}
We will also briefly encounter the following subsets of $\Mf$: $\M_{\scriptscriptstyle{+}}$, the finite positive measures; $\Mf^{\scriptscriptstyle{\leq c}}$, the finite measures with total variation $\leq c$; $\M_{\scriptscriptstyle{+}}^{\scriptscriptstyle{\leq c}} := \M_{\scriptscriptstyle +} \cap  \Mf^{\scriptscriptstyle{\leq c}}$; and $\Mprob_c := \Mc \cap \Mprob$.  Note that $\Mfs \subset \Mc \subset \Mf \subset \Mr$. Given a measure $\mu \in \Mr$, we note $|\mu|$ its absolute value. In particular, $|\mu|(\inputS)$ equals its (possibly infinite) total variation. Given a set $\F$ of $\mu$-integrable functions, for $f \in \F$, we will write $\mu(f) := \int f \diff \mu$. In particular, note that $\mu$ defines a linear form over $\F$. This leads us to

\paragraph{Dual spaces and spaces of measures.} 
	The \emph{(topological) dual} $\F'$ of a TVS $\F$ is the space of $\emph{continuous}$ linear forms over $\F$.
	On locally compact Hausdorff spaces $\inputS$, some spaces of measures $\M$ are known to identify, algebraically and topologically,\footnote{when $\F'$ is equipped with its strong dual topology} with the dual $\F'$ of specific function spaces $\F$ via the map $\function{}{\M}{\F'}{\mu}{(\function{}{}{}{f}{\mu(f)})}$ (see Riesz-Markov-Kakutani Representation Theorem in Appendix~\ref{sec:Reminders}). In particular:  $\Mfs = (\Func)'$, $\Mc = \Cont{\prime}{}$, $\Mf = \Cont{\prime}{0}$ and $\Mr = \Cont{\prime}{c}$. When $\inputS = \real^d$, these are special cases of the following

\paragraph{Spaces of distributions.} The letter $\D$ will always be used as a generic to designate a set of distributions. We now list the respective duals of the previously defined spaces of functions.\footnote{Note that our notations differ from those of L.\ Schwartz in that we omit the additional prime he puts to every space of distribution. For him, $\D$ (without prime) would rather designate a space of functions.} 
\begin{description}
	\item[$\Mfs$] which can be shown to be the dual of $\Func$.
	\item[$\Dcomp{m}$] the space of distributions of order $m$ with compact support. It can be shown that, for $m \leq \infty$: $\Dcomp{m} = \{\sum_{|p| \leq n} \partial^p \mu_p \, | \, n < m+1, \, \mu_p \in \Mc \}$. In particular, $\Dcomp{0} = \Mc$.
	\item[$\DLone{m}$] the space of \emph{summable} or \emph{integrable} distributions of order $m$.  It can be shown that, for $m \leq \infty$: $\DLone{m} = \{\sum_{|p| \leq n} \partial^p \mu_p \, | \, n < m+1, \, \mu_p \in \Mf \}$. In particular, $\DLone{0} = \Mf$.
	\item[$\DLone{m}$] same space as before.
	\item[$\Dall{m}$] the space of Schwartz-distributions of order $m$. We will refer to $\Dall{\infty}$ simply as \emph{the space of distributions}. It can be shown that, for $m < \infty$: $\Dall{m} = \{\sum_{|p| \leq m} \partial^p \mu_p \, | \, \mu_p \in \Mr \}$. In particular, $\Dall{0} = \Mr$.
	\item[$\Lp{q'}$] where for $1 \leq q < \infty$, $\frac{1}{q'} + \frac{1}{q} =1$. Note that the dual of $\Lp{1}$ is $\Lp{\infty}$, but the converse is \emph{not} true.
	\item[/]  We will not consider the duals of $\Sobo{m,q}{}$.
	\item[$\Sobo{-m, q'}{0} := (\Sobo{m,q}{0})'$ ($q < \infty$)] For a characterisation of these spaces of distributions, see \citet[Theorem~3.10]{adams75}.
\end{description}
We will call any element of one of these dual spaces a \emph{distribution}. Note that they can all be seen as an element of $\Dall{\infty}$. In particular:  $\Mfs \subset \Dcomp{m} \subset \Sobo{-m-1, q'}{0} \subset \DLone{m} \subset \Dall{m} \subset \Dall{\infty}$ and $\Lp{q'} \subset \Sobo{-m,q}{0} \subset \Dall{\infty}$. These inclusions generalise those noted for measures. A diagram summarising all inclusions of functions and of their duals can be found in Appendix~\ref{sec:InclusionsDiagram}.

\paragraph{Topologies on duals.}
	When $\mathcal A$ is a set of linear forms over a TVS $\F$, we note $b(\mathcal A, \F)$ or $b(\mathcal \F', \F) \cap A$ (resp.\ $w(\mathcal A, \F)$ or $w(\mathcal \F', \F) \cap A$) the topology of bounded (resp.\ pointwise) convergence over $\F$. We will call $b(\F', \F)$ (resp.\ $w(\F', \F)$) the \emph{strong dual} (resp.\ \emph{weak dual}, or \emph{weak-star}) topology on $\F'$. When $\F = \Cont{}{b}$, we will call $w(\mathcal A, \Cont{}{b})$ the \emph{narrow topology} over $\mathcal A$ and use the letter $\sigma$ rather than $w$: $\sigma(\mathcal A, \Cont{}{b}) := w(\mathcal A, \Cont{}{b})$. Finally, we call $w(\mathcal A, \Cont{}{c})$ the \emph{vague topology} over $\mathcal A$. We will write $(\F')_b$ (resp.\ $(\F')_w$, $(\F')_\sigma$, $(\F')_\topo$) to specify that $\F'$ carries its strong dual topology (resp.\ its weak dual topology, its narrow topology, a given topology $\topo$). By default, $\F'$ will be equipped with its strong topology. When $\F$ is a Banach space, the strong topology of $\F'$ is the topology induced by the dual norm. 
A \emph{topological subspace of distributions} ---or topological distribution space--- is any subspace of $\Dall{\infty}$ that embeds continuously into $\Dall{\infty}$.
%
 If a LCv TVS $\F_1$ is densely and continuously contained in another one $\F_2$, then $\F'_2$ is continuously in $\F'_1$ (see Section~\ref{sec:WwSs}). In particular, all dual spaces listed above are topological distribution spaces: they verify the following continuous inclusions.
\begin{equation*}
	\Mfs \csubset \Dcomp{m} \csubset \Sobo{-m-1, q'}{0} \csubset \DLone{m} \csubset \Dall{m} \csubset \Dall{\infty} \quad \mathrm{and} \quad
	\Lp{q'} \csubset \Sobo{-m,q}{0} \csubset \Dall{\infty} \, .
\end{equation*}
On $\DLone{m}$, we will only consider the strong and weak topologies $b(\DLone{m}, \Cont{m}{0})$ and $w(\DLone{m}, \Cont{m}{0})$.
\iflong
Here are a few remarks on the topologies of the following spaces:
\begin{description}
	\item[$\Mf$] The strong topology is induced by the total variation norm. Thus, $\mu_n \longrightarrow \mu$ strongly \gls{iff} ${\int \diff |\mu_n-\mu| \longrightarrow 0}$.
	\item[$\Mc$] Neither weak nor strong topologies are induced by a norm.  $\mu_n \longrightarrow \mu$ strongly (resp.\ weakly) iff
		\begin{enumerate}
			\item There exists a compact $K \subset \inputS$ such that, for all $n$, $\supp \mu_n \subset K$.
			\item $\int \diff |\mu_n-\mu| \longrightarrow 0$ (resp.\ $\mu_n \longrightarrow \mu$ weakly in $\Mf$).
		\end{enumerate}
	\item[$\Dcomp{m}$] Again, $D_n \longrightarrow D$ strongly (resp.\ weakly) in $\Dcomp{m}$ iff
		\begin{enumerate}
			\item There exists a compact $K \subset \inputS$ such that, for all $n$, $\supp D_n \subset K$.
			\item $D_n \longrightarrow D$ strongly (resp.\ weakly) in $\Dall{m}$ or, equivalently, in $\DLone{m}$.
		\end{enumerate}
		In particular, note the coherence with the precedent inclusion: ${\Dcomp{m} \csubset \DLone{m}}$ and ${(\Dcomp{m})_w \csubset (\DLone{m})_w}$.
	\item[$\Dcomp{\infty}, \Dall{\infty}$] Being so-called \emph{Montel spaces}, they enjoy a property very useful in applications: \emph{a sequence converges strongly \gls{iff} it converges weakly}. This does not hold for finite $m$.
	\item [$\HK$] Identifying $\HK$ with its dual space $\HK'$ via anti-linear the Riesz representer map for Hilbert spaces (see Appendix~\ref{sec:Reminders}), one may  define a strong and weak dual topology on $\HK$. The strong is then simply the one induced by the usual RKHS norm. As for the weak topology, \emph{a sequence $f_n \in \HK$ converges weakly to $f \in \HK$ iff, for all $g \in \HK$, $\ipdK{g}{f} \longrightarrow \ipdK{g}{f}$}.
\end{description}

	Finally, let us now introduce another topology that we will use over finite Borel measure and in particular probability measures: the \emph{narrow (convergence) topology}.

	As the names suggests, the weak topology is weaker than the strong one. In particular, $D_n \overset{b}{\longrightarrow} D$ implies $D_n \overset{w}{\longrightarrow} D$. As a matter of fact, the weak topology is sometimes inadequately weak. For instance, in $\Mf$ (and in any $\DLone{m}$), $\delta_n$ converges weakly to $0$, although it is one and the same measure that is being translated. This does not happen in $\Mf$'s strong topology. But then, it turns out the strong topology is often inadequately strong. For instance, $\delta_{x_n}$ converges to $\delta_x$ in total variation norm \gls{iff} $x_n = x$ for all $n \in \nat$, except possibly for a finite number of them. Thus in Probability Theory, on $\Mprob$ (or more generally on $\Mf$) one often prefers the so-called \emph{narrow convergence} or \emph{convergence in distribution} topology, noted $\narrow$. A sequence $\mu_n \in \Mf$ converges narrowly to $\mu \in \Mf$, written $\mu_n \overset{\narrow}{\longrightarrow} \mu$, if $\mu_n(f) \longrightarrow \mu(f)$ for any $f \in \Cont{}{b}$, the set of bounded continuous functions. This topology both prevents mass from disappearing in infinity and enjoys the handiness of the weak topology in applications. The downside however, is that, unless the input space $\inputS$ is compact, the narrow convergence is not a general dual topology\footnote{in the sense that it is not defined for any dual space}. Thus many general functional analysis theorems won't apply directly, and the proofs have to be tuned specifically. That is why in Section~\ref{sec:Topology}, the analysis of the narrow convergence case has a whole subsection on its own!
\fi

\paragraph{Function and distribution related notations.} In accordance with the notation ${\mu(f) := \int f \diff \mu}$, we may also designate $D(f)$ by
\begin{equation*}
	\int f \diff D \quad \mathrm{or} \quad \int f(x) \diff D(x) \quad \mathrm{or}  \quad D_x(f)  \quad \mathrm{or} \quad  D_x(f(\hat x)) \, .
\end{equation*}
Note that, given a function $f$ (or functional $D$), it will sometimes be handy to specify the name of its input variable. To do so, we will write $f(\hat{x})$ (or $D(\hat{f})$), meaning $f$ (or $D$), where $x$ (or $f$) is the dummy input variable. In particular $\K(.,\hat{x}) = (\function{}{}{}{x}{\K(.,x)})$. Finally, we note $\bar f$ the complex conjugate function $\overline{f(\hat x)}$ of $f$.

\paragraph{Supports.}
	The support $\supp \varphi$ of a function $\varphi \in \Func$ is the closure of the set of points $x \in \inputS$ such that $\varphi(x) \not = 0$. The support $\supp D$ of a distribution $D$ is the greatest closed set $S \subset \inputS$, such that for any $\varphi \in \Cont{\infty}{c}$ with support in $\inputS \backslash S$, $D(\varphi) = 0$. When identifying function $\varphi$ with the distribution $\function{}{}{}{f}{\int f \varphi}$, the two definitions coincide. \\


Let us now turn to some definitions concerning kernels.

\section{Characteristicness: a Dual Counterpart to Universality\label{sec:Duality}}

In this section, we show how to embed general distribution spaces into an RKHS. The link between universal and characteristic kernels will then be straightforward. But before we go into detail, let us briefly sketch the idea of a distribution embedding.

Like any other Hilbert space, an RKHS $\HK$ identifies with its conjugate dual $\overline \HK'$ via the \emph{Riesz representation} map $\function{}{\HK}{\overline \HK'}{f}{\ipdK{\hat \varphi}{f}}$ (see Appendix~\ref{sec:Reminders})\footnote{In this paper, the inner products are chosen linear on the left, and anti-linear on the right.}. For the KME $\int \K(.,x) \diff \mu(x)$ of a measure $\mu$ (and under suitable assumptions on the kernel $\K$ and the measure $\mu$), this associated linear form is simply $\bar \mu \big|_{\HK}$, because
\begin{equation*}
	\ipdK{\hat \varphi}{\int \K(.,x) \diff \mu(x)} = \int \ipdK{\hat \varphi}{\K(.,x)} \diff \bar \mu(x) = \bar \mu(\hat \varphi) \, .
\end{equation*}
Thus, whenever a measure $\mu$ defines a continuous linear form over $\HK$, its KME is the Riesz representer of $\bar \mu \big|_{\HK}$. But when is $\mu$ continuous over $\HK$? It is, as soon as $\mu$ belongs to a dual $\F'$ such that $\HK$ embeds continuously into $\F$. 
This holds not only for measures, but for all elements of $\F'$. And depending on the choice of $\F$, those elements can be distributions that are not measures.

Now the details. We start with reminders on function spaces and universal kernels, continue with the definition of distribution embeddings and characteristic kernels, and finish with the natural link between universal and characteristic kernels.



\subsection{Universal Kernels}

The literature distinguishes various kinds of universal kernels $\K$, such as $c$-, $cc-$ or $c_0$-universal kernels. They are all special cases of the following unifying definition.

\begin{definition}[Universal Kernels]
	Let $\F$ be a \gls{lcv} TVS and $\K$ a kernel such that $\HK \csubset \F$.
	$\K$ is said \emph{universal over $\F$} if $\HK$ is dense in $\F$.
	When $\F$  equals $\Cont{m}{0}$ (resp.\ $\Cont{m}{0}$), $\K$ will be said $c^m$- (resp.\ \mbox{$c^m_0$-)} universal. When $m=0$, we may drop the superscript $m$.
\end{definition}

\begin{remark} 
	$c$-universality usually refers to the special case where $\inputS$ is compact, the general case being known as $cc$-universality \citep{sriperumbudur10a, sriperumbudur10b}. We find this an unnatural distinction and stick to $c$-universal for both. 
\end{remark}
In the literature, the hypothesis that $\HK$ be not only contained, but \emph{continuously} contained in $\F$ is usually replaced by (necessary and) sufficient smoothness assumptions on the kernel, so that this inclusion holds. For example (see proofs in Appendices~\ref{proof:HKinC} and \ref{proof:HKinCm}):

\begin{proposition}[Characterisation of $\HK \csubset \Cont{}{\star}$\label{prop:ContinuousRKHS}] \label{prop:HKinC}
	$\HK \subset \Cont{}{0}$ (resp.\ $\HK \subset \Cont{}{b}$, resp.\ $\HK \subset \Cont{}{}$) iff the two following conditions hold.
	\begin{enumerate}[label = (\roman*)]
		\item For all $x \in \inputS$, $K(.,x) \in \Cont{}{0}$ (resp.\ $K(.,x) \in \Cont{}{b}$, resp.\ $K(.,x) \in \Cont{}{}$). \label{p:1}
		\item  $\function{}{}{}{x}{\K(x,x)}$ is bounded (resp.\ bounded, resp.\ \emph{locally} bounded, meaning that, for each $y \in \inputS$, there exists a (compact) neighbourhood of $y$ on which $\function{}{}{}{x}{\K(x,x)}$ is bounded.). \label{p:2}
	\end{enumerate}
	If so, then $\HK \csubset \Cont{}{0}$ (resp.\ $\HK \csubset \Cont{}{b}$\,, thus $\HK \csubset (\Cont{}{b})_c$\,, resp.\ $\HK \csubset \Cont{}{}$).
\end{proposition}

\Needspace{2\baselineskip}
\begin{proposition}[Sufficient condition for $\HK \csubset \Cont{m}{\star}$ \label{prop:HKinCm}]
	\leavevmode
	\begin{enumerate}[label = , nosep]

	\item If $\K \in \Cont{(m,m)}{}$ , then $\HK \csubset \Cont{m}{}$.
	\item If $\K \in \Cont{(m,m)}{0}$, then $\HK \csubset \Cont{m}{0}$.
	\item If $\K \in \Cont{(m,m)}{b}$, then $\HK \csubset \Cont{m}{b}$, thus $\HK \csubset (\Cont{m}{b})_c$.
	\end{enumerate}
\end{proposition}

Proposition~\ref{prop:HKinC} shows that $\HK \subset \Cont{}{*}$ \gls{iff} $\HK \csubset \Cont{}{*}$. As there exist non-continuous kernels whose RKHS functions are all continuous, the hypothesises in Proposition~\ref{prop:HKinCm} is sufficient but not necessary. We do not know whether, in general, $\HK \subset \Cont{m}{*}$ implies $\HK \csubset \Cont{m}{*}$. However, if $\HK \subset \Cont{m+1}{}$, then $\K \in \Cont{(m,m)}{}$ (\citealp{grothendieck53}, cited by \citealp[Lemme~2]{schwartzFDVV}), thus $\HK \csubset \Cont{m}{}$. In particular, if $\HK \subset \Cont{\infty}{}$, then $\HK \csubset \Cont{\infty}{}$.

\subsection{Distribution Embeddings and Characteristic Kernels}

Before defining characteristic kernels, we must define the KME of a distribution. The trick resides in the interpretation of the integral $\int \K(.,x) \diff \mu(x)$.

For a bounded measurable kernel, the KME of a finite measure $\mu$ is defined as the RKHS function given by $\int \K(.,x) \diff \mu(x)$. However, as $\function{\K(.,\hat x)}{}{}{x}{\K(.,x)}$ takes its values in a possibly infinite dimensional vector space, one must agree on the definition of the integral $\int \K(.,x) \diff \mu(x)$. Usually one uses the Bochner-integral, which is convenient, because $\K(.,\hat x)$ is Bochner-integrable \gls{wrt} $\mu$ iff $\int \normK{\K(.,x)} \diff |\mu|(x) < \infty$. 
Unfortunately, the Bochner-integral is difficult to generalise to arbitrary distributions. So we will use a more general integral: the \emph{weak}- (or \mbox{\emph{Pettis}-)} integral.
\begin{remark}
	Alternatively, one may be tempted to see $\int \K(\hat{s},x) \diff \mu(x)$ as a parametric integral with parameter $s$. (This corresponds to replacing the condition $\forall f \in \HK$ in the upcoming Equations~\ref{eq:WICharacterization} and \ref{eq:KMECharacterisation} by the condition $\forall \K(.,s) , s \in \inputS$.) However, even though the integral may be defined for any parameter value $s$, the resulting function need not be in the RKHS (see Appendix~\ref{sec:Parametric}). Thus this notion of integrability is too weak for KMEs.
\end{remark}

\begin{definition}[Weak Integral and KME]
	Let $D$ be a linear form over a \gls{lcv} TVS of functions $\F$. Let  $\function{\vec{\varphi}}{\inputS}{\HK}{}{}$ be an RKHS-valued function such that for any $f \in \HK$, $\ipdK{f}{\vec \varphi(\hat x)} \in \F$.
	Then $\function{\vec{\varphi}}{\inputS}{\HK}{}{}$ is said to be weakly integrable \gls{wrt} $D$ if there exists a function in $\HK$, noted $\int \vec{\varphi}(x) \diff D(x)$ or $D(\vec \varphi)$, such that
	\begin{equation}\label{eq:WICharacterization} 
		\forall f \in \HK, \qquad \ipdK{f}{\int \vec{\varphi}(x) \diff D(x)} = \int \ipdK{f}{\vec{\varphi}(x)} \diff \bar D(x) \, .
	\end{equation}
	If $\vec{\varphi}$ is weakly integrable \gls{wrt} $D$ (or \gls{wrt} any $D$ in a set of linear forms $\D$), we say that $D$ (resp.\ $\D$) \emph{embeds} ---or \emph{is embeddable}--- \emph{into $\HK$ via $\vec{\varphi}$}. If $\vec{\varphi}(\hat x) = \K(.,\hat x)$, we omit the ``via $\K(.,\hat x)$'' and call $\int \K(.,x)\diff \mu(x)$ the kernel mean embedding (KME) of $\mu$. We note $\emb_{\vec{\varphi}}$ the map $\function{\emb_{\vec{\varphi}}}{\D}{\HK}{D}{\int \vec \varphi(x) \diff D(x)}$. It is linear, whenever $\D$ is a vector space.
\end{definition}

This definition extends the usual Bochner-integral: if $\vec{\varphi}$ is Bochner-integrable \gls{wrt} a Radon measure $\mu$, then, for any $f \in \HK$, the function $\ipdK{f}{\vec \varphi(\hat x)}$ is $\mu$-(Lebesgue-)integrable and $\vec \varphi$ is weakly integrable \gls{wrt} $\mu$ \citep[Proposition~2.3.1]{schwabik05}.

In principle, $\F$ could be any \gls{lcv} TVS of functions and $D$ any linear form. For us however, $\F$ will be one of the function spaces defined in Section~\ref{sec:Definitions}, and $D$ a distribution. Furthermore, to define KMEs, we will focus on $\vec \varphi(\hat x) = \K(.,\hat x)$. The condition that $\ipdK{f}{\vec \varphi(\hat x)} \in \F$ for any $f \in \HK$ then simply becomes $\HK \subset \F$. And, remembering that by definition $\int f(x) \diff D(x)$ stands for $D(f)$, Equation~\eqref{eq:WICharacterization} now simply reads:
\begin{equation}\label{eq:KMECharacterisation}
	\forall f \in \HK, \qquad \ipdK{f}{\int \K(.,x) \diff D(x)} = \bar D(f) \, .
\end{equation}
The left side being continuous \gls{wrt} $f$, so is the right-side. Thus $\bar D$ and $D$ are continuous linear forms over $\HK$: $D \in \HK'$. Conversely, if $D$ is a continuous linear form over $\HK$, then, by the Riesz representation theorem (see Appendix~\ref{sec:Reminders}), there exists a unique element $\int \K(.,x) \diff D(x) \in \HK$, called \emph{the Riesz representer of $\bar D$}, such that \eqref{eq:KMECharacterisation} be satisfied. We just proved
\begin{lemma}\label{lem:KernelEmbedding}
A linear form $D$ over a \gls{lcv} TVS of functions $\F$ embeds into $\HK$ \gls{iff} $\HK \subset \F$ and if the restriction $D \big|_{\HK}$ of $D$ to $\HK$ is a continuous linear form over $\HK$. In that case  $\bar D \big|_{\HK}$ is also a continuous linear form over $\HK$ and $\int \K(.,x) \diff D(x)$ is its Riesz representer.
\end{lemma}
Thus, for an embeddable space $\D$, the embedding map $\embK$ is the composition of the following two linear maps:
\begin{equation}\label{eq:EmbeddingMap}
	\embK : \, \left \{ \
		\begin{array}{ccccl}
			\D & \longrightarrow & \overline \HK' & \longrightarrow & \HK \\
			& \text{\footnotesize{Conjugate restriction}}	& & \text{\footnotesize{Riesz representer}} & \\
			D & \longmapsto & \bar D \big |_{\HK} & \longmapsto & \int \K(.,x) \diff D(x)
		\end{array}
	\right . \, ,
\end{equation}
where $\overline \HK'$ is the conjugate space of $\HK'$. It is the space $\HK'$, where the scalar multiplication $\function{}{}{}{(\lambda, f')}{\lambda f'}$ is replaced by $\function{}{}{}{(\lambda, f')}{\bar \lambda f'}$. Using $\overline \HK'$ instead of $\HK'$ is simply a trick for both the conjugate restriction and the Riesz representation maps to become linear instead of anti-linear. $\overline \HK'$ contains exactly the same elements as $\HK'$. When $\HK$ is identified with its dual $\overline \HK'$, \emph{the KME $\embK$ \emph{is} the complex conjugate of the restriction map (in short: conjugate restriction map) of $\D$ to $\HK$.}\footnote{Had we chosen to identify $\HK$, not with $\overline \HK'$, but with $\HK'$ via the \emph{anti-linear} Riesz representation map, then we would end up identifying $\embK$, which is a \emph{linear} map, with an anti-linear map!} Furthermore, we have the following obvious but crucial lemma. 
\begin{lemma}\label{lem:Transpose}
	Let $\F$ be a \gls{lcv} TVS of functions.
	If $\HK \subset \F$, then the KME is simply the conjugate $\overline \imath^\star$ of the (algebraic) transpose $\imath^\star$ of the canonical embedding $\function{\imath}{\HK}{\F}{f}{f}$.
\end{lemma}

	The continuity assumption appearing in Lemma~\ref{lem:KernelEmbedding} may seem unpalatable. Unfortunately, whenever $\HK$ is infinite dimensional, \emph{non-continuous} linear forms over $\HK$ do exist, and those have no Riesz representer. Thus the continuity assumption is necessary. Now we may wonder, how to know if a linear form $D$, for example a distribution, is continuous over $\HK$. In practice, we will use the following easy but important result.
\begin{corollary}\label{cor:EmbedDuals}
	Let $\F$ be a \gls{lcv} TVS of functions.
	If $\HK \csubset \F$, then $\F'$ embeds into $\HK$. 
	In particular:
	\begin{enumerate}
		\item If $\HK \subset \Cont{}{}$ (resp.\ $\HK \subset \Cont{}{0}$, resp.\ $\HK \subset \Cont{}{b}$ ) then $\Mc$ (resp.\ $\Mf$, resp.\ $\Mf$) embeds into $\HK$.
		\item If $\K \in \Cont{(m,m)}{}$ (resp.\ $\K \in \Cont{(m,m)}{0}$, resp.\ $\K \in \Cont{(m,m)}{b}$), then $\Dcomp{m}$ (resp.\ $\DLone{m}$, resp.\ $\DLone{m}$) is embeddable into $\HK$.
	\end{enumerate}
\end{corollary}
\begin{proof}
	The first line is a straightforward corollary of lemma~\ref{lem:Transpose}, but may as well be proven directly. Indeed, suppose that $\HK \csubset \F$. Let $D \in \F'$ and let $f, f_1, f_2, \ldots \in \HK$. If $f_n \rightarrow f$ in $\HK$ then $f_n \rightarrow f$ in $\F$, thus $D(f_n) \rightarrow D(f)$. Thus $D$ is a continuous linear form over $\HK$. The ``in particular'' part then follows from Propositions~\ref{prop:HKinC} and \ref{prop:HKinCm}.
\end{proof}
\begin{remark}
	Anticipating on Section~\ref{sec:Topology}, let us already note that if $\HK \csubset \F$, meaning that the canonical embedding map $\imath$ is continuous, then (the conjugate of) its transpose is also continuous: the KME is thus continuous over $\F'$.
\end{remark}
For some machine learning algorithms, it is interesting when KMEs are injective. This leads to
\begin{definition}[Characteristic Kernel]
	Let $\D$ be a set of linear forms that embed into an RKHS $\HK$.
	The kernel $\K$ will be said \emph{characteristic to $\D$} if the embedding $\function{\embK}{\D}{\HK}{D}{\int \K(.,x) \diff D(x)}$ is injective. 
	When $\D = \Mprob$, $\K$ is simply said \emph{characteristic}.
\end{definition}
	Diagram~\eqref{eq:EmbeddingMap} shows that a kernel is characteristic to an embeddable set $\D$ \gls{iff} the conjugate restriction map $\function{}{}{}{D}{\bar D\big|_{\HK}}$ (or equivalently simply the restriction map $\function{}{}{}{D}{D\big|_{\HK}}$) is injective. When $\D$ is a dual of a space $\F$ that contains $\HK$, it suggests to characterise characteristic kernels over $\D$ by checking whether there are ``enough'' functions in $\HK$ to uniquely define a distribution $D \in \D$ by its values taken on $\HK$. ``Enough'' will be when $\HK$ is dense in $\F$.

%

\subsection{Duality Links Universal and Characteristic Kernels\label{sec:UniversalAndCharacteristic}}


	Suppose that $\HK \csubset \F$. If $\HK$ is dense in $\F$ (i.e., if $\K$ is universal over $\F$), any $D \in \F'$ is completely determined by its values taken over $\HK$. Thus the (conjugate) restriction map is injective.
Conversely, if $\HK$ is not dense in $\F$ ($\K$ is not universal over $\F$), then, by definition, its closure $\widetilde{\HK}$ in $\F$ is strictly smaller than $\F$. Then, as a direct corollary\footnote{It is in this corollary that one uses the local convexity of $\F$. See Corollary~3 of the Hahn-Banach Theorem by \citet{treves67}.} of the Hahn-Banach Theorem, one can construct a distribution $D \in \F'$ that is identically null on $\widetilde{\HK}$ and non-null outside: $D \neq 0$ but $\bar D \big |_{\HK} =0$. So the (conjugate) restriction map is not injective. We just proved

\begin{theorem}[Universal and Characteristic Kernels\label{theo:UniversalMeansCharacteristic}]
	Let $\F$ be a \gls{lcv} TVS such that $\HK \csubset \F$.
	The kernel $\K$ is universal over $\F$ \gls{iff} it is characteristic to $\F'$.
\end{theorem}
	In particular, if $\HK$ is continuously contained in $\Cont{m}{0}$ (resp.\ $\Cont{m}{}$), then $\K$ is $c^m_0$- (resp.\ \mbox{$c^m$-)} universal \gls{iff} it is characteristic to $\DLone{m}$ (resp.\ $\Dcomp{m}$). Theorem~\ref{theo:UniversalMeansCharacteristic} generalises Theorems~3 and~16 from \citet{sriperumbudur10a}, which respectively correspond to the case $\F=\Cont{}{0}$ and $\F=\Cont{}{}$. Table~\ref{tab:UCS} illustrates the wide range of applications of Theorem~\ref{theo:UniversalMeansCharacteristic}.
\begin{table}[t]
	\renewcommand{\arraystretch}{1.5}
	\begin{tabu} to \linewidth {X[$$,c]X[$$,c]X[$$,c]X[2.75,c]X[1.95c]} 
		\hline
		\toprule
		\mathrm{Universal}			& \mathrm{Characteristic}		& \mathrm{S.P.D.}	& Usual Name							& Proof \\ 
		\midrule
		\F								& \F'							& \F'				& /											& Th.\ref{theo:UniversalMeansCharacteristic} \& P.\ref{prop:SPD} \\
		\Func			 				& \Mfs 							& \Mfs 				& \gls{spd}						 			& Th.\ref{theo:UniversalMeansCharacteristic} \& P.\ref{prop:SPD} \\
		\Func/\One	 				& \Mfs^0						& \Mfs^0			& conditionally \gls{spd}	 				& P.\ref{prop:UniversalityForP} \& R.\ref{rem:M0} \& P.\ref{prop:SPD}\\
		\Cont{}{}						& \Mc							& \Mc				& $c$-universal (or $cc$-universal)		& Th.\ref{theo:UniversalMeansCharacteristic} \& P.\ref{prop:SPD} \\
		\Cont{}{0}						& \Mf							& \Mf				& $c_0$-universal							& Th.\ref{theo:UniversalMeansCharacteristic} \& P.\ref{prop:SPD} \\
		(\Cont{}{b})_c					& \Mf							& \Mf				& $\int$spd								& Th.\ref{theo:UniversalMeansCharacteristic} \& P.\ref{prop:SPD} \\
		((\Cont{}{b})_c)/\One			& \Mprob \ (\text{or } \Mf^0)	& \Mf^0			& characteristic							& L.\ref{lem:CharacteristicAndM0} \& P.\ref{prop:UniversalityForP} \\ 
		\Lp{q}							& \Lp{q'}						& \Lp{q'}			& /											& Th.\ref{theo:UniversalMeansCharacteristic} \& P.\ref{prop:SPD} \\
		\Cont{m}{}						& \Dcomp{m}					& \Dcomp{m}		& $c^m$-universal							& Th.\ref{theo:UniversalMeansCharacteristic} \& P.\ref{prop:SPD} \\
		\Cont{m}{0}					& \DLone{m}					& \DLone{m}		& $c^m_0$-universal						& Th.\ref{theo:UniversalMeansCharacteristic} \& P.\ref{prop:SPD} \\
		(\Cont{m}{b})_c				& \DLone{m}					& \DLone{m}		& /											& Th.\ref{theo:UniversalMeansCharacteristic} \& P.\ref{prop:SPD} \\
		\Sobo{m,q}{0} 					& \ \Sobo{-m,q'}{0}				& \ \Sobo{-m,q'}{0}	& /											& Th.\ref{theo:UniversalMeansCharacteristic} \& P.\ref{prop:SPD} \\
		\bottomrule
		\hline
	\end{tabu}
	\caption{\label{tab:UCS}Equivalence between the notions of universal, characteristic and \gls{spd} kernels. This table is to be read: ``If $\HK$ continuously embeds into \emph{Column 1}, then $\K$ is universal over \emph{Column~1} \gls{iff} it is characteristic to \emph{Column~2} and iff it is \gls{spd} over \emph{Column~3}. See \emph{Column~4}.'' Here, $q<\infty$, $m \in \nat \cup \{ \infty \}$, and Th,P,R stand for Theorem, Proposition and Remark. The case of $\Mprob$ (last row) is not directly covered by Theorem~\ref{theo:UniversalMeansCharacteristic} and Proposition~\ref{prop:SPD}, because it is not even a vector space.}
\end{table}
	
	In general, the assumption that $\HK \csubset \F$ cannot be dropped ---at least, not the assumption $\HK \subset \F$. Indeed, if $\K$ is a universal kernel over $\Cont{}{0}$, then the kernel $\K + 1$ is characteristic to $\Mf$ (because \gls{ispd}, see Section~\ref{sec:SPD}). But $\HK$ is not contained in $\Cont{}{0}$.

	Theorem~\ref{theo:GeneralExistence} and its corollary will show that $c_0^m$-universal kernels do exist. Their proof use results from sections to come, which establish basic calculus rules for embedded distributions.

\section{Distributional Calculus\label{sec:Calculus}}

This section's primary goal is to show the following calculus rules and a few implications. For any embeddable distributions $D$ and $T$ and any distribution $D$ such that $\partial^p D$ is embeddable,
\begin{enumerate}[label = (R\arabic*), leftmargin = 4em]
	\item $\ipdK{f}{\int \K(.,x) \diff D(x)} = \int \ipdK{f}{\K(.,x)} \diff D(x)$ (Definition of embedding / weak integral) \label{rule1}
	\item $\ipdK{\embK(D)}{\embK(T)} = \int \K(x,y) \diff D(y) \diff \bar{T}(x)$ (Fubini) \label{rule2}
	\item $\embK(\partial^p S) = (-1)^{|p|} \int \partial^{(0,p)}\K(.,x) \diff S(x)$ (Differentiation). \label{rule3}
\end{enumerate}
Rule~\ref{rule1} holds by definition of an embedding. So let us move directly to Rule~\ref{rule2}.

\subsection{Inner Product and Norm of Embeddings}
If $D$ and $T$ are two embeddable distributions for a kernel $\K$, we write:
\begin{align*}
	\ipdK{D}{T} &:= \ipdK{\embK(D)}{\embK(T)} \\
	\normK{D} &:= \normK{\embK(D)} \, .
\end{align*}
Note that this induces a new metric on the set of embeddable distributions (see Section~\ref{sec:Topology}).
\begin{theorem}[Fubini] \label{theo:Fubini}
	Let $D,T$ be two embeddable distributions into $\HK$. Then:
	\begin{align}
		\ipdK{D}{T} &= \iint \K(x,y) \diff D(y) \diff \bar T(x) = \iint \K(x,y) \diff \bar T(x) \diff D(y) \label{eq:FubiniIP} \\
		\normK{D}^2 &= \iint \K(x,y) \diff D(y) \diff \bar D(x) = \iint \K(x,y) \diff \bar D(x) \diff D(y) \, , \nonumber 
	\end{align}
	where $ \iint \K(x,y) \diff D(y) \diff \bar T(x)$ is to be interpreted as $\bar T_x(\mathcal{I}_{\hat x})$ with $\mathcal{I}_x = D_y(\K(x, \hat y))$.
\end{theorem}
\begin{proof}
	See Appendix~\ref{proof:Fubini}.
\end{proof}
	We will also use the tensor product notation
	\begin{equation*}
		[\bar{T} \otimes D](\K) := [\bar{T}_x \otimes D_y](\K) := \iint \K(x,y) \diff \bar{T}(x) \diff D(y) = \ipdK{D}{T} \, .
	\end{equation*}

\subsection{Application to Strict Positive Definiteness\label{sec:SPD}}

\begin{definition}[S.P.D.]
	A kernel $\K$ is said \gls{spd} over a set of distributions $\D$ if $D$ is embeddable into $\HK$ and if for any $D \in \D$,
	\begin{equation}\label{eq:SPD}
		\normK{D}^2 = 0 \qquad \Rightarrow \qquad D=0 \, .
	\end{equation}
	Following the tradition, when $\D = \Mfs$, we simply call $\K$ \gls{spd} When $\D = \Mf$, $\K$ is said\glsreset{ispd} \gls{ispd}.
\end{definition}
\begin{remark}
	By definition, a kernel is \gls{spd} over $\D$ \gls{iff} the semi-norm induced by $\K$ in $\D$ is a proper norm. Also note that, using Theorem~\ref{theo:Fubini} (Fubini), we may rewrite Equation~\eqref{eq:SPD} as follows.
	For any $D \in \D$, 
		\[
			\int \K(x,y) \diff D(y) \diff \bar D(x) = 0 \qquad \Rightarrow \qquad D = 0.
		\]
\end{remark}
Remembering that by definition, for any embeddable $D$, $\normK{D} := \normK{\embK(D)}$, we immediately get:
\begin{proposition}\label{prop:SPD}
	For any embeddable \emph{vector space} of distributions $\D$, $\K$ is characteristic to $\D$ \gls{iff} it is \gls{spd} over $\D$.
\end{proposition}
In particular, saying that a kernel $\K$ is \gls{spd} (resp.\ \gls{ispd}) is equivalent to saying it is characteristic to $\Func$ (resp.\ $\Mf$). Note also that the proposition does not apply to the case where $\D$ is the set of probability measures $\Mprob$. However, Lemma~\ref{lem:CharacteristicAndM0} will show that a kernel is characteristic to $\Mprob$ \gls{iff} it is \gls{spd} over $\Mf^0 := \{ \mu \in \Mf \, | \, \int \diff \mu = 0 \}$.

\subsection{Differentiation}\label{sec:Differentiation}
Let us now turn towards differentiation. We start with an easy, but useful lemma proved in Appendix~\ref{proof:VectorDiff}.

\begin{lemma}\label{lem:VectorDiff}
		Let $\K \in \Cont{(m,m)}{}$ and $\vec \varphi_{\K}(\hat x) = \K(.,\hat x)$. The RKHS-valued function $\vec \varphi_{\K}$ is $m$-times continuously differentiable and $\partial^p \vec \varphi_{\K} = \partial^{(0,p)} \K$. 
\end{lemma}
This lemma is used in Remark~\ref{rem:DiffIntro} to give a concise interpretation of the following theorem (shown in Appendix~\ref{proof:Differentiation}).
\begin{theorem}\label{theo:Differentiation}
	Let $\K \in \Cont{(m,m)}{}$ and $p \in \nat^d$ such that $|p| \leq m$.
	A distribution $D$ embeds into $\HK$ via $\partial^{(0,p)} \K$ \gls{iff} $\partial^p D$ embeds into $\HK$ via $\K$. In that case,
	\begin{equation*}\label{eq:DiffLong}
		\embK(\partial^p D) = (-1)^{|p|} \int [\partial^{(0,p)} \K](.,x) \diff D(x) = (-1)^{|p|} \, \emb_{\partial^{(0,p)} \K}(D) \, .
	\end{equation*}
\end{theorem}
\begin{remark}\label{rem:DiffIntro}
	Noting $\vec \varphi_{\K}(\hat x) = \K(.,\hat x)$, Equation~\eqref{eq:DiffLong} reduces to
		$\partial^p D (\vec \varphi_{\K}) = (-1)^{|p|} D(\partial^p \vec \varphi_{\K}) \, .$
	A sufficient condition for this equality to hold is $D \in \Cont{m - |p|}{*}$. Taking $D = \delta_x$ proves Equation~\eqref{eq:DiffIntro} from the introduction: by definition, $\delta_x(\partial^p \vec \varphi_{\K}) = \partial^p \vec \varphi_{\K}(x)$ and $\partial^p \delta_x (\vec \varphi_{\K}) = \embK(\partial^p \delta_x)$.
\end{remark}
This result has two important practical applications.

First and foremost, it gives a concrete way to numerically calculate the KME of a distribution $D$. Indeed, consider the case where $D$ is a measure $\mu$. Equation~\eqref{eq:DiffLong} shows we may compute the embedding of $\partial^p \mu$ using only integration \gls{wrt} $\mu$. In other words, we have reduced the computation of an integral \gls{wrt} the distribution $\partial^p \mu$ to an integral \gls{wrt} the measure $\mu$. And this trick applies integrable to any embeddable and integrable distribution, because \citep[around p.100]{schwartzFDVV}:
\begin{proposition} \label{prop:RepresentationOfDLone}
	For any $m \leq \infty$ and any distribution in $D \in (\Cont{m}{0})'$ (resp.\ $D \in (\Cont{m}{})'$) there exists a \emph{finite} family of measures $\mu_p \in \Mf$ (resp.\ $\mu \in \Mc$), with $p \in \nat^d$, $|p| \leq m$, such that $D = \sum_{|p| \leq m} \partial^p \mu_p$.
\end{proposition}


A second application is the following. 
Suppose we could easily measure the values of a function $f$ (say the location of a particle), but not of its derivative $\partial f$ (the speed). But we want the KME of $\partial f$. 
Then Theorem~\ref{theo:Differentiation} asserts that, under suitable assumptions on $f$, instead of measuring $\partial f$ and embed it, we can measure $f$, embed it and differentiate the embedding.

Let us now show how embeddings in $\HK$ relate to those in $\Hpp$. The following proposition, proved in Appendix~\ref{proof:Isometry}, shows that $\Hpp$ is isometrical to a subspace of $\HK$.
\needspace{4\baselineskip}
\begin{proposition}\label{prop:Isometry}
	Let $\K \in \Cont{(m,m)}{}$ and $p \in \nat^d$ with $|p| \leq m$. Define $\Hop$ to be the closed subspace in $\HK$ generated by $\partial^{(0,p)}\K (., x)$ when $x \in \inputS$: 
	\begin{equation*}
		\Hop := \overline{\lin \{\partial^{(0,p)} \K(.,x) \, | \, x \in \inputS \}} \subset \HK \, .
	\end{equation*}
	Then the following defines a bijective isometry from $\Hop$ onto $\Hpp$:
	\begin{equation*}
		\function{\Dp}{\Hop}{\Hpp}{f}{\partial^p f} \, .
	\end{equation*}
\end{proposition}

Obviously $\Hop$ is a closed subset of $\HK$, but is it a proper subset? Let us focus on the case where $\K \in \Cont{(m,m)}{0}$. Assume that $|p| = 1$, say $p = (p_1, 0, \ldots, 0) \in \nat^d$. As, for any $f \in \HK$, $\ipdK{f}{\partial^p \K(.,x)} = \partial^p f(x)$, the orthogonal complement $\Hop^\perp$ of $\Hop$ is contained in the set of functions that do not depend on the first coordinate. Thus, for example, if $\inputS$ is 1-dimensional ($d=1$), then $\Hop^\perp$ is also at most 1-dimensional. And obviously, if $\HK \subset \Cont{m}{0}$, then it is even $\{0\}$. And by recurrence on $|p|$, we get
\begin{theorem}\label{theo:Isometry}
	Let $\K \in \Cont{(m,m)}{0}$ and $p \in \nat^d$ such that $|p| \leq m$. Then $\Hop = \HK$. $\HK$ and $\Hpp$ are then isometrically isomorphic via $\Dp$, the partial derivation operator.
\end{theorem}

\begin{remark}
	In geostatistics, one often considers the subspace $\HK^0$ of $\HK$ defined as the closure of $\{ \sum_{i = 1}^n \lambda_i \K(.,x_i) \, | \, n \in \nat, \, x_i \in \inputS, \, \lambda_i \in \complex, \, \sum_{i =1}^n \lambda_i = 0 \}$. If $\K \in \Cont{(m,m)}{}$, it is not difficult to see that $\Hop \subset \HK^0$. Thus, if $\K \in \Cont{(m,m)}{0}$, $\HK^0 = \HK$. So when doing so-called \emph{intrinsic kriging}, as opposed to \emph{simple kriging} (i.e.\ with a known mean, see specialised literature such as \citealp{chiles12} and \citealp{wackernagel03}), we solve two different minimisation problems, yet, when $\HK \subset \Cont{}{0}$, we look for solutions that lie in the same set of functions $\HK$.
\end{remark}

\needspace{5\baselineskip}
\begin{proposition}[Switches of Derivation and Integration]\label{prop:IntDiffSwitch}
	Let $\K \in \Cont{(m,m)}{}$, $p \in \nat^d$ such that $|p| \leq m$ and let $D$ be a distribution. Consider the following three statements.
	\begin{enumerate}
		\item $\partial^p D$ embeds into $\HK$ via $\K$. \label{pPoint1}
		\item $D$ embeds into $\HK$ via $\partial^{(0,p)} \K$. \label{pPoint2}
		\item $D$ embeds into $\Hpp$ via $\partial^{(p,p)} \K$. \label{pPoint3}
	\end{enumerate}
	Conditions \ref{pPoint1} and \ref{pPoint2} are equivalent, and they imply \ref{pPoint3}.
	If they are satisfied, then
	\begin{equation}\label{eq:IntDiffSwitch}
		(-1)^{|p|} \partial^p \embK(\partial^p D) = \partial^p \emb_{\partial^{(0,p)} \K}(D) = \emb_{\partial^{(p,p)} \K}(D) \, .
	\end{equation}
	Additionally, if $\K \in \Cont{(m,m)}{0}$, then all three conditions are equivalent.
\end{proposition}

This proposition, proved in Appendix~\ref{proof:IntDiffSwitch}, calls for some comments. First, Equation~\eqref{eq:IntDiffSwitch} states in particular, that the integration and derivation signs may be switched:
\begin{equation*}
	\partial^p [\int \partial^{(0,p)} \K(.,x) \diff D(x)] = \int \partial^{(p,p)} \K(., x) \diff D(x) \, .
\end{equation*}
Second, when $\K \in \Cont{(m,m)}{0}$, the function $\emb_{\partial^{(0,p)} \K}(D)$ is \emph{the only} primitive of $\emb_{\partial^{(p,p)} \K}(D)$ contained in $\HK$: it is the primitive that converges to 0 at infinity. This confirms what we already knew: $\Dp$ is injective. Third, in general, even if the conditions \ref{pPoint1}-\ref{pPoint3} are met, the distribution $D$ may not embed into $\HK$ via $\K$. Because even if a distribution $\partial^p D$ is in $\DLone{m-|p|}$, the distribution $D$, or any other primitive of order $p$ of $\partial^p D$, may not be in $\DLone{m}$.
\begin{example}
Consider the Gaussian kernel $\K(\hat x, \hat y) = e^{(\hat x - \hat y)^2}$ over $\inputS = \real$ and the measure $D = \arctan x \diff x$. Then $D \in \DLone{1}$, $\partial D = \frac{\diff x}{1+x^2} \in \DLone{0}$ and $\partial D$ embeds into $\HK$; but $D$ does not. For if it did, the function $\function{}{}{}{y}{\int \K(y, x) \arctan x \diff x}$ would vanish at infinity, because $\HK \subset \Cont{}{0}$. But, when $y \rightarrow +\infty$, it converges to a strictly positive real.
\end{example}

Noting that for stationary kernels, $\partial^{(0,p)} \K(\hat x,\hat y) = (-1)^{|p|} \, \partial^p \psi (\hat x - \hat y)$, we get the following relation between the embedding of the derivative and the derivative of the embedding.
\begin{corollary}[For Stationary Kernels]
	Suppose that the kernel $\K \in \Cont{(m,m)}{}$ be stationary and that assumptions \ref{pPoint1}-\ref{pPoint3} of Proposition~\ref{prop:IntDiffSwitch} are met. Then
	\begin{equation*}
		\embK(\partial^p D) = \partial^p [\embK(D)].
	\end{equation*}
\end{corollary}
Another corollary, proved in Appendix~\ref{proof:dpMandMCharacteristic}, is the following.
\begin{corollary}\label{cor:dpMandMCharacteristic}
	Let $\K \in \Cont{(m,m)}{}$ and $p \in \nat^d$ such that $|p| \leq m$. Let $\D$ be a set of distributions such that the set $\partial^p \D := \{ \partial^p D \, | \, D \in \Dall{} \}$ embeds into $\HK$. Then the kernel $\partial^{(p,p)} \K$ is characteristic to $\D$ \gls{iff} it is characteristic to $\partial^p \D$.
\end{corollary}
When $\K \in \Cont{(m,m)}{0}$ (resp.\ $\K \in \Cont{(m,m)}{}$), this corollary applies in particular to any subset $\D$ of $\DLone{m - |p|}$ (resp.\ $\Dcomp{m - |p|}$). This will play an important role in the proofs of Theorems~\ref{theo:GeneralExistence} and~\ref{theo:GeneralExistenceCompact}, that show the existence of $c^\infty_0$- and $c^\infty$-universal kernels.


\section{The Set of Probability Measures \texorpdfstring{$\Mprob$}{P}}\label{sec:ProbabilityMeasures}

So far, we focused on the properties of distribution embeddings in general, putting a special emphasis on duals of function spaces, such as $\Mf$ or its generalisations $\DLone{m}$. In most applications, however, one only wants to embed probability measures. And, although the set of probability measures $\Mprob$ is not a vector space, the same questions as before arise. In particular: when is a kernel characteristic, not over all $\Mf$, but solely over the set of probability measures $\Mprob$? As $\Mprob \subset \Mf$, it is clear that if $\K$ is characteristic to $\Mf$ ---i.e., if $\K$ is \gls{ispd}--- then it is characteristic to $\Mprob$. In Section~\ref{sec:CandISPD} we show that the converse does not hold and give an astonishingly simple characterisation of characteristic kernels. Section~\ref{sec:MprobAndUniversality} then gives a second, more theoretical characterisation that complements Section~\ref{sec:UniversalAndCharacteristic}.
\\

Before we start, we make some important preliminary observations about characteristic kernels.
The set of probability measures lies in the (closed) affine hyperplane $\Mf^1$ of $\Mf$ given by the equation $\mu(\One) = 1$, where $\One$ is the constant function that equals $1$: $\Mf^1 := \{ \mu \in \Mf \, | \, \mu(\One) = 1\}$
. $\Mprob$ is not equal to $\Mf^1$. But when translated by, say, $-\delta_x$ ($x \in \inputS$) to lie in the hyperplane $\Mf^0 := \{ \mu \in \Mf \, | \, \mu(\One) = 0\}$, its linear span spans the whole hyperplane $\Mf^0$ : $\lin (\Mprob - \delta_x) = \Mf^0$. Thus:
\begin{lemma} [Important!] \label{lem:CharacteristicAndM0}
	A kernel $\K_0$ is characteristic to $\Mprob$ (resp.\ $\Pc := \Mprob \cap \Mc$) \gls{iff} it is characteristic to the closed hyperplane $\Mf^0:= \{ \mu \in \Mf \, | \, \mu(\One) = 0\}$ in $\Mf$ (resp.\ $\Mc^0:= \Mf^0 \cap \Mc$ in $\Mc$).
\end{lemma}
Being given a kernel $\K$ is equivalent to being given its KME $\embK$ (over a set containing at least $\Mfs$): the former defines the latter by construction, and the latter defines the former because
\begin{equation*}
	\K(x,y) = \ipdK{\embK(\delta_y)}{\embK(\delta_x)} =: \ipdK{\delta_y}{\delta_x} \, .
\end{equation*}
Thus, informally speaking, Lemma~\ref{lem:CharacteristicAndM0} shows that a characteristic kernel $\K_0$ is ``almost'' \gls{ispd} Both the KME of a characteristic and the KME of an \gls{ispd} kernel are injective over $\Mf^0$. But the first need not be injective over $\Mf$: there might be one single line in $\Mf$, not contained in $\Mf^0$ and mapped to the null function of $\HH_{\K_0}$. This suggests that, from an \gls{ispd} kernel $\K$, it is easy to construct a characteristic kernel $\K_0$ that is not \gls{ispd} anymore (Proposition~\ref{prop:ISPDtoC}): consider its embedding $\embK$, choose a direction $\nu_0 \in \Mf \backslash \Mf^0$, and construct a new embedding $\emb_{\K_0}$ that equals $\embK$ on $\Mf^0$, but maps $\nu_0$ to the null-function. The kernel associated to this new embedding is characteristic but not \gls{ispd} Conversely, it also suggests that, given a characteristic kernel $\K_0$ that is not already \gls{ispd}, one may easily construct an \gls{ispd} kernel (Proposition~\ref{prop:CtoISPD}): add 1 dimension to $\HH_{\K_0}$, and construct an embedding $\embK$ that coincides with $\emb_{\K_0}$ on $\Mf^0$ but maps $\nu_0$ to the additional new dimension. Thus, contrary to $\emb_{\K_0}$, $\embK$ does not identify all measures parallel to $\nu_0$ and is thus injective over $\Mf$. By definition, its associated kernel $\K$ is then \gls{ispd} These considerations will lead to the characterisation of characteristic kernels given in Theorem~\ref{theo:CKCharacterisation}.

\subsection{\texorpdfstring{\GLS{spd}}{S.P.D.} Kernels and Kernels Characteristic to \texorpdfstring{$\Mprob$}{P}\label{sec:CandISPD}}

\begin{proposition}[Characteristic $\not \Rightarrow$ \gls{ispd}] \label{prop:ISPDtoC}
	Let $\K$ be an \gls{ispd} kernel or, equivalently, a characteristic kernel over $\Mf$. Let $\nu_0 \in \Mf$ such that $\nu_0(\One) \neq 0$. 
	Define, for any $x,y \in \inputS$, the kernel $\K_0(x,y) := \ipdK{\delta_x - \nu_0}{\delta_y - \nu_0}$.
	Then $\K_0$ embeds $\Mf$, but is not characteristic to $\Mf$. However, $\K_0$ and $\K$ induce the same semi-metric $d_{\K}$ in $\Mf^0$ and $\Mprob$. Thus $\K_0$ is a kernel that is characteristic to $\Mprob$, but not over $\Mf$.
\end{proposition}
For instance, taking $z_0 \in \inputS$ and $\nu_0 = \delta_{z_0}$ leads to $\K_0(x,y) = \K(x,y) - \K(x,z_0) - \K(z_0,y) + \K(z_0,z_0)$. As $\emb_{\K_0}(\delta_{z_0}) = 0$, this example shows the following
\begin{corollary}[Characteristic $\not \Rightarrow$ \gls{spd}]
	There exist characteristic kernels, which are not strictly positive definite.
\end{corollary}
\begin{proof}{(of Proposition~\ref{prop:ISPDtoC})}
	$\K_0$ is a kernel, because it is of the form $\ipdK{\vec \varphi(x)}{\vec \varphi(y)}$, for the map $\function{\vec \varphi}{\inputS}{\HK}{x}{\embK(\delta_x-\nu_0)}$. 
	As $\emb_{\K_0}(\nu_0) = 0$, it is not characteristic to $\Mf$. To prove that $\K$ and $\K_0$ induce the same semi-metric in $\Mf^0$ (and thus in $\Mprob$ by Lemma~\ref{lem:CharacteristicAndM0}), simply notice the following. The kernel $\K_0$ verifies: $\K_0(x,y) = \K(x,y) + \psi(x) + \bar \psi(y)$ for some function $\function{\psi}{\inputS}{\complex}{}{}$. Using Theorem~\ref{theo:Fubini} (Fubini), we thus get, for any $\mu \in \Mf^0$
	\begin{equation*}
		\norm{\mu}_{\K_0} = [\bar \mu_x \otimes \mu_y](\K(x,y) + \psi(x) + \bar \psi(y))
			= [\bar \mu \otimes \mu](\K) + 0 + 0
			= \normK{\mu} \,. 	\vspace{-1.8em}
	\end{equation*}
\end{proof}

With their Example~1, \citet{sriperumbudur11} already exhibited a kernel that is characteristic but not \gls{spd} and a fortiori not \gls{ispd} But their construction may seem somewhat abstruse. Proposition~\ref{prop:ISPDtoC} on the contrary not only gives an easy and systematic way to construct a characteristic but non-\gls{ispd} kernel, but the next proposition ---the converse--- shows that \emph{all} such kernels can be constructed with this procedure! 

\begin{proposition}[Characteristic is ``almost'' \gls{ispd}] \label{prop:CtoISPD}
	Let $\K_0$ be a characteristic kernel to $\Mprob$.
	Then there exists a kernel $\K$ and a measure $\nu_0 \in \Mf$ such that, $\K_0(\hat x,\hat y) = \ipdK{\delta_{\hat x} - \nu_0}{\delta_{\hat y} - \nu_0}$.
\end{proposition}
\begin{remark}
	Both in Proposition~\ref{prop:ISPDtoC} and \ref{prop:CtoISPD} we may replace ``characteristic (to $\Mprob$ or $\Mf^0$)'' and ``\gls{ispd}'' (i.e.\ ``characteristic to $\Mf$'') by respectively ``characteristic to $\M^0$'' and ``characteristic to $\M$'', where $\M$ is $\Mfs$ or any embeddable $\Dcomp{m}$ or $\DLone{m}$, and $\M^0 := \{ \mu \in \M \, | \, \int \diff \mu = 0 \}$. 
Note that $\K_0$ and $\K$ define the same equivalence class of \gls{cpd} kernels (see Remark~\ref{rem:EquivClassCPD} in Appendix~\ref{sec:CPD}).
\end{remark}

We prove Proposition~\ref{prop:CtoISPD} in Appendix~\ref{proof:CtoISPD}. Minor changes in this proof (see Appendix~\ref{proof:CKCharacterisation}) yield the following, seemingly more general and astonishingly simple characterisation of characteristic kernels. Surprisingly, we could not find any similar statement in the literature.
\needspace{5\baselineskip}
\begin{theorem}[Characterisation of Characteristic Kernels] \label{theo:CKCharacterisation}
	Let $\K_0$ be a kernel. The following conditions are equivalent.
	\begin{enumerate}
		\item $\K_0$ is characteristic to $\Mprob$.
		\item $\K_0$ is characteristic to $\Mf^0$.
		\item There exists $\epsilon \in \real$ such that the kernel $\K(\hat x,\hat y):= \K_0(\hat x, \hat y) + \epsilon^2$ is characteristic to $\Mf$ (i.e.\ \gls{ispd} over $\Mf$).
		\item For all $\epsilon \in \real\backslash\{0\}$, the kernel $\K(\hat x,\hat y):= \K_0(\hat x, \hat y) + \epsilon^2$ is characteristic to $\Mf$.
		\item There exists an RKHS $\HK$ with kernel $\K$ and a measure $\nu_0 \in \Mf\backslash\Mf^0$ such that $\K$ is characteristic to $\Mf$, and $\K_0(\hat x, \hat y) = \ipdK{\delta_{\hat x} - \nu_0}{\delta_{\hat y} - \nu_0}$.
	\end{enumerate}
	If these conditions are met, then $\K_0$ and $\K$ induce the same semi-metric $d_{\K}$ in $\Mf^0$ and in $\Mprob$.
\end{theorem}

\subsection{Characteristic Kernels to \texorpdfstring{$\Mprob$}{P}: Characterisation in Terms of Universality\label{sec:MprobAndUniversality}} 

This section is a complement to Theorem~\ref{theo:UniversalMeansCharacteristic}. The latter applies to dual spaces but not necessarily to strict subsets. In particular, Theorem~\ref{theo:UniversalMeansCharacteristic} does not directly apply to probability measures. Nevertheless, we will now characterise characteristic kernels (to $\Mprob$) in terms of universality. 

	To do so, consider the space of bounded continuous functions $\Cont{}{b}$. Following \citet{schwartzFDVV}, we equip $\Cont{}{b}$ with a topology $\topo_c$ that is weaker than its usual one.\footnote{An alternative would be to equip $\Cont{0}{b}$ with the so-called \emph{strict topology} \citep[cf.][]{fremlin72}.} 
Let $\topo_1$ be the topology induced by $\Cont{}{}$ in $\Cont{}{b}$ and $\topo_2$ the topology of uniform convergence over the compact sets of $\Mf$. Then we define $\topo_c$ as the smallest topology containing $\topo_1$ and $\topo_2$. We note\footnote{\Citet{schwartzFDVV} uses the symbol $\B^m_c$ rather than $(\Cont{}{b})_c$.} $(\Cont{}{b})_c$ the space $\Cont{}{b}$ equipped with $\topo_c$. In this space, a sequence $(f_n)_{n \in \nat}$ converges to $f$ iff
\begin{enumerate}
	\item $f_n \rightarrow f$ in $\Cont{}{}$ (i.e., on every compact set $\norm{f_n - f}_{\infty} \rightarrow 0$) and
	\item for any compact set $A \subset \Mf$, $\sup_{\mu \in A} |\mu(f_n - f)| \rightarrow 0$.
\end{enumerate}
The space $(\Cont{}{b})_c$ is complete and its dual is $\Mf$. The advantage of using $(\Cont{}{b})_c$ instead of $\Cont{}{0}$ as a predual of $\Mf$ is that $\Cont{}{b}$ contains the constant unit function $\One$. Thus we can define the quotient space $((\Cont{}{b})_c / \One)$ and may state (and prove in Appendix~\ref{proof:UniversalityForP})
\begin{proposition}\label{prop:UniversalityForP}
	Let $\K$ be such that $\HK \csubset (\Cont{}{b})_c$. Then $\K$ is characteristic to $\Mprob$ \gls{iff} $\K$ is universal over the quotient space $((\Cont{}{b})_c/\One)$.
\end{proposition}
\begin{remark}\label{rem:M0}
	Supposing that $\HK \csubset ((\Cont{}{b})_c/\One)$ is sufficient. Additionally, Proposition~\ref{prop:UniversalityForP} also holds if, in this whole section, one systematically replaces $\Cont{}{}$, $\Cont{}{b}$, $\Mf$, $\Mprob$ and $\norm{f}_{\infty}$ by $\Cont{m}{}$, $\Cont{m}{b}$, $\DLone{m}$, $(\DLone{m})^0 := \{ D \in \DLone{m} \, | \, D(\One) = 0 \}$ and $\forall |p| \leq m, \norm{\partial^p f}_{\infty}$ respectively. It also holds with $\Func$ and $\Mfs^0$ instead of $\Cont{}{b}$ and $\Mprob$.
\end{remark}

Applying Theorem~\ref{theo:UniversalMeansCharacteristic} to $\F = (\Cont{m}{b})_c$, we get the following, in practice very useful lemma.

\begin{lemma}
	Let $\K$ be such that $\HK \csubset (\Cont{m}{b})_c$. (For example, let $\K \in \Cont{(m,m)}{b}$.) Then $\K$ is characteristic to $\DLone{m}$ \gls{iff} $\K$ is universal over $(\Cont{m}{b})_c$.
\end{lemma}

We close this section with a remark on an unfruitful, but maybe enlightening attempt to characterise characteristic kernels.

\begin{remark}
	Let $\inputS = \real$. For any $\mu \in \Mf$, if $\partial \mu \in \Mf$, then $\partial \mu \in \Mf^0$. In short: $\partial \Mf \cap \Mf \subset \Mf^0$. In light of Proposition~\ref{prop:IntDiffSwitch} and its Corollary~\ref{cor:dpMandMCharacteristic}, we were tempted to formulate a proposition like: ``Let $\K \in \Cont{(1,1)}{0}$ be a kernel such that $\partial^{(1,1)} \K$ is characteristic to $\Mf$. Then $\K$ is characteristic to $\Mf^0$ (i.e.\ to $\Mprob$).'' Unfortunately $\partial \Mf \cap \Mf \not = \Mf^0$~, thus we failed to conclude.
\end{remark}

\section{\texorpdfstring{$c^m$}{Cm}- and \texorpdfstring{$c^m_0$}{Cm0}-Universal Kernels\label{sec:UniversalKernels}}

In this section, we will give sufficient conditions for a kernel to be $c_0^m$- (Theorem~\ref{theo:GeneralExistence}) or $c^m$-universal (Theorem~\ref{theo:CompactCharacteristic}). This will finally show that $c_0^\infty$- and $c^\infty$-universal kernels do exist. Furthermore, stationary kernels $\K(x,y) := \psi(x-y)$ will behave particularly beautifully: as soon as they are smooth enough, $c^m_0$- (resp.\ \mbox{$c^m$-)} and $c_0$- (resp.\ $c-$) universal kernels turn out to be one and the same thing (Corollary~\ref{cor:Existence}, resp.\ Proposition~\ref{prop:CompactCharacteristicRd}). 
Thus, machine learners need not change their habits: if they show, as usual, that a smooth enough, stationary kernel is $c_0$-universal, they get $c^m_0$-universality for free, even without any prerequisites in distribution theory!

\subsection{\texorpdfstring{$c_0^m$}{Cm0}-Universal Kernels}

\begin{theorem}\label{theo:GeneralExistence}
	\textcolor{red}{PROOF FLAWED.} Let $\K$ be a kernel in $\Cont{(m,m)}{b}$. If, for any $p \in \nat^d$ with $|p| \leq m$, $\partial^{(p,p)} \K$ is characteristic to $\Mf$, then $\K$ is characteristic to $\DLone{m}$.
\end{theorem}
\begin{proof}
	We proceed by recurrence on $|p|$. For ease of notation however, we restrict ourselves to the step from $|p|=0$ (which obviously holds) to $|p| = 1$.
	
	Let $p \in \nat^d$ with $|p| = 1$. Suppose there exist two different distribution $D,T \in \DLone{1}$ such that $\embK(D) = \embK(T)$. As $\partial^{(p,p)} \K$ is characteristic to $\Mf$, $\K$ is characteristic to $\partial^p \Mf$ (Corollary~\ref{cor:dpMandMCharacteristic}). Thus $\K$ is characteristic to $\Mf$ and over $\partial^p \Mf$. Consequently $D$ and $T$ cannot be both in $\Mf$ or both in $\partial^p \Mf$. Thus, one of them, say $D$, is in $\Mf \backslash (\partial^p \Mf)$.
	Now, $\embK(D) = \embK(T)$ implies that, for any $f \in \HK$: $D(f) = T(f)$. But $D$ being a measure, and $\HK$ being dense in $\Cont{0}{0}$, $D$ defines a unique element in $\Mf$ by its values taken on $\HK$. In particular, it is uniquely defined on $\Cont{1}{0} \subset \Cont{0}{0}$ by its values taken on $\HK$. Thus so is $T$. Thus $D = T$. Contradiction. Thus $\K$ is characteristic to $\Mf \cup (\partial^p \Mf)$.
	For any other $q \in \nat^d$ with $|q| = 1$, one now shows in a similar fashion that $\K$ is characteristic to $(\Mf \cup (\partial^p \Mf)) \cup \partial^q \Mf$. Continuing this process, one finds that $\K$ is characteristic to $\cup_{|p| \leq 1} \partial^p \Mf$. In other words: $\K$ is characteristic to $\DLone{1}$ (Proposition~\ref{prop:RepresentationOfDLone}).
\end{proof}

As a beautiful corollary, we get:
\needspace{5\baselineskip}
\begin{corollary}\footnote{The result of this corollary still holds, despite
the proof of Theorem~\ref{theo:GeneralExistence} being flawed. The alternative
proof given in Appendix~\ref{proof:Existence} and shortened in the JMLR version
of this paper is independent of Theorem~\ref{theo:GeneralExistence}.} \label{cor:Existence}
	Let $\K(\hat x,\hat y) = \psi(\hat x-\hat y) \in \Cont{(m,m)}{b}$ be a stationary kernel over $\inputS = \real^d$. 
	The following are equivalent.
	\begin{enumerate}
		\item $\K$ is characteristic (to $\Mprob$). \label{th:Mprob}
		\item $\K$ is characteristic to $\Mf$ (i.e.\ \gls{ispd}). \label{th:Mf}
		\item $\K$ is characteristic to $\DLone{m}$. \label{th:DLone} 
		\item $\supp \Four \psi = \inputS$, meaning that the distributional Fourier transform $\Four \psi$ of $\psi$ has full support. \label{th:Support}
	\end{enumerate}
	In particular, if $\psi \in \Cont{m}{0}$, then $\K$ is $c_0^m$-universal \gls{iff} it is $c_0$-universal.
\end{corollary}
\begin{remark}
	Modulo a possible normalisation or scaling factor that depend on the chosen convention in the definition of the Fourier transform, $\Four \psi$ \emph{is} but the (positive) measure $\Lambda$ that the Bochner theorem associates to the stationary kernel $\K$.
\end{remark}


\begin{proof}
    The following proof relies on Theorem~\ref{theo:GeneralExistence} and is
    therefore flawed. Use the proof given in \citet{simon18kernel} instead
    (which is a shortened version of the alternative proof given in
    Appendix~\ref{proof:Existence}).

	That \ref{th:Mprob} and \ref{th:Support} are equivalent is Theorem~9 of \citet{sriperumbudur10b}. Looking at the proof, it is straightforward to see that they are also equivalent to~\ref{th:Mf}. It is clear that~\ref{th:DLone} implies~\ref{th:Mf}. Thus we only need to show that~\ref{th:Support} implies~\ref{th:DLone}.
	
	Suppose \ref{th:Support} holds. The Fourier transforms of $\partial^{(p,p)} \K(\hat x, \hat y) = (-1)^p \partial^{2p}\psi(\hat x - \hat y)$ satisfy $\Four (\partial^{2p} \psi) = (-i)^{|p|} \hat \xi^p \Four \psi$. Thus they do also have full support. So for each $|p| \leq m$, $\partial^{(p,p)} \K$ is characteristic to $\Mf$. Thus Theorem~\ref{theo:GeneralExistence} applies and $\K$ is characteristic to $\DLone{m}$.
\end{proof}
We give an alternative proof that is independent of
Theorem~\ref{theo:GeneralExistence} in Appendix~\ref{proof:Existence}.

This result calls for comments. First, it shows that the popular Gaussian kernel $\K(x,y) =  C e^{-(x-y)^2/\sigma^2}$ (where $C, \sigma^2 > 0$) is $c_0^\infty$-universal. Now, informally, what does it mean when a KME is injective over $\DLone{\infty}$? 
A distribution $D$ is by definition determined by all the values $D(\psi)$ it takes, when $\psi$ runs over the space of test functions $\Cont{\infty}{c}(\real^d)$. Even better: let $\psi_{\epsilon} \geq 0$ be a function in $\Cont{\infty}{c}$ with support contained in a compact set of diameter~$\leq \epsilon$ and such that $\int \psi_\epsilon(y) \diff y = 1$. Then the convolutional product $[D*\psi_\epsilon](x) := D_y(\psi_\epsilon(x-y))$ is a function in $\Cont{\infty}{c}$ that converges in $\Dall{\infty}$ to $D$ when $\epsilon \rightarrow 0$. This often leads physicists to interpret $D_y(\psi_\epsilon(x-y))$ as the measurement at point $x$ of some natural phenomenon (modelled by) $D$ by an imprecise instrument $\psi_\epsilon$. They see $D_y(\psi_\epsilon(x-y))$ as a ``weighted mean'' of $D$ on a small compact set (of diameter at most $\epsilon$) around $x$. The more precise the measurement device $\psi_\epsilon$, the smaller the $\epsilon$, the better the measurement. In particular, a physicist need not define a natural phenomenon $D$ by its exact values in every point $x$, but rather by all the possible measurements that one could do of $D$ at every point $x$. This brings us back to $c^\infty_0$-universality. For it implies that you do not need a whole family of functions $(\psi_\epsilon)_{\epsilon > 0}$ (and their translations) to characterise a distribution in $\DLone{\infty}$. A single function $\psi$ (and its translations) suffices: indeed, $\psi$ having compact support, $\supp \Four \psi = \real^d$, thus the kernel $\psi(\hat x - \hat y)$ is $c^\infty_0$-universal. So any distribution $D \in \DLone{\infty}$ is uniquely characterised by $D_y(\psi(\hat x - y))$, the KME of $D$ \gls{wrt} $\psi$. Informally speaking, it means we do not need infinite measuring devices: one device (translated over the input space) suffices\footnote{The physicist however is not out of the woods, for if he wants to perfectly reconstruct $D$, he now needs to perfectly know his magnifying glass $\psi$. And anyone working in deconvolution knows how hard this can be!} to distinguish $D$ from all other natural phenomena modelled by $\DLone{\infty}$. And we are back to usual functions: a function $f$ is characterised by the graph $\function{}{}{}{x}{f(x)}$; an integrable distribution $D$ is characterised by the graph $\function{}{}{}{x}{D_y(\psi(x - y))}$.

	
	We now return to more mathematically rigorous statements and discuss $c^m$-universality.


\subsection{\texorpdfstring{$c^m$}{Cm}-Universal Kernels}

We start with the analogue of Theorem~\ref{theo:GeneralExistence}, but for characteristic kernels over $\Dcomp{m}$. Its proof is literally the same, with $\Mf$ and $\DLone{}$ being replaced by $\Mc$ and $\Dcomp{}$ respectively.

\begin{theorem}\label{theo:GeneralExistenceCompact}
	\textcolor{red}{PROOF FLAWED.} Let $\K$ be a kernel in $\Cont{(m,m)}{}$. If, for any $p \in \nat^d$ with $|p| \leq m$, $\partial^{(p,p)} \K$ is characteristic to $\Mc$, then $\K$ is characteristic to $\Dcomp{m}$.
\end{theorem}

Obviously, any characteristic kernel over $\DLone{m}$ is also characteristic to $\Dcomp{m}$; or if one prefers: any $c^m_0$-universal kernel is also $c^m$-universal. Now, in Corollary~\ref{cor:Existence}, the condition that the support of $\Four \psi$ be all $\real^d$ may seem quite non-restrictive. But it actually prevents some usual kernels from being $c_0$-universal. An example in $\inputS = \real$: the kernel $\sinc(\hat x - \hat y)$ (where $\sinc(x) := \sin(x) / x$ for any $x \not = 0$ and $1$ otherwise). However, the next theorem and the next proposition will show that the $\sinc$ kernel, as the vast majority of stationary continuous kernels used in practice, is $c$-universal (and even $c^\infty$-universal). We thereby refine Proposition~17 of \citet{sriperumbudur10a} and provide a converse. For the proof, see Appendix~\ref{proof:CompactCharacteristic}.
\needspace{5\baselineskip}
\begin{theorem}\label{theo:CompactCharacteristic}
	Let $\K(\hat x,\hat y) = \psi(\hat x-\hat y) \in \Cont{(m,m)}{}$ be a stationary kernel with $\inputS = \real$. For a set $\mathcal S$, we note $\# \mathcal S$ its cardinal. The following two statements are equivalent.
	\begin{enumerate}
		\item $\K$ is not characteristic to $\DLone{m}$.
		\item The support $\mathcal S$ of the distributional Fourier transform $\Four \psi$ is at most countable and, there exists a constant $M \in \real$ such that, for any $r > 0$
			\begin{equation}\label{eq:GrowthOfSupport}
				\# \{ s \in \mathcal S \, | \, |s| \leq r \} \leq M r \, .
			\end{equation}
	\end{enumerate}
\end{theorem}

Equation~\eqref{eq:GrowthOfSupport} upper bounds the average density of the support $\mathcal S$ of $\Four \psi$ in $\real$. 
In other words, to prevent a stationary and smooth enough kernel $\K$ from being characteristic to $\DLone{m}$, the support of $\Four \psi$ needs not only be countable, but its density in $\real$ should also be sufficiently low. Such kernels $\psi$ include for example any periodic or any finite sum of periodic stationary kernels.

	Theorem~\ref{theo:CompactCharacteristic} accounts only for the case $\inputS = \real$. When $\inputS = \real^d$, we expect a similar statement to hold, with $M r$ replaced by $M r^d$. 
This yet remains to be proven. However, we may state and prove (in Appendix~\ref{proof:CompactCharacteristicRd}) the following proposition.
\begin{proposition}\label{prop:CompactCharacteristicRd}
	Let $\K(\hat x,\hat y) = \psi(\hat x-\hat y) \in \Cont{(m,m)}{}$ be a stationary kernel with $\inputS = \real^d$.
	\begin{enumerate}
		\item The kernel is characteristic to $\Dcomp{m}$ \gls{iff} it is characteristic to $\Mc$.\label{p1CC}
		\item If the kernel is not characteristic to $\Dcomp{m}$, then the support of the Fourier transform of $\psi$ is at most countable.\label{p2CC}
	\end{enumerate}
\end{proposition}


	As a final remark before moving to the next section, note that one can give sufficient conditions (for a kernel to be characteristic to $\Dcomp{m}$) that apply to a slightly more general class of kernels than the stationary ones.  They apply to so-called \glsreset{cpd}\gls{cpd} kernels, and their proof uses a generalisation of the Bochner theorem. As this needs too much of an introduction to \gls{cpd} kernels, with rather small benefits in the end, we move these discussions to Appendix~\ref{sec:CPD}. Among other things, these sufficient conditions prove that the Brownian motion kernel $\K(\hat x, \hat y) := \min(|\hat x|, |\hat y|)$ ($x,y \in \real$) is characteristic to the set of measures with compact support.

\section{Topology Induced by \texorpdfstring{$\K$}{k}\label{sec:Topology}}

	
	So far, we embedded various spaces of distributions into $\HK'$ and studied whether these embeddings are injective. We now inquire their continuity properties \gls{wrt} various topologies.

	Each distribution space $\D$ defined in Section~\ref{sec:Definitions} can be endowed with its strong dual, its weak dual or any other topology. Naturally, when embedding $\D$ into $\HK$ one may ask: does it preserve the topological structure of $\D$? More precisely:
\begin{enumerate}[label= (Q\arabic*), leftmargin = 4em, series=questions]
	\item \label{Q1} If a sequence $D_n \in \D$ converges to $D \in \D$, does $\embK(D_n)$ converge to $\embK(D)$? Or more abstractly: is $\embK$ continuous?
	\item \label{Q2} If $\embK(D_n)$ converges to $\embK(D)$ in $\HK$, does $D_n$ converge to $D$?
\end{enumerate}
In this section we answer those questions for various sets of distributions $\D$ (and various topologies). Table~\ref{tab:TopoSummary} summarises these answers.
	
	Previous studies already addressed those questions (\citealp{guilbart78} or more recently \citealp{sriperumbudur10b, sriperumbudur13}). However, all of them focused on $\Mf$ or on subsets like $\M_+$ (positive finite measures) or $\Mprob$. \Citet{sriperumbudur10b} for example inquire on which conditions a kernel $\K$ metrises the narrow convergence topology over $\Mprob$. In other words, they seek positive answers to \ref{Q1} and \ref{Q2} with $\D = \Mprob$, when it is equipped with its narrow topology and $\HK$ with its usual norm. We, on the contrary, start with the general case where $\D$ is a subset of $\Dcomp{m}$ or $\DLone{m}$ and only then focus on subsets of~$\Mf$. 

\newcommand{\Ta}{^{\mathrm P.\scriptstyle{\text{\ref{prop:ContinuousEmbedding}}}}}
\newcommand{\Tb}{^{\mathrm C.\scriptstyle{\text{\ref{cor:Q2onDcomp}}}}}
\newcommand{\Tbb}{^{\mathrm R.\scriptstyle{\text{\ref{rem:Q2onDcomp}~\ref{rem:p2}}}}}
\newcommand{\Tc}{^{\mathrm S.\scriptstyle{\text{\ref{sec:DL1wVsDk}}}}}
\newcommand{\Td}{^{\mathrm{Th}.\scriptstyle{\text{\ref{theo:MetrizationOfNarrow2}}}}}
\newcommand{\Te}{^{\mathrm{Th}.\scriptstyle{\text{\ref{theo:Q2}}}}}
\newcommand{\Tf}{^{\mathrm P.\scriptstyle{\text{\ref{prop:WeakerThanNarrow}}}}}

\begin{table}[t]
	\renewcommand{\arraystretch}{1.5}
	\begin{tabu} to \linewidth {X[$$]X[$$]X[$$]X[$$]X[$$]X[$$]X[$$]X[$$]X[$$]} 
		\hline
		\toprule
							&\F'				&\DLone{m}			&\Dcomp{m} 	&\Dcomp{\infty}	&\Mf				&\Mc				&\M_{\scriptscriptstyle{+}}	&\Mprob \\
		\midrule
			b				&\supset\Ta		&\supset\Ta			&\supset\Ta		&\supset\Ta		&\supset\Ta			&\supset\Ta			&\supset\Ta					&\supset\Ta \\
			b \cap \B		&\supset\Ta		&\supset\Ta			&\supset\Ta		&=\Tb			&\supset\Ta			&\supset\Ta			&\supset\Ta					&\supset\Ta \\
			w				&				&\not \supset\Tc		&				&				&\not \supset\Tc		&					&\not \supset\Tc				&=\Td \\
			w \cap \B		&				&\subset\Te			&=\Tb			&=\Tb			&\not \supset\Tc		&=\Tb				&\not \supset\Tc				&=\Td \\
			\sigma			&\mathrm{n.d.}	&\mathrm{n.d.}		&\mathrm{n.d.}	&\mathrm{n.d.}	&\supset\Tf			&					&\not \supset\Tf				&=\Td \\
			\sigma \cap B	&\mathrm{n.d.}	&\mathrm{n.d.}		&\mathrm{n.d.}	&\mathrm{n.d.}	&\supset\Tf			&=\Tbb				&\not \supset\Tf				&=\Td \\
		\bottomrule
		\hline
	\end{tabu}
	\caption{\label{tab:TopoSummary}Relations between topology $\topo_{\K}$, induced by the kernel metric $d_\K$, and other topologies. Letters $b$, $w$ and $\sigma$ designate the strong dual, the weak dual and the narrow topologies. Th, P, C, S, R and n.d.\ stand for Theorem, Proposition, Corollary, Section, Remark and not defined. For example, Column~3  Line~3 reads: ``Over any bounded subset $\B$ of $\DLone{m}$, and under the assumptions of Proposition~\ref{prop:ContinuousEmbedding}, the strong topology of $\DLone{m}$ dominates $\topo_\K$: $b(\DLone{m},\Cont{m}{0}) \cap \B \supset \topo_\K \cap \B$''. Column~3 Line~4 reads: ``The weak topology of $\DLone{m}$ need not be stronger than $\topo_\K$. See discussions in Section~\ref{sec:DL1wVsDk}''. The two last lines of last column feature this section's main theorem: On a locally compact Hausdorff space $\inputS$, a bounded kernel $\K$ metrises the weak convergence of probability measures iff $\K$ is continuous and characteristic to $\Mprob$. 
}
\end{table}

After a few preliminary remarks and notations (Section~\ref{sec:Preliminaries}), we show that, under very general assumptions on $\D$, the answer to \ref{Q1} is positive for all embeddable distribution sets $\D$, when both $\D$ and $\HK$ are equipped with their strong (in short: $b \rightarrow b$) or with their weak  ($w \rightarrow w$) dual topology (Section~\ref{sec:WwSs}). We then refine those results by analysing the case $w \rightarrow b$ (Section~\ref{sec:Ws}). We then address \ref{Q2}, starting with the general case where $\D$ could be any subset of $\Dcomp{m}$ or $\DLone{m}$. Quickly, however, we focus on bounded subsets of $\Dcomp{m}$ and $\DLone{m}$. In Section~\ref{sec:Narrow}, we then concentrate on the case where $\D$ is $\Mf$ or $\Mprob$, equipped with the narrow topology. We finish with a short investigation in Section~\ref{sec:Surjective} of the questions: ``When is the KME surjective? And when does $\HK$ identify with a (topological) subspace of distributions?''.


\subsection{Preliminary Remarks and Notations\label{sec:Preliminaries}}

In all this section, $\D$ designates an \emph{embeddable} set of distributions. For a subset $\B$ of a set $\S$ with topology $\topo$, we note $\topo \cap \B$ the relative topology induced by $\D$ in $\B$.

\paragraph{Identification of $\HK$, \,$\overline \HK'$ and $\HK'$.} $\HK$ and $\overline \HK'$ are isomorphic, algebraically and topologically, via the Riesz representation map (see Appendix~\ref{sec:Reminders}). As depicted in Diagram~\eqref{eq:EmbeddingMap}, when $\HK$ and $\overline \HK'$ are identified, the KME $\embK$ is the complex conjugate of the restriction map $\function{}{}{}{D}{D\big|_{\HK}}$. But in this section, we will only be studying its continuity properties: they are the same, with or without the complex conjugation. So we will drop all complex conjugation bars and simply \emph{consider that the KME $\embK$ is the restriction map} (without complex conjugation) from $\D$ into $\HK'$. Mathematically speaking, we identify $\overline \HK'$ and $\HK'$ via the complex conjugation map $\function{}{}{}{f}{\bar f}$. In practice, it simply means that instead of writing for instance $b(\overline \HK', \overline \HK)$, $w(\overline \HK', \overline \HK)$ or $\bar D_\alpha \rightarrow \bar D$, we will write $b(\HK', \HK)$, $w(\HK',\HK)$ and $D_\alpha \rightarrow D$. It will be useful for this section to always remember that, when $\HK$ continuously embeds into $\F$, then \emph{the KME $\embK$ is the transpose $\imath^\star$ of the canonical embedding map $\function{\imath}{\HK}{\F}{f}{f}$} and coincides with the restriction map of $\F'$ to $\HK$. In short: if $\HK \csubset \F$, then $\embK = \imath^\star$.

\paragraph{Metric and strong topology induced by $\K$.} The strong topology induced by $\K$ in $\D$ is simply the topology induced by the semi-metric $d_\K$ defined by: for any $D,T \in \D$,
\begin{equation*}
	d_\K(D,T) := \normK{\embK(D) - \embK(T)} \, .
\end{equation*}
We will call $d_\K$ the \emph{kernel(-induced) (semi-)metric}. When $\D$ is a vector space, this semi-metric is induced by the semi-inner product
\begin{equation*}
	\ipdK{D}{T} := \ipdK{\embK(D)}{\embK(T)} \, .
\end{equation*}
The ``semi'' means that the topology induced by $d_\K$ may not be Hausdorff: $d_\K(D,T) = 0$ may not imply $D=T$. $d_\K$ is a (non-degenerate) metric \gls{iff} the KME is injective. If it is, we note $b(\HK',\HK) \cap \D$ the induced topology. A net\footnote{A net is a generalisation of a sequence, in which the directed index set $\mathcal A$ need not be countable.} $(D_\alpha)_{\alpha \in \mathcal A}$ in $\D$ converges to $D \in \D$ for the kernel semi-metric \gls{iff} $\normK{D_\alpha - D} \rightarrow 0$.

\paragraph{Weak topology induced by $\K$.} We define the weak topology induced by $\K$ over an embeddable set of distributions $\D$ by its convergent nets. A net $D_\alpha$ in $\D$ converges to $D$ in $\D$ for the weak topology induced by $\K$ \gls{iff} for any $f \in \HK$, $D_\alpha(f) \rightarrow D(f)$. Said differently, the weak topology induced by $\K$ over $\D$ is the topology of pointwise convergence (with points) in $\HK$. Here again, this topology is Hausdorff \gls{iff} the KME is injective over $\D$. If it is, we note $w(\HK',\HK) \cap \D$ this topology.

\paragraph{Induced topologies, continuity of $\embK$ and metrisation.} Remember that the strong (resp.\ weak) topology induced by $\K$ is weaker than the original topology $\topo$ of $\D$ \gls{iff} the KME $\embK$ is continuous when $D$ is equipped with $\topo$ and $\HK$ with its RKHS norm (resp.\ with its weak topology). We then write $\embK : \D_{\topo} \csubset (\HK')_b$ (resp.\ $\embK : \D_{\topo} \csubset (\HK')_w$\,). When the topology induced by $d_\K$ in $\D$ equals the original topology of $\D$, $\K$ is said to \emph{metrise the topology of $\D$}. In that case, a net $D_\alpha$ converges to $D$ in $\D$ for the original topology of $\D$ \gls{iff} $\normK{D_\alpha - D} \rightarrow 0$.

\paragraph{Continuity and sequential continuity.} Let $X$ and $Y$ be two arbitrary topological spaces and $f$ a function from $X$ to $Y$. The function $f$ is continuous \gls{iff} for any net $x_\alpha$ that converges to $x$ in $X$, $f(x_\alpha)$ converges to $f(x)$ in $Y$. In particular, if $f$ is continuous, then $f$ is also \emph{sequentially continuous}: for any sequence $x_n$ of $X$, if $x_n \rightarrow x$, then $f(x_n) \rightarrow f(x)$. If $X$ is a so-called \emph{sequential space}, then sequential continuity also implies continuity. But in general, the converse is wrong. Fortunately, many spaces are sequential. For example, any metric space ---in particular, any Banach or Fréchet space--- is sequential. Unfortunately, \gls{lcv} TVSs need not be sequential. For instance, $\DLone{m}$ is sequential \gls{iff} $m$ is finite. Thus, in all this section, we will carefully distinguish continuity and sequential continuity. In particular, question~\ref{Q1} asks whether the KME is sequentially continuous over a space $\D$. For practical purposes, that is all what matters. Nevertheless, we will mainly establish (non-sequential) continuity properties of the KME. 
The reader should keep in mind that they imply sequential continuity. In short: if $\embK : \D_\topo \csubset \HK$, then \ref{Q1} is answered positively. A similar remark holds for~\ref{Q2}. Indeed, let $\topo_1$ and $\topo_2$ be two topologies over a space $X$. Obviously, when $\topo_1$ is stronger than $\topo_2$, then $\topo_1$ is also sequentially stronger than $\topo_2$: if $x_n \rightarrow x$ in $\topo_1$ then $x_n \rightarrow x$ in $\topo_2$. From metric spaces, we are used to the converse being true: if $\topo_1$ is sequentially stronger than $\topo_2$, then $\topo_1$ is stronger than $\topo_2$. However, when $\topo_2$ is not sequential (meaning that $X$ equipped with $\topo_2$ is not a sequential space), this converse may fail to hold: sequentially stronger need not imply stronger. Coming back to question~\ref{Q2}, we see it asks whether the topology induced in $\D$ by $\embK$ is \emph{sequentially} stronger than the initial topology of $\D$. There again, for practical purposes, that is all what matters. Nevertheless, we will mainly establish statements such as: $\topo_1$ is stronger than (we also say \emph{dominates}) $\topo_2$, in which case we write $\topo_1 \supset \topo_2$.

\subsection{Embeddings of Dual Spaces are Continuous\label{sec:WwSs}}

To start, let us remind how one usually embeds $\Mf$ into the space of distributions $\Dall{\infty}$. One simply notes that the space $\Cont{\infty}{c}$ is continuously contained and dense in $\Cont{}{0}$. Thus the transpose $\imath^\star$ of the canonical embedding map $\function{\imath}{}{}{f}{f}$ defines an injective linear map from the dual $\Mf$ of $\Cont{}{0}$ into the dual $\Dall{\infty}$ of $\Cont{\infty}{0}$. (Note that $\imath^\star$ is the restriction map $\function{}{}{}{\mu}{\mu \big |_{\Cont{\infty}{c}}}$.) Thus identifying $\Mf$ with its image in $\Dall{\infty}$, one says that $\Mf$ is a subspace of $\Dall{\infty}$. But not only is $\imath^\star$ injective, it is also continuous, when both duals are equipped with their strong or with their weak topology. (That is why, when $\Mf$ is equipped with its strong topology, the total variation norm, $\Mf$ is not only a space of distributions, it is also a \emph{topological} space of distributions.) This is simply an instance of a more general fact: if a \gls{lcv} TVS $\F_1$ embeds continuously into another $\F_2$ via a linear map $\imath$, then the transpose $\imath^\star$ is a continuous linear embedding from $\F_2'$ into $\F_1'$, when $\F_1'$ and $\F_2'$ carry either both their weak or both their strong dual topology \citep[Proposition~19.5 and Corollary]{treves67}. Applying this to $\F_2 = \HK$ yields:



\begin{proposition}\label{prop:ContinuousEmbedding}
	Let $\inputS$ be any topological space (not necessarily locally compact and Hausdorff). If $\HK$ continuously embeds into $\F$, then the KME $\function{\embK}{\F'}{\HK'}{}{}$ is continuous, when $\F'$ and $\HK'$ carry either both their strong or both their weak topologies.
	In short:
	\begin{equation}\label{eq:ContinuousEmbedding}
		\HK \csubset \F \qquad \Longrightarrow \qquad \left \{
			\begin{array}{lll}
				\embK : \F'_b \csubset (\HK')_b \\
				\embK : \F'_w \csubset (\HK')_w
			\end{array}
		\right . \, .
	\end{equation}
\end{proposition}
The conclusion of Proposition~\ref{prop:ContinuousEmbedding} concerning continuity \gls{wrt} weak topologies is actually trivial! It states that for a net $D_\alpha$ and $D$ in $\F'$, if $D_\alpha(f) \rightarrow D(f)$ for any $f \in \F$, then $D_\alpha(f) \rightarrow D(f)$ for any $f \in \HK$ (which is clear, because $\HK \subset \F$\,!). The conclusion concerning strong continuity is more interesting, although not much more difficult to prove. It says: if $D_\alpha \rightarrow D$ in $\F'_b$, then $\normK{D_\alpha - D} \rightarrow 0$. Thus, when $\HK$ embeds continuously into $\F$, then Proposition~\ref{prop:ContinuousEmbedding} answers \ref{Q1} positively, when $\F'$ and $\HK'$ are equipped both with their strong or both with their weak topologies.

\begin{example}
	We illustrate Proposition~\ref{prop:ContinuousEmbedding} on spaces $\F$ that are not listed in Section~\ref{sec:Definitions}. Let $\K$ be a bounded kernel. Equip $\inputS$ with the topology induced by the kernel-induced semi-distance $\rho_\K$, defined as: for any $x,y \in \inputS$, $\rho_\K(x,y):= d_\K(\delta_x, \delta_y) = \normK{\delta_x - \delta_y}$. ($\inputS$ may or may not be locally compact and Hausdorff.) Then $\HK$ continuously embeds into any of the following spaces:
	\begin{description}
		 \item[$\Cont{}{b}$] the space of bounded $\rho_\K$-continuous functions, equipped with the supremum norm $\norm{f}_\infty$.
		 \item[$\Cont{}{L} := \{ f \in \Func \, | \, \norm{f}_L < \infty \}$] the space of Lipschitz continuous functions, equipped with the Lipschitz-norm $\norm{f}_L := \sup \{ |f(x) - f(y)| / \rho_\K(x,y) \, | \, x,y \in \inputS\,, \rho_\K(x,y) \not = 0 \}$.
		 \item[$\Cont{}{BL} := \{ f \in \Func \, | \, \norm{f}_{BL} < \infty \}$] the space of bounded and Lipschitz continuous functions, equipped with the BL-norm $\norm{f}_{BL} := \norm{f}_\infty + \norm{f}_L$.
	\end{description}
	The duals of these spaces contain all probability measures. The respective metrics induced by their dual norms on $\Mprob$ are: the total variation metric, the \emph{Kantorovich metric} (which coincides with the \emph{Wasserstein metric} when $\inputS$ is separable) and the \emph{Dudley metric}. Proposition~\ref{prop:ContinuousEmbedding} shows, without any computation, that all these metrics dominate the kernel-induced semi-metric $d_\K$.
\end{example}

\subsection{When Does \texorpdfstring{$\F'_w$}{F'_w} Continuously Embed into \texorpdfstring{$\HK'$}{Hk'}?\label{sec:Ws}} 

Proposition~\ref{prop:ContinuousEmbedding} answers \ref{Q1} positively when $\F$ embeds continuously into $\HK$, and when both $\F'$ and $\HK'$ are equipped with their strong (resp.\ weak) dual topology. The weak dual topology being weaker than the strong one, under the assumption of Proposition~\ref{prop:ContinuousEmbedding} we obviously also have $\embK : \F' \csubset (\HK')_w$. How about $\embK : \F'_w \csubset \HK'$\ ? This will be the focus now.

The answer depends on the specific space $\F'$, for the weak or strong dual topologies may vary from one dual to another. The weak (resp.\ strong) dual topology of $\Dcomp{m}$, for instance, is strictly stronger than the weak (resp.\ strong) topology it carries when seen as a subspace of $\DLone{m}$ (resp.\ $(\DLone{m})_w$). As we will see, it will also depend on some assumptions made on the kernel.

\subsubsection{Does \texorpdfstring{$(\DLone{m})_w$}{(Dm)_w} Continuously Embed into \texorpdfstring{$\HK'$}{Hk'}?\label{sec:DL1wVsDk}}
	
	In general no: $(\DLone{m})_w$ need not continuously embed into $\HK'$. For instance, suppose $\K$ is stationary and take a sequence $x_n \in \inputS$ such that $x_n \rightarrow \infty$. Then $\delta_{x_n}$ converges to $0$ in $(\DLone{m})_w$, but $\normK{\delta_{x_n}}^2 = \K(x_n, x_n) = \K(x_1,x_1)$. Thus $\delta_{x_n}$ does not converge in the strong topology induced by $\K$.

\subsubsection{Does \texorpdfstring{$(\Dcomp{m})_w$}{(Em)_w} Continuously Embed into \texorpdfstring{$\HK'$}{Hk'}?}
	
	Under mild smoothness assumptions on $\K$, yes: $(\Dcomp{m})_w$ continuously embeds into $\HK'$ ---at least sequentially. 
More precisely, inspired by a remark from \citet{matheron73}, we get the following two theorems.

\begin{theorem} \label{theo:Ws0}
	The following six conditions are equivalent.
	\begin{enumerate}
		\item $\K$ is continuous. \label{Ws01}
		\item $\function{\K(.,\hat{x})}{\inputS}{\HK'}{x}{\embK(\delta_x) = \K(.,x)}$ is continuous.\label{Ws02}
		\item $\function{\embK}{(\Mc)_{w}}{\HK'}{}{}$ is continuous over the bounded sets of $(\Mc)_w$. \label{Ws04}
		\item $\function{\embK}{(\Mfs \cap \M_{\scriptscriptstyle{+}})_\sigma}{\HK'}{}{}$ is continuous. \label{Ws07}
		\item $\function{\embK}{(\Mprob)_\sigma}{\HK'}{}{}$ is continuous. \label{Ws05}
		\item $\function{\embK}{(\M_{\scriptscriptstyle{+}})_\sigma}{\HK'}{}{}$ is continuous. \label{Ws06}
	\end{enumerate}
	If these conditions are satisfied, all these maps are sequentially continuous.
\end{theorem}
The $\sigma$ in \ref{Ws05} and \ref{Ws06} refer to the narrow topology $\sigma(\Mf, \Cont{}{b})$, which will be studied in more details in Section~\ref{sec:Narrow}.
For the general case $\D = \Dcomp{m}$, we state:
\begin{theorem} \label{theo:Ws}
	If $\K$ is $(m,m)$-times continuously differentiable, then the embedding $\function{\embK}{(\Dcomp{m})_w}{\HK'}{}{}$ is continuous over the bounded sets of $(\Dcomp{m})_w$. In particular, $\embK$ is sequentially continuous over $(\Dcomp{m})_w$.
\end{theorem}
\begin{proof}{[of Theorems~\ref{theo:Ws0} and \ref{theo:Ws}]}
	First, a word on Theorem~\ref{theo:Ws0}. That \ref{Ws01} and \ref{Ws02} are equivalent is proven for example by \citet[Theorem~4.29]{steinwart08}. Furthermore, it is clear that any of the points \ref{Ws04} to \ref{Ws06} imply \ref{Ws02} and that \ref{Ws06} $\Rightarrow$ \ref{Ws05} $\Rightarrow$ \ref{Ws07}. Thus, we are left with proving that \ref{Ws01} implies both \ref{Ws04} and \ref{Ws06}. The proof of these implications as well as of Theorem~\ref{theo:Ws} follow the same proof-schema, which we illustrate on the proof of \ref{Ws01} $\Rightarrow$ \ref{Ws04}.
	
\begin{enumerate}
	\item Let $\mu_\alpha$ be a net contained in a bounded subset $\mathcal B$ of $\Dcomp{m}$ such that $\mu_\alpha \rightarrow 0$ in $w(\Mc,\Cont{}{}) \cap \mathcal B$.
	\item Show that, because $\K \in \Cont{}{}(\inputS \times \inputS)$, the map
		\begin{equation*}
			\function{B}{(\Mc(\inputS))_w \times (\Mc(\inputS))_w}{(\Mc(\inputS \times \inputS))_w}{(\mu,\nu)}{\mu \otimes \nu}
		\end{equation*}
		is (weakly) continuous when restricted to $\mathcal B \times \mathcal B$.
	\item Conclude that $\normK{\mu_\alpha}^2 = [\mu_\alpha \otimes \bar \mu_\alpha](\K) \rightarrow 0$.
\end{enumerate}
Appendix~\ref{sec:Proofs} details how to adapt the input and output sets of $B$ to prove each remaining implication. The continuity of each such map $B$ is also proven there. Finally, continuity over bounded sets implies sequential continuity (over the entire space $\D$), because any convergent sequence is bounded. (In general, however, it need not imply continuity over unbounded sets.)
\end{proof}

\begin{remark}
	Theorem~\ref{theo:Ws0} shows that even if $\Mc$ embeds into $\HK'$, the embedding may not be continuous from $(\Mc)_w$ into $\HK'$. Indeed, there exist non-continuous kernels such that $\HK \subset \Cont{}{}$ \citep[see][]{lehto52}. With Proposition~\ref{prop:ContinuousRKHS}, this implies $\HK \csubset \Cont{}{}$, thus $\Mc$ embeds into $\HK'$. But as $\K$ is not continuous, $(\Mc)_w \csubset \HK'$ does not hold.
\end{remark}

	In Theorem~\ref{theo:Ws0}, one may replace $(\Mc)_w$ by its dense subset $\Mfs$ equipped with the induced weak topology $w(\Mc, \Cont{}{}) \cap \Mfs$. This shows we could have defined the KME as follows.
	\begin{enumerate}
		\item Embed $\Mfs$ into $\HK$ ;
		\item Note that this embedding is continuous when $\Mfs$ is equipped with the weak topology induced by $\Mc$;
		\item Extend this embedding by continuity to $\Mc$. Theorem~\ref{theo:Ws0} then guarantees that this extension coincides with our previously defined KME.
	\end{enumerate}
	This technique of continuous extension has often been used as a starting point to embed measures into $\HK$ (\citealp{matheron73, guilbart78} or \citealp[Chapter~4]{berlinet04}). The following proposition shows that for $\Dcomp{\infty}$, this continuous extension technique leads to the same KME definition as ours. Indeed, $\Mfs$ is dense in $\Dcomp{\infty}$, when equipped with the weak or strong topology. Thus Theorem~\ref{theo:Ws} immediately yields

\begin{proposition}\label{prop:UniqueExtension}
	Let $\K \in \Cont{(\infty,\infty)}{}$. The KME of $\Dcomp{\infty}$ is the unique continuous linear extension of the KME of $\Mfs$, when $\Mfs$ is equipped with the (weak or strong) topology induced by $\Dcomp{\infty}$.
\end{proposition}	

\subsection{When does \texorpdfstring{$\K$}{K} metrise the topology of \texorpdfstring{$\F'$}{F'}?\label{sec:Q2}}

The two preceding sections dealt solely with question \ref{Q1}. 
They inquired under which conditions the weak and strong topologies induced by $\K$ in a given dual $\F'$ are weaker than the original topology.  To answer question \ref{Q2}, we now investigate under which conditions they are stronger.

	First, the kernel \emph{must} be characteristic to $\F'$. If it was not, the weak and strong kernel induced topologies would not be Hausdorff (and the semi-metric $d_\K$ not be a metric), thus could not be stronger than any topology we will consider. 
So from now on, we assume that $\K$ is characteristic to the set to embed.

	Second, we will suppose that $\F$ is barrelled (see Definition in Appendix~\ref{sec:Reminders}). This is merely a technical assumption used to apply the Banach-Steinhaus theorem (see Appendix~\ref{sec:Reminders}, Theorem~\ref{theo:BS}) in the proof of Theorem~\ref{theo:Q2}. In practice, it is very unrestrictive: any Banach, Fréchet or Limit-Fréchet\footnote{Example of Limit-Fréchet spaces which are not Fréchet spaces are $\Cont{m}{c}$ and $\Dcomp{\infty}$ \citep{schwartzTD}.} space is barrelled. In particular, \emph{all function spaces defined in Section~\ref{sec:Definitions} are barrelled, except $(\Cont{m}{b})_c$ and perhaps $\Func$.}
	
	Third, we will focus on 
	bounded subsets of $\F'$. Let us justify why. 
First, if a sequence converges in a Hausdorff topology, then it is bounded. Thus, two topologies that coincide on every bounded set define the same convergent sequences. Second, one of the most used sets of distributions is bounded: the set of probability measures $\Mprob$. (Indeed, the topology of $\Mf$ is given by the total variation norm $\norm{.}_{TV}$ and, for any $P \in \Mprob$, $\norm{P}_{TV} =  \int \diff P = 1$.) But would $\Mprob$ still be bounded, had $\Mf$ been equipped with its weak or with its narrow topology? Yes. This leads to a third good reason to consider bounded sets: \emph{if $\F$ is barrelled, then the Banach-Steinhaus theorem implies that all topologies stronger than the weak but weaker than the strong dual topology define the same bounded sets over $\F'$}. 
In other words, for any such topology, if $\F$ is barrelled, a subset $\mathcal B$ of $\F'$ is bounded iff
\begin{equation}\label{eq:BoundedSets}
	\forall f \in \F \,, \qquad \sup_{D \in \mathcal B} |D(f)| < \infty \, .
\end{equation}
The narrow topology is clearly such a topology: it is stronger than the weak-star, but weaker than the strong dual topology. 

To introduce our next theorem, consider a space of finite Borel measures defined on a locally compact Hausdorff space $\inputS$. There are two common options to construct $\Mf$ as a dual space. One is to consider the space $\Cont{}{c}$ endowed with the topology of uniform convergence. The other is to consider its completion, $\Cont{}{0}$. Their duals both identify with $\Mf$. But do their weak dual topologies coincide? If $\inputS$ is not compact, they do not.\footnote{Proof: Equipped with $w(\Mf, \Cont{}{0})$, the dual of $\Mf$ is $\Cont{}{0}$. But equipped with $w(\Mf, \Cont{}{c})$ it is $\Cont{}{c}$.} The Banach-Steinhaus theorem, however, guarantees that, at least, they coincide on all bounded subsets of $\Mf$. The following theorem, proved in Appendix~\ref{proof:Q2}, states essentially the same, but with $\HK$ and $\F$ instead of $\Cont{}{c}$ and $\Cont{}{0}$.

\begin{theorem}\label{theo:Q2}
	Let $\F$ be a barrelled space such that $\HK \csubset \F$ and that $\K$ be universal over $\F$. On any bounded subset $\B$ of $\F'$, the weak-star topology coincides with the weak topology induced by $\K$. Thus the topology induced by the semi-metric $d_\K$ dominates the weak-star topology over any bounded subset of $\F'$ and
	\begin{equation}\label{eq:Q2s}
		b(\F',\F) \cap \B \ \supset \ b(\HK',\HK) \cap \B \ \supset \ w(\F',\F) \cap \B \ = \ w(\HK',\HK) \cap \B \, .
	\end{equation}
\end{theorem}
Theorem~\ref{theo:Q2} may seem abstract, but for sequences, it simply means the following.
\begin{corollary} \label{cor:Q2}
	Let $\F$ and $\K$ be as in Theorem~\ref{theo:Q2}. Let $D_n$ be a bounded sequence in $\F'$, and $D \in \F'$.
	\begin{enumerate}
		\item $D_n(f) \rightarrow D(f)$ for any $f \in \F$ \gls{iff} $\bar D_n(f) \rightarrow \bar D(f)$ for any $f \in \HK$.
		\item If $\normK{D_n - D} \rightarrow 0$\,, then $D_n(f) \rightarrow D(f)$ for any $f \in \F$.
		\item If $D_n \rightarrow D$ in $b(\F',\F)$, then $\normK{D_n - D} \rightarrow 0$.
	\end{enumerate}
\end{corollary} 
With $\F = \Cont{m}{}$, we get the following important and stunning corollary.
\begin{corollary} \label{cor:Q2onDcomp}
	Let $\K$ be a kernel in $\Cont{(m,m)}{}$, and let $\mathcal B$ be a bounded subset of $\Dcomp{m}$. Then
	\begin{equation}\label{eq:Q2bis}
		b(\F',\F) \cap \B \ \supset \ b(\HK',\HK) \cap \B \ = \ w(\F',\F) \cap \B \ = \ w(\HK',\HK) \cap \B \, .
	\end{equation}
	In particular, on bounded subsets of $\Dcomp{m}$, the following statements are equivalent.
	\begin{enumerate}[series=equiv]
		\item For any $f \in \F$\,, $D_n(f) \rightarrow D(f)$. \label{conv1}
		\item For any $f \in \HK$\,, $D_n(f) \rightarrow D(f)$. \label{conv2}
		\item $\normK{D_n - D} \rightarrow 0$. \label{conv3}
	\end{enumerate}
	Moreover, when $m = \infty$, the inclusion in \eqref{eq:Q2bis} is an equality and \ref{conv1}~-~\ref{conv3} are equivalent to:
	\begin{enumerate}[resume=equiv]
		\item $D_n \rightarrow D$ in $b(\Dcomp{\infty},\Cont{\infty}{})$. \label{conv4}
	\end{enumerate}
\end{corollary}
\begin{proof}
	Equality~\eqref{eq:Q2bis} is a straightforward consequence of Theorems~\ref{theo:Ws} and~\ref{theo:Q2}. The statements concerning the special case $m=\infty$ stem from the fact that, because $\Dcomp{\infty}$ is a Montel space, on bounded subsets, its weak and strong dual topologies coincide.
\end{proof}

Before closing this subsection, we group a few remarks.
\needspace{2\baselineskip}
\begin{remark}\label{rem:Q2onDcomp}
	\leavevmode
	\begin{enumerate}[label = (\roman*), itemsep=1ex]
		\item Two topologies that coincide on every bounded subset of a topological space $X$ need not coincide on the whole space $X$. However, as any convergent sequence is bounded, they both define the same convergent sequences. (Thus, if both topologies are sequential and coincide on the bounded subsets of $X$, then they coincide on the whole space $X$.)
		\item A subset $\mathcal B$ of $\Mf$ is bounded \gls{iff} it is uniformly bounded for the total variation norm, which is equivalent to the seemingly weaker condition~\eqref{eq:BoundedSets}. It is bounded in $\Mc$ \gls{iff} it is bounded in $\Mf$ (or even $\Mr$) and if all its elements have their support in a common compact subset of $\inputS$. In that case, $w(\Mc, \Cont{}{}) \cap \B = \sigma(\Mc, \Cont{}{b}) \cap \B$. More generally, $\mathcal B$ is bounded in $\Dcomp{m}$ \gls{iff} it is bounded in $\DLone{m}$ (or even $\Dall{m}$) and if the support of its elements are all contained in a common compact subset of $\inputS$.\label{rem:p2}
		\item In Theorem~\ref{theo:Q2}, $\F$ continuously embeds into $\HK$, thus $\F'$ continuously embeds into $\HK'$. So bounded subsets of $\F'$ are also bounded for the kernel-induced metric. In general, the converse is false: if a subset $\B$ of $\F'$ is $d_\K$-bounded, it need not be bounded in the original topology of $\F'$. Consider for example an unbounded sequence $x_n$ in $\inputS=\real^d$ and a bounded kernel $\K$. Then $\embK(\delta_{x_n})$ is bounded in $\HK$\,, but $\delta_{x_n}$ is not bounded in $\Mc$\,, because the supports of the measures $\delta_{x_n}$ are not contained in a common compact subset of $\inputS$.
		\item Corollary~\ref{cor:Q2onDcomp} states that on bounded subsets $\B$ of $\Dcomp{m}$, the weak and strong topologies of $\HK$ coincide. This may seem surprising, as, in an infinite dimensional Hilbert space, even on bounded sets, weak and strong topologies in general do not coincide. For example, Bessel's inequality shows that any infinite orthonormal sequence converges weakly to 0. But it certainly does not converge strongly to 0. However, there again, $\B$ is not any bounded subset of $\HK$~: it is both included in $\Dcomp{m}$ and bounded for its topology (which is stronger than that of $\HK$~!).
	\end{enumerate}
\end{remark}

So far, we focused on the two most common dual topologies: the weak and the strong dual topologies. Probability theory, however, certainly uses the narrow topology $\narrow(\Mf, \Cont{}{b})$ even more. So let us see how it compares with the topologies induced by $\K$.

\subsection{When does \texorpdfstring{$\K$}{K} Metrise the Narrow Convergence?\label{sec:Narrow}}

A kernel $\K$ metrises the narrow topology on a set $\B \subset \Mf(\inputS)$ if the strong topology induced by $\K$ coincides on $\B$ with the narrow topology; i.e., if: $b(\HK',\HK) \cap \B = \narrow(\Mf, \Cont{}{b}) \cap \B$. When $\B$ is a bounded set consisting of \emph{positive} measures over a metrisable and separable (not necessarily locally compact) domain $\inputS$, the narrow convergence topology is metrisable\footnote{They are for example metrised by the Dudley or the Prokhorov metrics \citep[see][]{dudley02}.}. In that case, it is entirely defined  on $\B$ by its convergent sequences, and we can equivalently define: $\K$ metrises the narrow topology on a bounded set $\B \subset \Mf$ of positive measures iff: for any $\mu,\mu_1,\mu_2, \ldots \in \Mf(\inputS)$, $\mu_n \rightarrow \mu$ narrowly \gls{iff} $\embK(\mu_n) \rightarrow \embK(\mu)$ in RKHS norm.

The quest for kernels that do metrise the narrow topology dates back at least to the 1970s. Criteria were already given by \citet{guilbart78} \citep[see also][Theorems~100 and 101]{berlinet04}. More recently, \citet[Theorem~23]{sriperumbudur10b}\footnote{Note that they call ``weak'' what we call ``narrow''.} pointed out that \emph{on compact domains $\inputS$}, any $c_0$-universal kernel metrises the narrow convergence topology on $\Mprob$. Putting together Theorems~\ref{theo:Ws0} and \ref{theo:Q2} immediately yields a similar result (Corollary~\ref{cor:MetricOnCompact}). But things become trickier on unbounded domains. As we will see, the narrow topology is in general stronger than (the topology of) the kernel-induced metric $d_\K$ (Proposition~\ref{prop:WeakerThanNarrow}), and often \emph{strictly} stronger (Proposition~\ref{prop:NonMetrization}). But there is an important exception: when focusing on probability measures (Proposition~\ref{prop:MetrizationOfNarrow})! 

We start with the case when the input set $\inputS$ is compact. Then $\Mf(\inputS) = \Mc(\inputS)$ and $w(\Mf,\Cont{}{0}) = w(\Mc,\Cont{}{})$. Thus Theorems~\ref{theo:Ws0} and \ref{theo:Q2} give:
\begin{corollary}\label{cor:MetricOnCompact}
	Let $\inputS$ be a compact Hausdorff space. A kernel $\K$ metrises the narrow topology on any bounded set of $\Mf$ \gls{iff} it is continuous and $c_0$-universal. In particular, a continuous $c_0$-universal kernel metrises the narrow topology over $\Mprob$.
\end{corollary}
\begin{proof}
	If $\K$ is continuous and bounded, then $\HK \csubset \Cont{}{}$. Thus, Theorem~\ref{theo:Q2} shows that the kernel metric $d_\K$ is stronger than the narrow convergence topology, and Theorem~\ref{theo:Ws0} that it is weaker. So both topologies coincide. Conversely, if both topologies coincide, then $\K$ must be characteristic to $\Mf$ (for $d_\K$ to be metric) and the implication $\ref{Ws05} \Rightarrow \ref{Ws01}$ of Theorem~\ref{theo:Ws0} shows that $\K$ is continuous. 
\end{proof}

	We now turn to the general case, where $\inputS$ is a locally compact Hausdorff set. Reformulating Theorem~\ref{theo:Ws0} (which also holds, even if $\inputS$ is not paracompact), we immediately get:
\begin{proposition}[$(\Mf)_\sigma \csubset \HK$]\label{prop:WeakerThanNarrow}
	Let $\inputS$ be a locally compact Hausdorff space and let $\K$ be a bounded continuous kernel on $\inputS$. The topology induced by the semi-metric $d_\K$ in $\Mf$ is weaker than the narrow topology. 
\end{proposition}
There exists an analogue of Proposition~\ref{prop:WeakerThanNarrow} when, instead of being locally compact, $\inputS$ is a separable metric space (see Appendix~\ref{proof:WeakerThanNarrowPolish}).

Summarising, we see that, whether $\inputS$ be a locally compact Hausdorff or a separable metric space, the narrow topology is in general stronger than the kernel-induced metric. Conversely, \citet{guilbart78} showed that, on any separable metric space, there exist kernels that do metrise the narrow topology over the set of finite \emph{positive} measures $\M_{\scriptscriptstyle{+}}$. However, the following proposition (proved in Appendix~\ref{proof:NonMetrization}) shows that being continuous and characteristic to $\Mf$ is not sufficient.
\begin{proposition}\label{prop:NonMetrization}
	Let $\inputS = \real^d$ and let $\K$ be a continuous stationary kernel $\K(\hat x, \hat y) = \psi(\hat x - \hat y)$ such that $\psi \in \Cont{}{0}$. 
	Then $\K$ does not metrise the narrow convergence topology over any subset of $\Mf(\real^d)$ that contains $\Mprob \cup \{ 0 \}$.
\end{proposition}
Examples of such sets include $\Mf(\real^d)$, $\Mf^{\scriptscriptstyle{\leq 1}}(\real^d)$, $\M_{\scriptscriptstyle{+}}^{\scriptscriptstyle{\leq 1}}(\real^d)$, but not $\Mprob(\real^d)$. We do not know whether this result extends to  any (possibly non-stationary) kernel such that $\HK \csubset \Cont{}{0}(\real^d)$.

Proposition~\ref{prop:NonMetrization} is somewhat disappointing. On the one side, kernels satisfying its assumptions~---Gaussian kernels for example--- are among the most used kernels in machine learning and geostatistics. They have several nice properties. For instance, the norm associated to a measure does not depend on the choice of origin in $\inputS$. On the other side, they do not metrise the narrow topology, probably the most wide-spread topology in probability theory and statistics. These fields, however, focus on a very specific subset of $\Mf$: \emph{probability} measures. For those, we will now show that \emph{any} continuous  and characteristic kernel metrises the narrow convergence. We consider this the second main result of this paper.

\begin{theorem}\label{theo:MetrizationOfNarrow2}
	A bounded kernel $\K$ over a locally compact Hausdorff space $\inputS$ metrises the narrow topology over $\Mprob$ \gls{iff} it is continuous and characteristic (to $\Mprob$). In particular, if $\K$ is continuous and \gls{ispd} (or continuous and $c_0$-universal), then $\K$ metrises the narrow topology over $\Mprob$. 
\end{theorem}
\begin{proof}{[Theorem~\ref{theo:MetrizationOfNarrow2}]}
	If $\K$ metrises the narrow topology over $\Mprob$, then $\K$ is obviously characteristic, and Theorem~\ref{theo:Ws0} shows that $\K$ is continuous. Conversely, if $\K$ is continuous, then by Theorem~\ref{theo:Ws0}, the kernel metric on $\Mprob$ is weaker than the narrow topology. Thus it suffices to show that $w(\Mf, \HK)$ is stronger than the narrow topology over $\Mprob$. Consider the kernel $\K_1 := \K + 1$. From Theorem~\ref{theo:CKCharacterisation} it follows that $\K_1$ is \gls{ispd}, thus Lemma~\ref{lem:WandW} (see Appendix~\ref{proof:WandW}) shows that $\K_1$ metrises the vague topology $w(\Mprob, \Cont{}{c})$. But on $\Mprob$, the vague and the narrow topology coincide \citep[Chapter~2, Corollary~4.3]{berg84}. Thus $\K_1$ metrises the narrow topology over $\Mprob$. And $\K_1$ and $\K$ induce the same metric in $\Mprob$.
\end{proof}

To the best of our knowledge, so far, the most general sufficient conditions for a kernel to metrise the narrow topology over $\Mprob$ were given by Theorem~2 of \citet{sriperumbudur13}. It actually deals with families of kernels, but when applied to a single kernel, it states: on a Polish\footnote{A \emph{Polish space} is a complete, separable, metric space.}, locally compact (Hausdorff) space, any continuous, $c_0$-universal kernel metrises the narrow convergence topology over $\Mprob$. Let us compare it to Theorem~\ref{theo:MetrizationOfNarrow2}. A kernel $\K$ is continuous and $c_0$-universal \gls{iff} $\K$ belongs to $\Cont{(0,0)}{0}$ and is characteristic to $\Mf$. Both these assumptions are now relaxed: instead of vanishing at infinity, $\K$ only needs to be bounded; instead of characteristic to $\Mf$, it only needs to be characteristic to $\Mprob$. Thanks to these weaker assumptions, Theorem~\ref{theo:MetrizationOfNarrow2} can be stated as an \gls{iff} statement. 
The other major difference is that we relax the assumption that $\inputS$ be Polish, and keep only the local compact Hausdorff assumption. Doing so, we align with the trend promoted by Bourbaki to do probability theory on locally compact Hausdorff rather than Polish spaces. Wether 
we could have relaxed the local compact Hausdorff and kept the Polish assumption instead remains the main and most important open question of this paper. 

\subsection{Is \texorpdfstring{$\HK'$}{Hk'} a Space of Distributions?\label{sec:Surjective}}

So far, we focused on the two (complementary) questions: when is the KME injective over a set of distributions $\D$? And when does $\K$ metrise the topology of $\D$? We now briefly examine the two converse questions:
\begin{enumerate}[resume*=questions]
	\item When is the KME surjective? \label{Q3}
	\item And if it is, does $\K$ metrise the topology of $\D$? \label{Q4}
\end{enumerate}
Both these questions have important practical applications. Indeed, just reason in terms of sequences. A positive answer to the first (resp.\ second) question would yield: ``If $\embK(D_n)$ converges to $g \in \HK$, then there exists a distribution $D$ such that $g = \embK(D)$ (resp.\ and $D_n \rightarrow D$ in $\Dall{\infty}$).''
The difference with \ref{Q1} and \ref{Q2} is huge: they assume that the limit $D$ is the KME of a distribution. Instead, if the KME was surjective, then whatever the convergent sequence of KMEs and \emph{whatever its limit} (in RKHS norm at least), we would know that this limit is the KME of a distribution. Nevertheless, our answers to \ref{Q3} and \ref{Q4} will unfortunately apply only to very specific cases.

We concentrate on \ref{Q4}. Recall that  $\D$ is a topological space of distributions if it is a subspace of $\Dall{\infty}$ that carries a stronger topology $\topo_\D$ than the topology induced by $\Dall{\infty}$. In particular, $\topo_\D$ is always Hausdorff. Consequently, if $\D$ embeds into an RKHS $\HK$ and if the induced kernel-metric dominates the initial topology $\topo_\D$, then $\embK$ must be injective. So if $\embK$ is also surjective, then it is bijective. We may then identify $\D$ and $\HK$ algebraically. And if we endow $\D \equiv \HK$ with the kernel metric, then $\HK$ (or rather $\overline \HK'$) becomes a topological space of distributions. Asking \ref{Q4} thus boils down to asking: when does $\HK$ (or $\overline \HK'$) identify with a topological space of distributions? So, to answer \ref{Q4}, we may use any known criterion to check whether a space of linear forms is a topological space of distributions. The most common criterion is the following: if $\Cont{\infty}{c}$ is continuously contained and dense in a \gls{lcv} TVS $\F$, then $\F'$ is continuously contained in $\Dall{\infty}$. Said differently, $\F'$ is then a topological space of distributions. Applying this with $\F = \Cont{}{0}$, for example, is how one identifies $\Mf$ as a (topological) subspace of $\Dall{\infty}$. And applying it with $\F = \HK$ yields Proposition~\ref{prop:TrivialEmbedding}, proved in Appendix~\ref{proof:TrivialEmbedding}.
\begin{proposition}\label{prop:TrivialEmbedding}
	If $\Cont{\infty}{c}$ is continuously contained and dense in $\HK$, then $\HK$ identifies via the KME with a topological space of distributions via the KME.
\end{proposition}
In practice, the assumptions of Proposition~\ref{prop:TrivialEmbedding} seem cumbersome to verify. 
Both examples of RKHSs that do and that do not verify these assumptions exist. Any Sobolev space $\Sobo{m,2}{0}(\real^d)$ does. And it is an RKHS whenever $m > 1/2$ \citep[Theorem~121]{berlinet04}. (But honestly: no need to use Proposition~\ref{prop:TrivialEmbedding} to re-discover that the duals of these Sobolev spaces are topological spaces of distributions.) On the other hand, Gaussian kernels do not verify the assumptions, because, contrary to $\Cont{\infty}{c}$, all functions of a Gaussian RKHS are extendable to entire functions over $\complex$ \citep[Corollary~4.40]{steinwart08}.

Instead of wondering on what conditions $\HK$ identifies with a subset of distributions, we may wonder, whether it might be a superset of distributions. Could there be any RKHS such that all distributions $\Dall{\infty}$ would embed into $\HK$? If $\HK$ was continuously included in $\Cont{\infty}{c}$, yes. But the following lemma, proved in Appendix~\ref{proof:HKnotBigger}, often excludes this inclusion.
\begin{lemma}\label{lem:HKnotBigger}
	If $\inputS$ is non-compact subset of $\real^d$ and if $\K$ is \gls{spd}, then $\HK$ contains functions with non-compact support.
\end{lemma}
Two comments on the assumptions of this result. First, the assumption that $\K$ be \gls{spd} is not very restrictive, for the identification of $\HK$ with a sub- or superset of $\Dall{\infty}$ would only make sense, if the canonical embedding of one into the other was injective. And this requires an \gls{spd} kernel. As for the other assumption, that $\inputS$ be non-compact, it is also rather natural, because we already know that, as soon as $\K$ is smooth, $\HK$ can embed all distributions with compact support.

Summarising, we answered \ref{Q4} positively in some cases (when $\Cont{\infty}{c}$ embeds continuously into $\HK$), but left many others unanswered. Lemma~\ref{lem:HKnotBigger} is good news: it fortifies the presumption that RKHSs are at least not bigger than $\Dall{\infty}$. Whether they can always be seen as the KME of a distribution space remains, however, an open question.

\section{Conclusion\label{sec:Conclusion}}

We first discuss how this work relates and contributes to existing machine learning literature and then conclude with future work and closing remarks.

\paragraph{Comparison with Related Machine Learning Literature.}
	Universal and characteristic kernels play an important role in learning theory and machine learning. Universal kernels ensure the consistency of some RKHS-based estimators in the context of regression and classification \citep{steinwart01}, whereas characteristic kernels have found applications in the context of two-sample tests \citep{gretton07, gretton12}, independence tests \citep{gretton08, gretton10, fukumizu08} and kernel density estimators \citep{sriperumbudur13}. The terminology is surprisingly recent: in the past fifteen years, the machine learning community introduced various notions of universality \citep{steinwart01, micchelli06, carmeli06, caponnetto08},  accompanied by sufficient conditions on kernels to guarantee their universality. Meanwhile, \citet{fukumizu04} introduced the notion of a characteristic kernel to $\Mprob$ followed by a series of articles \citep{fukumizu04, fukumizu08, fukumizu09a, fukumizu09b, gretton07, sriperumbudur08, sriperumbudur10a, sriperumbudur10b} to characterise this notion. \Citet{gretton07} were first to link universal and characteristic kernels, by showing that characteristic implies universal. 
\Citet{fukumizu09b} were first to explicitly link \gls{ispd} and characteristic kernels, by showing that one implies the other. \Citet{sriperumbudur10b} then used this link to prove handy sufficient criteria for a kernel to be characteristic. In their overview article, \citet{sriperumbudur11} catalogue known (and complete some missing) relations between the various notions of universal, characteristic and \gls{spd} kernels that had appeared. However, they do not clearly systemise their equivalence. Theorem~\ref{theo:UniversalMeansCharacteristic} and Proposition~\ref{prop:SPD} fill this gap, by showing that, if $\HK$ continuously embeds into $\F$, then $\K$ is universal over $\F$ iff it is characteristic (or, equivalently, \gls{spd}) over $\D = \F'$. This is surprisingly easy to prove, simple to remember and covers almost any case encountered in practice... except the most common one: when the kernel is characteristic to $\Mprob$. Theorem~\ref{theo:CKCharacterisation} and Proposition~\ref{prop:UniversalityForP} palliate this shortage and complete respectively Theorem~\ref{theo:UniversalMeansCharacteristic} and Proposition~\ref{prop:SPD}.
	
	Although introduced much earlier in the mathematical community \citep{guilbart78}, the question whether a given kernel metrises the narrow topology was introduced and studied in the machine learning community by \cite{sriperumbudur10b}. They provided sufficient conditions, later on improved and used by \citet{sriperumbudur13} to upper bound convergence rates of kernel density estimators. Theorem~\ref{theo:MetrizationOfNarrow2} generalises these sufficient conditions and finally provides a converse on locally compact Hausdorff spaces. 
	
	The key observation underlying this paper is the dual link between the embedding of $\HK$ into $\F$ and the KME of $\F'$, which is captured in Lemma~\ref{lem:Transpose} and Corollary~\ref{cor:EmbedDuals}. We do not claim originality for this insight but hope that this work will convince the machine learning community of its usefulness. Not only does it underly the link between universal, characteristic and \gls{spd} kernels, but it straightforwardly leads to the extension of KMEs from measures to distributions. Although we discuss possible applications (see discussions after Theorem~\ref{theo:Differentiation}), we have not yet specifically used this extension in a machine learning problem. But the calculus rules of Section~\ref{sec:Calculus} are promising, particularly the link between derivatives of embeddings and embeddings of derivatives. And considering the importance of Schwartz-distributions even in such applied problems as the resolution of (partial) differential equations, we are confident that the extension of KMEs will find useful machine learning applications. Meanwhile, it can be seen as a curious byproduct of Lemma~\ref{lem:Transpose} that completes the theory of KMEs and offers memorable results (Fubini Theorem, Corollary~\ref{cor:Existence}, Proposition~\ref{prop:CompactCharacteristicRd}, Proposition~\ref{prop:UniqueExtension}).

\paragraph{Future Work and Closing Remark.}
	Except when explicitly stated otherwise, we always assumed that $\inputS$ is locally compact. This hypothesis seems more and more unadapted to the general trend in probability theory which is to use Polish spaces, such as separable Banach spaces. But a Banach space is locally compact iff it is finite dimensional. Whether Theorem~\ref{theo:MetrizationOfNarrow2} also holds for any Polish space is certainly the most important open question of this paper. Our choice to work with locally compact spaces is all the more questionable, seeing that every kernel over $\inputS$ defines a very natural (semi-) metric over $\inputS$, namely $\rho_\K(\hat x, \hat y) := \normK{\K(.,\hat x) - \K(.,\hat y)}$. Topologies over graphs, for example, are often defined via such kernel-induced metrics. However, local compactness naturally fits distributions in general, and more specifically measures, when defined via the Riesz representation theorem. 

	Overall, we think our work offers good insights in the nature of the functions that appear when completing $\HKpre$ to $\HK$. Indeed, recall that $\HKpre$ is the set of finite linear combinations of $\K(.,x)$. Thus the question: when completing $\HKpre$, what else could appear than functions that are ``infinite linear combinations of $\K(.,x)$''? When saying ``infinite linear combinations of $\K(.,x)$'', one might first think of continuous sums of weighted infinitesimals $\K(.,x) \diff \mu(x)$, meaning the KME $\int \K(.,x) \diff \mu(x)$ of a measure~$\mu$. One may, however, reasonably argue that distributions are similar to measures, but allow for more general weighting schemes. So it is only natural that, under suitable assumptions on $\K$, they may also embed into $\HK$; in general not all of them, but at least all those with compact support as soon as $\K$ is smooth. In some cases, the embeddable distributions even suffice to cover the whole space $\HK$ (for example, when $\HK$ is a Sobolev space $\Sobo{m,2}{0}(\real^d)$ with $m > 1/2$). However, remember that $\HK'$ identifies with $\HK$ via the anti-linear Riesz representer map. And after all, with weak-integration, this is just saying that $\HK$ contains exactly all those functions $\int \K(.,x) \diff \mu(x)$ where $\mu$ might not be a continuous linear form over $\Cont{\infty}{c}$ (in which case it would be a distribution), but is at least a continuous linear form over $\HK$.

\acks{
	We are deeply indebted to Alphabet for their support via a Google European Fellowship in Causal Inference accorded to Carl-Johann Simon-Gabriel. We thank Bharath Sriperumbudur for enlightening remarks, Jonas Peters for his kind and helpful advice, and Ilya Tolstikhin for useful discussions and encouragements.
}



\appendix



\section{Proofs\label{sec:Proofs}}

In this section, we gather all the complements to non fully proved theorems, propositions, corollaries or lemmas appearing in the main text.

\subsection{Proof of Proposition~\ref{prop:HKinC}\label{proof:HKinC}}
\begin{proof}
	Suppose that $\HK \subset \Cont{}{0}$. \ref{p:1} clearly holds. Suppose \ref{p:2} was not met. Then let $x_n \in \inputS$ such that $\K(x_n, x_n) = \normK{\K(.,x_n)}^2 \rightarrow \infty$. Thus $\K(.,x_n)$ is unbounded. But  $\ipdK{f}{\K(.,x_n)} = f(x_n)$ is bounded for any $f \in \HK$, thus $\K(.,x_n)$ is bounded (Banach-Steinhaus Theorem). Contradiction. Thus \ref{p:2} is met.
	
	Conversely, suppose that \ref{p:1} and \ref{p:2} hold. Let $\HKpre := \lin \{ \K(.,x) \, | \, x \in \inputS \}$. Then, $\HKpre \subset \Cont{}{0}$, and for any $f, g \in \HK$, $\norm{f - g}_\infty \leq \normK{f-g} \norm{\K}_\infty$. Thus $\HKpre$ continuously embeds into the \emph{closed} $\Cont{}{0}$, thus so does its $\normK{.}$-closure, $\HK$. The proof of the cases $\HK \subset \Cont{}{}$ and $\HK \subset \Cont{}{b}$ are similar \citep[see also][Theorem~17]{berlinet04}.
\end{proof}

\subsection{Proof of Proposition~\ref{prop:HKinCm}\label{proof:HKinCm}}

\begin{proof}
		Suppose that $\K \in \Cont{(m,m)}{b}$. Then $\HKpre \subset \Cont{m}{b}$ \citep[Corollary~4.36]{steinwart08} and for any $x \in \inputS$, $f \in \HKpre$, and $|p| \leq m$, we have $\norm{\partial^p f}_\infty \leq \normK{f} \norm{\sqrt{\partial^{(p,p)} \K}}_\infty$. Thus $\HKpre$ continuously embeds into the \emph{closed} space $\Cont{m}{b}$, thus so does its $\normK{.}$-closure, $\HK$. But, almost by definition of $(\Cont{m}{b})_c$ (see Section~\ref{sec:MprobAndUniversality}), we have $\Cont{m}{b} \csubset (\Cont{m}{b})_c$. Thus $\HK \csubset (\Cont{m}{b})_c$
\end{proof}

\subsection{Proof of Theorem~\ref{theo:Fubini}\label{proof:Fubini}}

\begin{proof}
	The definition of the embedding of a distribution, together with the kernel's property that $\K(y,x) = \overline{\K(x,y)}$ leads to:
	\begin{align*}
		\ipdK{D}{T} &= \int_x \ipdK{\int_y \K(.,y) \diff D(y)}{\K(.,x)} \diff \bar T(x) \\
			&= \int_x \overline{\ipdK{\K(.,x)}{\int_y \K(.,y) \diff D(y)}} \diff \bar T(x) \\
			&= \int_x \overline{\int_y \ipdK{\K(.,x)}{\K(.,y)} \diff \bar D(y)} \diff \bar T(x) \\
			&= \iint \K(x,y) \diff D(y) \diff \bar T(x).
	\end{align*}
	Use $\ipdK{D}{T} = \overline{\ipd{T}{D}}_\K$ to prove the right-most part of \eqref{eq:FubiniIP}.
\end{proof}

\subsection{Proof of Lemma~\ref{lem:VectorDiff}\label{proof:VectorDiff}}

\begin{proof}
	That $\vec \varphi_\K$ is continuously differentiable follows from Corollary~4.36 of \citet{steinwart08}.
	Suppose now that $|p| =1$. Then, for any $x,y \in \inputS$, $h \in \real$:
	\begin{align*}
		[\vec \varphi_\K(x)](y) &= \ipdK{\partial^p \vec \varphi_\K(x)}{\K(.,y)} \\
			&\overset{(*)}{=} \lim_{h \rightarrow 0} \ipdK{\frac{\vec \varphi_\K(x+h \, p) - \vec \varphi_\K(x)}{h}}{\K(.,y)} \\
			&= \lim_{h \rightarrow 0} \frac{\K(y,x+h \, p) - \K(y,x)}{h} \\
			&= \partial^{(0,p)} \K(y,x) ,
	\end{align*}
	where $(*)$ holds because strong RKHS-convergence implies weak convergence. We then iterate the process by increasing $|p|$ step by step, until we get the announced result: $\partial^p \vec \varphi_\K = \partial^{(0,p)} \K$.
\end{proof}

\subsection{Proof of Theorem~\ref{theo:Differentiation}\label{proof:Differentiation}}

\begin{proof}
	The proof holds in the following equalities. For any $f \in \HK$,
	\begin{align}
		\ipdK{f}{\int \K(.,x) \diff [\partial^p D](x)} &= \int \ipdK{f}{\K(.,x)} \diff [\partial^p \bar D](x) = [\partial^p \bar D](f) \label{step1}\\
			&= (-1)^{|p|} \bar D(\partial^p f) \label{step3} \\
			&= (-1)^{|p|} \bar D(\ipdK{f}{\partial^{(0,p)}\K(.,\hat x)}) \label{step4} \\
			&= \ipdK{f}{(-1)^{|p|} \int \partial^{(0,p)} \K(.,x) \diff D(x)} \, . \label{step5}
	\end{align}
	Equation~\eqref{step3} uses the definition of the distributional derivative \citep[see][]{schwartzTD}; \eqref{step4} uses the fact that, for any $x \in \inputS$, $\partial^p f(x) = \ipdK{f}{\partial^{(0,p)}\K(.,x)}$ \citep[see][proof of Corollary 4.36]{steinwart08}. And as the left-hand-side of \eqref{step1} is a continuous linear form over $\HK$, it suffices to apply the weak integral's definition to \eqref{step4} to get Equation~\eqref{step5}.
\end{proof}	

\subsection{Proof of Proposition~\ref{prop:Isometry}\label{proof:Isometry}}

\begin{proof}
	The linear span $\HH^{\scriptscriptstyle{pre}}_{\partial^{(0,p)}}$ of $\{ \partial^{(0,p)} \K(.,x) \, | \, x \in \inputS \}$ and that of $\{ \partial^{(p,p)} \K(.,x) \, | \, x \in \inputS \}$ are dense in $\Hop$ and $\Hpp$ respectively. And for any $x,y \in \inputS$ \citep[Lemma~4.29]{steinwart08},
	\begin{equation*}
		 \ipdK{\partial^{(0,p)} \K(.,x)}{\partial^{(0,p)} \K(.,y)} = \partial^{(p,p)} \K(y,x) = \ipd{\partial^{(p,p)} \K(.,x)}{\partial^{(p,p)} \K(.,y)}_{\partial^{(p,p)} \K} \, .
	\end{equation*}
	Thus $\Hop$ and $\Hpp$ are congruent (i.e.\ isometrically isomorphic) via the isometry
	\begin{equation*}
		\function{\widetilde{\Dp}}{\vspace{5pt} \Hop}{\Hpp}{\partial^{(0,p)} \K(.,x)}{\partial^{(p,p)} \K(.,x) \qquad \forall x \in \inputS} \, .
	\end{equation*}
	The maps $\widetilde{\Dp}$ and $\Dp$ coincide on $\HH^{\scriptscriptstyle{pre}}_{\partial^{(0,p)}}$. Let now $f \in \Hop$ and $f_n \in \HH^{\scriptscriptstyle{pre}}_{\partial^{(0,p)}}$ such that $f_n \rightarrow f$. Then, strong convergence implying weak convergence: for any $x \in \inputS$,
	\begin{equation*}
	\begin{array}{ccccccc}
		\partial^p f_n(x) &=&\ipdK{f_n}{\partial^{(0,p)} \K(.,x)} &\longrightarrow& \ipdK{f}{\partial^{(0,p)} \K(.,x)} &=& \partial^p f(x) \\
		\partial^p f_n(x) &=& \ipd{\Dp(f_n)}{\Dp(\partial^{(0,p)} \K(.,x))} &\longrightarrow& \ipdK{\Dp(f)}{\Dp(\partial^{(0,p)} \K(.,x))} &=& \Dp(f)(x)
	\end{array} \, .
	\end{equation*}
	Thus $\Dp(f) = \partial^p f$. Thus $\widetilde{\Dp} = \Dp$.
\end{proof}

\subsection{Proof of Proposition~\ref{prop:IntDiffSwitch}\label{proof:IntDiffSwitch}}

\begin{proof}
	Conditions~\ref{pPoint1} and \ref{pPoint2} are equivalent by Theorem~\ref{theo:Differentiation}. Let now $f \in \Hop$, and let $D$ be a distribution such that $D(\partial^p f)$ is well defined. Using Proposition~\ref{prop:Isometry}, we may then write
	\begin{equation}\label{eq:IDSaux}
	\begin{aligned}
		D \left(\ipdK{f}{\partial^{(0,p)} \K(.,x)} \right) &= D \left(\ipdK{\Dp f}{\Dp [\partial^{(0,p)} \K(.,x)]} \right) \\
			&= D \left(\ipd{\partial^p f}{\partial^{(p,p)} \K(.,x)}_{\partial^{(p,p)} \K} \right) \, .
	\end{aligned}
	\end{equation}
	As $\Dp$ is an isometry from $\Hop$ onto $\Hpp$, we see that the left-hand-side defines a continuous linear form over $\Hop$ \gls{iff} the the right-hand-side defines a continuous linear form over $\Hpp$.
	This both shows that \ref{pPoint2} implies \ref{pPoint3} (because $\Hop \subset \HK$), and that \ref{pPoint2} and \ref{pPoint3} are equivalent when $\K \in \Cont{(m,m)}{0}$ (because then $\Hop = \HK$).
	Let us now suppose that \ref{pPoint2} is met (which covers both preceding cases). Then the embedding $\emb_{\partial^{(0,p)} \K}(D)$ belongs to $\Hop$, because, for any $f$ in the orthogonal of $\Hop$, we have
	\begin{equation*}
		\ipdK{f}{\emb_{\partial^{(0,p)} \K}(D)} = \bar D \left(\ipdK{f}{\partial^{(0,p)} \K(.,x)} \right) = 0 \, .
	\end{equation*}
	And rewriting Equation~\eqref{eq:IDSaux} with KMEs, we get, for any $f \in \Hop$
	\begin{equation*}
		\ipdK{f}{\emb_{\partial^{(0,p)} \K}(D)} = \ipd{\Dp f}{\Dp \emb_{\partial^{(0,p)} \K(.,x)}(D)}_{\partial^{(p,p)} \K} = \ipd{\partial^p f}{\emb_{\partial^{(p,p)} \K(.,x)}(D)}_{\partial^{(p,p)} \K} \, .
	\end{equation*} 
	Thus $\partial^p \emb_{\partial^{(0,p)} \K}(D) = \emb_{\partial^{(p,p)} \K}(D)$, which proves Equation~\eqref{eq:IntDiffSwitch}.
\end{proof}

\subsection{Proof of Corollary~\ref{cor:dpMandMCharacteristic}\label{proof:dpMandMCharacteristic}}

\begin{proof}
	As the conditions \ref{pPoint1}-\ref{pPoint3} from Proposition~\ref{prop:IntDiffSwitch} are met, Equation~\eqref{eq:IntDiffSwitch} applies. But $\embK(\partial^p D)$ equals $\emb_{\partial^{(0,p)} \K}(D)$, which belongs to $\Hop$. Thus Equation~\eqref{eq:IntDiffSwitch} shows that $\emb_{\partial^{(p,p)} \K}(D)$ is simply the image of $(-1)^{|p|}\embK(\partial^p D)$ via the isometry $\Dp$ between $\Hop$ and $\Hpp$, which concludes.
\end{proof}

\subsection{Proof of Proposition~\ref{prop:CtoISPD}}\label{proof:CtoISPD}

\begin{proof}
	If $\K_0$ is characteristic to $\Mf$, then $\K_0 = \K$ and $\nu_0=0$ fulfil the requirements.
	Thus, for the rest of the proof, let us suppose that $\K_0$ is not characteristic to $\Mf$. Then there exists a non zero measure $\nu_0 \in \Mf$ such that $\emb(\nu_0) = 0$, and without loss of generality, we can choose $\nu_0(\One) = 1$.
	The proof now proceeds as follows. 
	\begin{enumerate}[label = (\alph*)]
		\item Show that if $\K_0$ is characteristic to $\Mf$, then the constant function $\One \not \in \HK$. \label{stepA}
		\item Construct a new Hilbert space of functions of the form $\HK = \lin \One \oplus \HH_{\K_0}$ that fulfils the requirements. \label{stepB}
		\item Show that it has a reproducing kernel $\K$. \label{stepC}
		\item Show that $\K$ fulfils the requirements. \label{stepD}
	\end{enumerate}
	\begin{itemize}
		\item[\ref{stepA}] Suppose that $\One \in \HH_{\K_0}$. Then $1 = \bar \nu_0(\One) = \int \ipd{\One}{\K_0(.,x)}_{\K_0} \diff \bar \nu_0(x) \overset{(*)}{=} \ipd{\One}{\int \K_0(.,x) \diff \nu_0(x)}_{\K_0} = \ipd{\One}{\emb_{\K_0}(\nu_0)}_{\K_0} =0$, where in $(*)$ we use Equation~\eqref{eq:WICharacterization} (weak integral characterisation). Contradiction. Thus $\One \not \in \HH_{\K_0}$.
		\item[\ref{stepB}] Define $\HH := \lin \One \oplus \HH_{\K_0}$ and equip it with the inner product $\ipd{.}{.}$ that extends the inner product of $\HH_{\K_0}$ so, that 
		\begin{equation}\label{eq:NormConstraint}
			\One \perp \HH_{\K_0} \quad \text{and} \quad \norm{\One}=1 \, .
		\end{equation}
		In other words, for any $f = c_f \One + f^\perp \in \HH$ and any $g= c_g \One + g^\perp \in \HH$: $\ipd{f}{g} := \ipd{f^\perp}{g^\perp}_{\K_0} + c_f \bar{c_g}$. Obviously $\HH$ is a Hilbert space of functions.
		\item[\ref{stepC}] We now construct a function $\K$ by constructing first a new embedding $\emb$ and then taking the associated kernel. We then show that $\K$ is indeed a reproducing kernel and that $\emb$ is its associated KME.
		
		Define the hyperplane $\Mf^0 := \{ \mu \in \Mf \, | \, \mu(\One) = 0\}$. Each measure $\mu \in \Mf$ can be decomposed uniquely in a sum: $\mu = \mu^\perp + \mu(\One) \nu_0$ where $\mu^\perp = \mu - \mu(\One) \nu_0 \in \Mf^0$. Thus we may define the linear embedding $\function{\emb}{\Mf}{\HH}{}{}$ such that
		\begin{equation*}
			\emb(\mu) := \left \{
				\begin{array}{ll}
					\emb_{\K_0}(\mu) \quad & \mathrm{if} \ \mu \in \Mf^0 \\
					\One \quad & \mathrm{if} \ \mu = \nu_0
				\end{array}
				\right . \, .
		\end{equation*}
		Said differently: $\emb(\mu) := \emb_{\K_0}(\mu^\perp) +\mu(\One) \One = \emb_{\K_0}(\mu) + \mu(1)\One$. By construction, $\emb$ is injective.
		Now define $\K(x,y) := \ipd{\emb(\delta_y)}{\emb(\delta_x)}$.
		Then, for any $f \in \HH$, and any $x \in \inputS$,
		\begin{equation}\label{eq:IPD}
			\ipd{f}{\emb(\delta_x)} = \ipd{f^\perp}{\delta_x^\perp}_{\K_0} + c_f = f^\perp(x) + c_f \One(x) = f(x) \, .
		\end{equation}
		In particular, taking $f = \emb(\delta_y)$, we get $\K(.,y) = \emb(\delta_y)$. Thus \eqref{eq:IPD} may be rewritten:
		\[
			\forall f \in \HH, \forall x \in \inputS, \qquad \ipd{f}{\K(.,x)} = f(x).
		\]
		Thus $\HH$ is an RKHS with reproducing kernel $\K(x,y) = \K_0(x,y) + 1$, and $\emb$ is its associated KME $\embK$.
		
		\item[\ref{stepD}] As $\emb$ is injective over $\Mf$, $\K$ is characteristic to $\Mf$.
		To conclude, note that $\K_0(x,y) = \ipd{\delta_y - \nu_0}{\delta_x - \nu_0}$. 
	\end{itemize}
	\vspace{-2.5em}
\end{proof}

\subsection{Proof of Theorem~\ref{theo:CKCharacterisation}}\label{proof:CKCharacterisation}

In the proof of Proposition~\ref{prop:CtoISPD}, we started with a characteristic kernel $\K_0$ and ended up showing that $\K(x,y) := \K_0(x,y) + 1$ is \gls{ispd} To prove Theorem~\ref{theo:CKCharacterisation}, it suffices, in \eqref{eq:NormConstraint}, to choose $\normK{\One} = \epsilon$ for some real $\epsilon >0$ instead of $\normK{\One} = 1$. This then shows that $\K(x,y) := \K_0(x,y) + \epsilon^2$ is \gls{ispd} The converse is clear.

\subsection{Proof of Proposition~\ref{prop:UniversalityForP}\label{proof:UniversalityForP}}

\begin{proof}
	As the vector space $N$ spanned by $\One$ is closed (because finite dimensional), the transpose of the canonical homomorphism $\function{}{(\Cont{}{b})_c}{(\Cont{}{b} / \One)}{}{}$ is a bijective linear map from the dual $((\Cont{}{b})_c/\One)'$ onto the polar of $N$, which is:  $\Mf^0 := \{ \mu \in \Mf \, | \, \mu(\One) = 0 \}$ \citep[Proposition~35.5]{treves67}. In other words, we may identify the dual of $(\Cont{}{b})_c / \One$ with the set of measures $\Mf^0$ that sum to 0. Thus $\K$ is characteristic to $\Mf^0$ \gls{iff} $\K$ is universal over $(\Cont{}{b})_c/\One$. Lemma~\ref{lem:CharacteristicAndM0} then concludes.
\end{proof}

\subsection{Proof of Theorem~\ref{theo:Q2}}\label{proof:Q2}

\begin{proof}
	That both weak topologies $w(\F',\F)$ and $w(\HK',\HK)$ coincide on bounded sets is the direct application of a known functional analysis result which can be stated as follows. On an equicontinuous set of a dual space $\F'$ (of a Hausdorff TVS $\F$), the weak-star topology ---or topology of pointwise convergence in $\F$--- coincides with the topology of pointwise convergence in a dense subset of $\F$ ---such as $\HK$, when $\K$ is universal over $\F$. The crux is then that, in a barrelled space, the equicontinuous sets in $\F'$  are precisely its bounded sets. For more details, see \citet[Proposition~32.5 and Theorem~33.1, ``Banach-Steinhaus Theorem'']{treves67}. This proves that $w(\F',\F) \cap \B \ = \ w(\HK',\HK) \cap \B$. The rest of Equation~\eqref{eq:Q2s} uses Proposition~\ref{prop:ContinuousEmbedding} and the fact that the weak topology of $\HK'$ is weaker than its strong topology.
\end{proof}

%
%
%

\subsection{A Second Proof of Corollary~\ref{cor:Existence}}\label{proof:Existence}

\begin{proof}{[Second Proof]}
	Theorem~9 of \citet{sriperumbudur10b} shows that we only need to prove that~\ref{th:Support} implies~\ref{th:DLone}. 
	The idea of this second proof is to generalise the proof of \citeauthor{sriperumbudur10b} so as to directly show that~\ref{th:Support} implies~\ref{th:DLone} instead of showing that~\ref{th:Support} implies~\ref{th:Mprob}. The justification of equalities $(a)$ and $(b)$ are rather long: we prove them in two separate proofs.
	
	Suppose that $\K$ be $c_0$-universal. We will show that, for any $D \in \DLone{m}$, $\normK{D}^2 = [D \otimes \bar D](\K) =0 \Rightarrow D=0$.
	Let $\Lambda := \Four \psi$, which by Bochner's Theorem is a positive, finite (Borel) measure. 
	Then:
	\begin{align*}
		[D \otimes \bar D](\K) &= [D_x \otimes \bar D_y]\left (\int e^{-i(x-y)\xi} \diff \Lambda (\xi) \right ) \\
			&\overset{(a)}{=} \int [D_x \otimes \bar D_y](e^{-i(x-y)\xi}) \diff \Lambda(\xi) \\
			&= \int |D_x(e^{-ix\xi})|^2 \diff \Lambda(\xi).
	\end{align*}
	As $\Lambda$ has full support over $\inputS$, this implies $D_x(e^{-ix\xi}) \overset{(b)}{=} [\Four D_x](\xi) = 0$ for
 all $\xi \in \inputS$. Thus $D = 0$.
\end{proof}

Before justifying $(a)$ and $(b)$, here is some necessary background material. Both justifications essentially boil down to the question, whether the two (or three) distributions of a tensor product commute. For example, in $(a)$, we have to justify that $[D_x \otimes \bar D_y \otimes \Lambda_\xi](e^{i(x-y)\xi}) = {[\Lambda \otimes D_x \otimes \bar D_y](e^{i(x-y)\xi})}$. Now, usually \citep[for example in][]{schwartzTD}, the tensor product of distributions $S_x \in \Dall{\infty}_x$ and $T_y \in \Dall{\infty}_y$ is defined as a new distribution $S_x \otimes T_y$ in $\Dall{\infty}_{x,y}$. (Here, and in all this proof, the indices such as $x,y$ specify the underlying dummy variables.) This distribution acts on test functions in $\Cont{\infty}{c,x,y}$, and for those functions $\varphi$, one knows that $S_x \otimes T_y(\varphi) = T_y \otimes S_x(\varphi)$. Here however, we only used the tensor product notation as a proxy for successive parametric integrations: $[D_x \otimes \bar D_y \otimes \Lambda_\xi](e^{i(x-y)\xi})$ means that we first integrate the function $e^{i(x-y)\xi}$ \gls{wrt} $\Lambda$, then \gls{wrt} $D_y$, then \gls{wrt} $D_x$. If $e^{i(x-y)\xi}$ was in $\Cont{\infty}{c,x,y,\xi}$, this would be the same as evaluating the usual tensor product at a given test function. But $e^{i(x-y)\xi}$ is not in $\Cont{\infty}{c,x,y,\xi}$ so, a priori, we cannot interpret our successive parametric integrations as a usual tensor product of distributions. However, notice that, if $D_x$ and $T_y$ are in $\DLone{m}$, we know that they can (be uniquely extended to) integrate not only test functions in $\Cont{\infty}{c}$, but also any function in $\Cont{m}{b}$. Thus their tensor product $D_x \otimes T_y$ should also be uniquely extendable to a larger set of functions than $\Cont{\infty}{c,x,y}$, which we call $\HH^{m,m}_{x,y}$. This is done by \citet{schwartzFDVV}. Furthermore, on these spaces $\HH^{m,m}_{x,y}$, Schwartz shows that the tensor product commutes and can still be evaluated by successive parametric integrations. Thus, if $e^{i\hat x \hat y}$ happened to be in $\HH^{m,m}_{x,y}$, we see the quantity $[S_x \otimes T_y](e^{ixy})$, seen as a successive parametric integration, is actually the tensor product of $S_x$ and $T_y$ evaluated in $e^{ixy}$. And it commutes.

We now proceed with the justifications of $(a)$ and $(b)$. They will mostly be about identifying this space $\HH^{m,m}_{x,y}$ (or its analogue for the tensor product of three distributions) and verifying that the integrand belongs to it.  
\\

\begin{proof}{[Justification of $(a)$]}
We want to show that $[D_x \otimes \bar D_y \otimes \Lambda_\xi](e^{i(x-y)\xi}) = {[\Lambda \otimes D_x \otimes \bar D_y](e^{i(x-y)\xi})}$ (where the right-hand side is not yet proven to be well-defined). As just mentioned, the tensor product of distributions is usually defined to act on test functions in $\Cont{\infty}{c}$. But $e^{i(\hat x-\hat y) \hat \xi}$ is obviously not such a test function. So we start by extending the domain of definition of certain tensor products. Proposition~14 of \citet{schwartzFDVV}, and its two preceding paragraphs (applied with $L^l_x = (\Cont{m}{b,x})_c$, $M^m_y = (\Cont{m}{b,y})_c$ and $E = (\Cont{\scriptscriptstyle 0}{b,\xi})_c$) show that, for any distributions $S_x, T_y, Q_\xi$ in $(\DLone{m})_x, (\DLone{m})_y$ and $(\Mf)_\xi$, their tensor product $S_x \otimes T_y \otimes Q_\xi$ can be extended to a continuous linear form over a much wider class of functions than $\Cont{\infty}{c, x, y, \xi}$, which we call $\HH^{m,m,\scriptscriptstyle 0}_{x,y,\xi}$. Precisely, $\HH^{m,m,\scriptscriptstyle 0}_{x,y,z}$ is the tensor product of $(\Cont{m}{b,x})_c$, $(\Cont{m}{b,y})_c$ and $(\Cont{\scriptscriptstyle 0}{b,\xi})_c$, completed for Grothendieck's $\epsilon$-topology. We write:
\begin{equation*}
	\HH^{m,m,\scriptscriptstyle 0}_{x,y,\xi} := (\Cont{m}{b,x})_c \hat \otimes_\epsilon (\Cont{m}{b,y})_c \hat \otimes_\epsilon (\Cont{\scriptscriptstyle 0}{b,\xi})_c \, .
\end{equation*}
 As the tensor products of $S_x$, $T_y$ and $Q_\xi$ commute on the algebraic (non-completed) tensor product $(\Cont{m}{b,x})_c \otimes (\Cont{m}{b,y})_c \otimes (\Cont{\scriptscriptstyle 0}{b,\xi})_c$, which is dense in $\HH^{m,m,\scriptscriptstyle 0}_{x,y,\xi}$, they commute also on $\HH^{m,m,\scriptscriptstyle 0}_{x,y,\xi}$. Now, the proof would be finished, if $e^{i(\hat x - \hat y) \xi}$ belonged to $\HH^{m,m,\scriptscriptstyle 0}_{x,y,\xi}$. But when $m \not = 0$, this is wrong. Indeed, Proposition~13 of \citet{schwartzFDVV} shows that $\HH^{m,m,\scriptscriptstyle 0}_{x,y,\xi}$ contains exactly all the functions $\varphi(\hat x, \hat y, \hat \xi)$ such that, for all $|p|,|q| \leq m$, the partial derivatives $\partial^p_x \partial^q_x \varphi$  exist and are continuous and bounded. Thus $e^{i(\hat x - \hat y) \hat \xi} \not \in \HH^{m,m,\scriptscriptstyle 0}_{x,y,\xi}$. 

To solve this issue, we use the following trick. First, write
\begin{equation*}
	[D_x \otimes \bar D_y \otimes \Lambda_\xi]\left (e^{i(x-y)\xi} \right) = [D_x \otimes \bar D_y \otimes \widetilde \Lambda_\xi]\left (\frac{e^{i(x-y)\xi}}{1+|\xi|^{2m}} \right) \, ,
\end{equation*}
where $\widetilde \Lambda := (1+|\hat \xi|^{2m}) \Lambda$. Now the function $e^{i(x-y)\xi}/(1+|\xi|^{2m})$ is obviously in $\HH^{m,m,\scriptscriptstyle 0}_{x,y,\xi}$. And to finish the proof, we show that $\widetilde \Lambda \in \Mf$. Indeed, remember that $\Lambda$ is the (distributional) Fourier transform of the kernel $\K(x,y) = \psi(x - y)$: $\Lambda = \Four \psi$. But $\K$ is in $\Cont{(m,m)}{x,y}$. Thus, for any $|p| \leq m$, $\partial^{(p,p)} \K(x,y) = (-1)^p \partial^{2p} \psi(x-y)$ is a continuous kernel. So, using Bochner's theorem, we see that the Fourier transform of $\partial^{2p} \psi$, $(-1)^{|p|} \Four {(\partial^{2p} \psi)} (\hat \xi) = \hat \xi^{2p} \Four \psi(\hat \xi) = \hat \xi^{2p} \Lambda(\hat \xi)$ is a finite positive measure. Being a positive linear combination of such measures, $\widetilde \Lambda$ also belongs to $\Mf$.
\end{proof}
\begin{proof}{[Justification of $(b)$]}
	Using Théorème~2 of \citet{schwartzFDVV}, applied with $H^m = (\Cont{m}{b})_c$ and $E = (\Cont{\scriptscriptstyle 0}{b})_c$, we see that the function $\function{f}{}{}{\xi}{D_x(e^{-i x \xi})}$ is well defined and continuous. (Note that, having justified step $(a)$ already shows that $f$ is well-defined and integrable \gls{wrt} $\Lambda$.)  Now, let $\function{T^f}{\Cont{\infty}{c}}{\complex}{\varphi}{\int \varphi f}$ be the distribution associated to $f$. We want to show that $T^f = \Four D$. To do so, let $\varphi \in \Dall{\infty}$ and let $T^\varphi$ be the distribution associated to $\varphi$. The following equality concludes.
	\begin{multline*}
		\Four D(\varphi) := D(\Four \varphi) = \D_\xi(\int e^{ix\xi} \varphi(x) \diff x) = [D_\xi \otimes T^\varphi_x] \left (e^{ix\xi} \right )
			= [T^\varphi_x \otimes D_\xi] \left (e^{ix\xi} \right ) = T^f(\varphi) \, .
	\end{multline*}
	Here, the justification for the swap of tensor products is essentially the same as the justification of $(a)$. ($[D_\xi \otimes T^\varphi_x]$ is a continuous linear form over $\HH^{m,\scriptscriptstyle 0}_{x,\xi} := \Cont{m}{b,x} \hat \otimes_\epsilon \Cont{\scriptscriptstyle 0}{\xi}$; it commutes on the dense subspace $\Cont{m}{b,x} \otimes \Cont{\scriptscriptstyle 0}{\xi}$ thus also on $\HH^{m,\scriptscriptstyle 0}_{x,\xi}$; and $e^{i\hat x \hat \xi}$ belongs to $\HH^{m,\scriptscriptstyle 0}_{x,\xi}$.)
\end{proof}
This ends the proof of Theorem~\ref{cor:Existence}.

\subsection{Proof of Theorem~\ref{theo:CompactCharacteristic}}\label{proof:CompactCharacteristic}

\begin{proof}{[of Theorem~\ref{theo:CompactCharacteristic}]}
	We recall that an entire (or holomorphic) function $\function{f}{\complex^d}{\complex}{}{}$ is said of exponential type \gls{iff} $|f(z)| \leq M e^{\tau |z|}$ for some positive numbers $M$ and $\tau$, and sufficiently large $z$. We also remind that the zeroes of an entire function are necessarily isolated: they are thus at most countable. Our proof now relies on the following three theorems. The first is proved by \citet[Chapitre~VII, Théorème~XVI]{schwartzTD} \citep[see also][Theorem~29.2]{treves67}; the second, which we will only state for the particular case of entire functions of exponential type, can be found in the book by \citet[Theorem~2.10.1]{boas54}; the third is the compilation of Theorem~2.6.5 of \citet{boas54} and the remarks just preceding it.
	\begin{theorem}[Paley-Wiener-Schwartz]
		The (distributional) Fourier transform of a distribution over $\inputS = \real^d$ with compact support is the restriction to $\real^d$ of an entire function of exponential type defined over $\complex^d$; and vice-versa.
	\end{theorem}
	\begin{theorem}[Lindelöf]
		Let $\function{f}{\complex}{\complex}{}{}$ be an entire function. Let $(z_n)_n$ be its zeroes. 
		The function $f$ is of exponential type \gls{iff} there exists a constant $M>0$, such that, for any $r>0$
		\begin{equation}\label{eq:Lindeloef}
			|\sum_{|z_n| \leq r} \frac{1}{z_n} | \leq M \qquad \text{and} \qquad \# \{|z_n| \leq r \} \leq M r \, .
		\end{equation}
	\end{theorem}
	\begin{theorem}[Weierstrass]
		For any integer $p$ and any $z \in \complex$, we define the \emph{Weierstrass primary factor} $E(z;p)$ as
		\begin{equation*}
			E(z;p) = \left \{
				\begin{array}{ll}
					1-z \quad & \mathrm{if} \ p = 0 \\
					(1-z) \exp(z +\frac{z^2}{2} + \cdots + \frac{z^p}{p}) & \mathrm{if} \ p>0
				\end{array}
				\right . \, .
		\end{equation*}
		Let now $(z_n)_n$ be a possibly finite complex sequence, ordered by increasing modulus, with $z_0 \not = 0$, and such that
		\begin{equation}\label{eq:Genus}
			\forall \epsilon >0, \qquad \sum_n \frac{1}{|z_n|^{p+\epsilon}} < \infty \, .
		\end{equation}
		Then the products $\prod_n E(\frac{z}{z_n} ; p)$ converge pointwise to an entire function $P(\hat z)$, whose zeroes are exactly the $z_n$'s.
	\end{theorem}

	We now prove Theorem~\ref{theo:CompactCharacteristic}.
	First, note that, as $\K$ is in $\Cont{(m,m)}{}$, the distributions with compact support of order $m$, $\Dcomp{m}$, indeed embed into $\HK$.
	Let now $D \in \Dcomp{m}$. As proven in the second proof of Corollary~\ref{cor:Existence}, we have
	\begin{equation}\label{eq:NormFour}
		\begin{aligned}
			\normK{D}^2 &= [D_x \otimes D_y] \left (\int e^{-i(x-y)\xi} \diff \Lambda (\xi) \right ) \\
				&\overset{(*)}{=} \int [D_x \otimes D_y](e^{-i(x-y)\xi}) \diff \Lambda(\xi) \\
				&= \int |D_x(e^{-ix\xi})|^2 \diff \Lambda(\xi) \, ,
		\end{aligned}
	\end{equation}
	where $\Lambda = \Four \psi$. In particular, we had shown that the function $D_x(e^{-ix \hat \xi})$ coincides with the distributional Fourier-transform $\Four D$ of $D$. This time, however, as $D$ has compact support, the Paley-Wiener-Schwartz theorem even shows that $\Four D$ is an entire function of exponential type. And Lindelöf's theorem shows that its zeroes $(z_n)_n$ verify Equation~\eqref{eq:Lindeloef}.
	\\
	
	For the direct part, suppose that $\K$ is not characteristic to $\DLone{m}$. Then we may as well suppose that $\normK{D} = 0$.
	Let us now decompose $\Lambda$ into its atomic part $\Lambda_\delta$ and non-atomic part $\Lambda_s$: $\Lambda = \Lambda_\delta +  \Lambda_s$ \citep[Theorem~3.5.1]{dudley02}. As $\Lambda$ is a positive measure (Bochner theorem), so are $\Lambda_\delta$ and $\Lambda_s$. Lindelöf's theorem and Equation~\ref{eq:NormFour} show that the support $\{x_n \}$ of the atomic part must satisfy $\# \{|z_n| \leq r \} \leq M r$ for any $r>0$ and a fixed $M>0$. As for $\Lambda_s$, we claim it is null, which will prove the direct part.
	
	Indeed, suppose that $\Lambda_s \not = 0$. Let $C$ be a compact subset of its support that contains no zero $z_n$ of $\Four D$. Then, for all $\xi \in C$, $|\Four D(\xi)|^2 \geq \min_C |\Four D|^2 > 0$ and
	\begin{equation}
		\normK{D}^2 = \int |D_x(e^{-ix\xi})|^2 \diff \Lambda(\xi)
			\geq \min_C |\Four D|^2 \int_C \diff \Lambda(\xi) > 0 \, .
	\end{equation}
	Contradiction. Thus $\Lambda_s = 0$, which proves the direct part.
	\\
	
	Conversely, suppose that the Fourier transform $\Lambda$ of $\psi$ has a countable support $(x_n)_n$ and that there exists $M >0$ such that, for any $r>0$
	\begin{equation}\label{eq:Growth2}
		\# \{ |x_n| \leq r \} \leq M r \, .
	\end{equation}
	We will now construct an entire function of exponential type that cancels in each $x_n$.
	
	Consider a complex sequence $z_n$, ordered by increasing modulus, such that $\{z_n\} = (\{ x_n \} \cup \{ - x_n \}) \backslash \{0\}$. As $(x_n)_n$ satisfies Equation~\eqref{eq:Growth2}, $(z_n)_n$ satisfies Equation~\eqref{eq:Genus}. Thus Weierstrass' theorem applies and give an entire function $P(z) := z \prod_n E(\frac{z}{z_n} ; 1)$ that cancels on each $x_n$. And, as for any $r>0$, $|\sum_{|z_n| \leq r} z_n^{-1} | = 0$, Lindelöf's theorem shows that $P$ is of exponential type. And the Paley-Wiener-Schwartz theorem shows that $P$ is the Fourier transform of a distribution $T$ with compact support: $T \in \Dcomp{\infty}$. $T$, however, need not be of order $m$. Thus, instead, we consider a distribution $D := \phi \star T$, where $\phi$ is any fixed smooth function with compact support and where $\star$ denotes the convolutional product operator. Then $D$ is a smooth function with compact support, so $D \in \Dcomp{m}$. And the Fourier transform $\Four D$ of $D$ is the product of the two entire functions $\Four \phi$ and $\Four T$, thus it cancels on every $x_n$. So $\normK{D} = 0$, showing that $\K$ is not characteristic to $\Dcomp{m}$.
\end{proof}
	
\begin{remark}
	We did not find an analogue of Lindelöf's theorem when $\inputS$ is $\real^d$ instead of $\real$. That is why we restricted this Theorem~\ref{theo:CompactCharacteristic} to the one-dimensional case.
\end{remark}

	To prove the converse part, we constructed a distribution $D \in \Dcomp{m}$ whose Fourier transform cancels on the support $(x_n)_n$ of $\Four \psi$. To do so, we used Weierstrass' theorem. Here is an alternative, if, instead of assuming that $(x_n)_n$ satisfies Equation~\eqref{eq:Growth2}, we suppose that is verifies the stronger condition
	\begin{equation}
		\sum_n \frac{1}{|x_n|} < \infty \, .
	\end{equation}
	
	We will construct $D$ as the limit of the sequence of convolutional products of the following functions: 
	\begin{equation*}
		f_{n}(x) := \left \{
		\begin{array}{cl}
			\frac{|x_n|}{2} \qquad &\mathrm{if} \ |x| \leq |x_n| \\
			0 \qquad &\mathrm{otherwise}
		\end{array}
		\right . \, .
	\end{equation*}
If the sequence $(x_n)_n$ has finite cardinal $N$, then $D$ can be chosen as the $N$-fold convolutional product $D := f_1 \ast f_2 \ast \cdots \f_N$. That $D$ then satisfies all the requirements is clear. Thus, from now on, we suppose that $(x_n)_{n \geq 1}$ is not a finite sequence. This time, $D$ will be the limit of the convolutional products $F_n := f_1 \ast f_2 \ast \cdots \ast f_n$ when $n$ grows indefinitely. To show that these convolutional products do converge to a probability measure $D$, we use Levy's continuity theorem \citep[Theorem~9.8.2]{dudley02}. Indeed: 
\begin{itemize}
	\item $[\Four f_n](\hat x) = \sinc(\hat x/|x_n|)$, where $\sinc(\hat x) := \frac{\sin \hat x}{\hat x}$. For any $x \in \real$, $|\sinc x| \leq 1$, thus $|[\Four F_n](x)| = |[\Four f_1](x)[\Four f_2](x) \cdots [\Four f_n](x)|$ decreases with $n$, thus converges. And, as soon as $|x| \leq |x_m|$, $[\Four f_m](x) \geq 0$. Thus, for any $n \geq m$ the sign of $\Four F_n$ stays constant, thus converges. Thus $\Four F_n$ converges pointwise to a function we denote $\phi$.
	\item Let us now show that $\phi$ is continuous in $0$. Let $|x| < |x_1|$. Then $1 \geq \sinc \frac{x}{|x_n|} \geq 1 - \frac{x^2}{6 |x_n|^2}$. Thus 
	\begin{equation*}
		0 = \ln \phi(0) \geq \ln \phi(x) 
			\geq \sum_n \ln(1 - \frac{x^2}{|x_n|^2})
			\geq - \frac{x^2}{6} \sum_n \frac{1}{|x_n|^2} \xrightarrow{x \rightarrow 0} 0 \, .
	\end{equation*}
	Thus $\phi$ is continuous in $0$.
\end{itemize}
Thus $\phi$ is the Fourier transform $\Four D$ of a probability measure $D$. Finally, note that for any $n$,
\begin{equation*}
	\supp F_n \subset \left [- \sum_n \frac{1}{|x_n|}, \sum_n \frac{1}{|x_n|} \right ] =: S \, .
\end{equation*}
Thus $P$ has also its support in the compact set $S$. Finally, by construction, for any $n$, $[\Four P](x_n) = 0$, thus $\normK{P}^2 = \int |\Four P|^2 \diff \Four \psi = 0$, although $P \not = 0$. Thus $\psi(\hat x - \hat y)$ is not characteristic to $\Mprob \cap \Dcomp{m}$.

\subsection{Proof of Proposition~\ref{prop:CompactCharacteristicRd}}\label{proof:CompactCharacteristicRd}

\begin{proof}
	The proof is completely similar to the direct part of the preceding proof (of Theorem~\ref{theo:CompactCharacteristic}). To show Statement~\ref{p2CC}, suppose that $\K$ is not characteristic to $\Dcomp{m}$ and take $D \in \Dcomp{m}$ such that $\normK{D} =0$. Write
	\begin{equation}
		\normK{D}^2  = \int |D_x(e^{-ix\xi})|^2 \diff \Lambda(\xi) \, .
	\end{equation}
	Notice that $D_x(e^{-ix\xi})$ has at most a countable number of zeroes. Deduce that the support of the atomic part of $\Lambda$ is at most countable. Conclude by showing that $\Lambda$ is an atomic measure. This proves~\ref{p2CC}. 
	
	To prove Statement~\ref{p1CC}, do the same, and consider now the convolutional product $P$ of $D$ with a smooth, compactly supported function $\varphi$ ($P := D \star \varphi$, with $\varphi \in \Cont{\infty}{c}$). Then $P$ is also a smooth, compactly supported function, so defines a measure in $\Mc$: $P \in \Mc$. But any zero of $\Four D$ is also a zero of $\Four P$, because $\Four P = \Four \varphi \cdot \Four D$. Thus they also contain the support of $\Lambda$, which shows that $\normK{P} = 0$, and that $\K$ is not characteristic to $\Mc$ either. By contraposition, we proved that if $\K$ is characteristic to $\Mc$, then it is characteristic to $\Dcomp{m}$. The converse is clear.
\end{proof}

\subsection{Proof of Theorems~\ref{theo:Ws0} and \ref{theo:Ws}}
%
What was left, was to show the continuity of appropriate tensor product maps $\function{}{}{}{(D,T)}{D \otimes T}$. These maps will be given in Lemma~\ref{lem:TensorP}, \ref{lem:Tensor0} and \ref{lem:TensorM}. They respectively prove {\ref{Ws04}$\Rightarrow$\ref{Ws01},} {\ref{Ws07}$\Rightarrow$\ref{Ws01}} and Theorem~\ref{theo:Ws}. 

\begin{lemma}\label{lem:TensorP}
	Let $\inputS, \mathcal Y$ be locally compact Hausdorff spaces. The bilinear map
	\begin{equation*}
		 \function{B}{(\M_{\scriptscriptstyle +}(\inputS))_\sigma \times (\M_{\scriptscriptstyle +}(\mathcal Y))_\sigma}{(\M_{\scriptscriptstyle +}(\inputS \times \mathcal Y))_w}{(\mu,\nu)}{\mu \otimes \nu}
	\end{equation*}
	is continuous (where $\sigma$ denotes the narrow convergence topology).
\end{lemma}
\begin{proof}
	See Chapter~2, Theorem~3.3 of \citet{berg84}.
\end{proof}
\begin{lemma}\label{lem:Tensor0}
	Let $\inputS, \mathcal Y$ be locally compact Hausdorff spaces. The bilinear map
	\begin{equation*}
		 \function{B}{(\Mc(\inputS))_w \times (\Mc(\mathcal Y))_w}{(\Mc(\inputS \times \mathcal Y))_w}{(\mu,\nu)}{\mu \otimes \nu}
	\end{equation*}
	is continuous over any sets $\mathcal B \times \mathcal C$, where $\mathcal B, \mathcal C$ are bounded subsets of $\Mc(\inputS)$ and $\Mc(\mathcal Y)$.
\end{lemma}
\begin{proof}
	As $\mathcal B$ is bounded in $\Mc(\inputS)$, there exists a compact $K_0 \subset \inputS$ such that, for any $\mu \in \mathcal B$, the support of $\mu$ is contained in $K_0$ \citep[p.\,90]{schwartzTD}. Let $K$ be a compact neighbourhood of $K_0$ (which exists, because $\inputS$ is locally compact Hausdorff). Let $\One_{K_0}$ be a function with support contained in $K$, that equals 1 on a neighbourhood of $K_0$. Then, for any $\mu \in \mathcal B$ and any $f \in \Cont{}{}(\inputS)$,$\One_{K_0} f \in \Cont{}{c}$ and $\mu(f) = \mu(\One_{K_0} f)$. Thus, on $\mathcal B$, the weak topology $w(\Mc, \Cont{}{})$ coincides with the vague topology $w(\Mc, \Cont{}{c})$. The same holds for $\mathcal C$ and for $B(\mathcal B \times \mathcal C)$ (because $\inputS \times \mathcal Y$ is also locally compact Hausdorff and $\supp \mu \otimes \nu = \supp \mu \times \supp \nu$). Thus the result follows from Proposition~6 of \citet[Chapitre~3, §4, n°3]{bourbakiINT} and its subsequent remark.
\end{proof}

The proof of Lemma~\ref{lem:TensorM} needs some preparation. Let $\inputS$ and $\mathcal Y$ be open subsets of $\real^d$ and let $\Cont{(m,m)}{}(\inputS \times \mathcal Y)$ be the the of functions $f$ such that, for any $p,q \in \nat \cup \{\infty\}$, with $|p|,|q| \leq m$, the derivative $\partial^{(p,q)} f$ exists and is continuous. Endow $\Cont{(m,m)}{}(\inputS \times \mathcal Y)$ with the topology of uniform convergence of $f$ and of all its derivatives up to the order $(m,m)$ over the compacts of $\inputS \times \mathcal Y$. It is the topology generated by the family
\begin{equation*}
	\mathcal V_{p, \, q, \, \mathcal K, \, \mathcal U} := \left \{\varphi \in \Cont{(m,m)}{}(\inputS \times \mathcal Y) \, | \, [\partial^{(p,q)} \vec \varphi](\mathcal K) \subset \mathcal U \right \} \, ,
\end{equation*}
where $|p| \leq m$, $|q| \leq m$, $\mathcal U$ is an open subset of $\real$ and $\mathcal K$ a compact in $\inputS \times \mathcal Y$.
Note $\Dcomp{(m,m)}(\inputS \times \mathcal Y)$ its dual. 
	Now, let $\Cont{m}{}(\inputS ; \Cont{m}{}(\mathcal Y))$ be the space of functions over $\inputS$ with values in $\Cont{m}{}(\mathcal Y)$ which are $m$-times continuously differentiable. Similarly to $\Cont{m}{}(\inputS ; \real)$, equip $\Cont{m}{}(\inputS; \Cont{m}{}(\mathcal Y))$ with the the topology generated by 
\begin{equation*}
	\mathcal V_{p, \, \mathcal K, \, \mathcal U} := \left \{\vec \varphi \in \Cont{m}{}(\inputS ; \Cont{m}{}(\mathcal Y)) \, | \, [\partial^p \vec \varphi](\mathcal K) \subset \mathcal U \right \} \, ,
\end{equation*}
where $|p| \leq m$, $\mathcal U$ is an open subset of $\Cont{m}{}(\mathcal Y)$ and $\mathcal K$ a compact in $\inputS$. Then
\begin{lemma}\label{lem:CanIso}
	The map
	\begin{equation*}
		\function{}{\Cont{(m,m)}{}(\inputS \times \mathcal Y)}{\Cont{m}{}(\inputS ; \Cont{m}{}(\mathcal Y))}{\varphi}{\function{\vec \varphi}{}{}{x}{\varphi(x,.)}}
	\end{equation*}
	is a bijective isomorphism for the TVS structures.
\end{lemma}
\begin{proof}
	Combine Proposition~9 and 12 of \citet{schwartzFDVV}.
\end{proof}
Now comes the actual lemma of interest for the proof of Theorem~\ref{theo:Ws}.
\begin{lemma}\label{lem:TensorM}
	Let $\inputS$ and $\mathcal Y$ be open subset of $\real^d$ and let $m \in \nat \cup \{\infty\}$. The bilinear map 
	\begin{equation*}
		 \function{B}{(\Dcomp{m}(\inputS))_w \times (\Dcomp{m}(\mathcal Y))_w}{(\Dcomp{(m,m)}(\inputS \times \mathcal Y))_w}{(D_x,T_y)}{D_x \otimes T_y}
	\end{equation*}
	is continuous over any sets $\mathcal B \times \mathcal C$, where $\mathcal B, \mathcal C$ are bounded subsets of $\Dcomp{m}(\inputS)$ and $\Dcomp{m}(\mathcal Y)$.
\end{lemma}
Actually, we will prove the following strengthened version of the lemma.
\begin{lemma}\label{lem:MHypocontinuity}
	Let $\inputS$ and $\mathcal Y$ be open subsets of $\real^d$ and let $m \in \nat \cup \{\infty\}$. Let $(\Dcomp{m}(\inputS))_c$ (and similarly $(\Dcomp{m}(\mathcal Y))_c$ and $(\Dcomp{(m,m)}(\inputS \times \mathcal Y))_c$) denote the space $\Dcomp{m}(\inputS)$ equipped with the topology of uniform convergence over the compact subsets of $\Cont{m}{}(\inputS)$. The bilinear map 
	\begin{equation*}
		 \function{\widetilde B}{(\Dcomp{m}(\inputS))_c \times (\Dcomp{m}(\mathcal Y))_c}{(\Dcomp{(m,m)}(\inputS \times \mathcal Y))_c}{(D_x,T_y)}{D_x \otimes T_y}
	\end{equation*}
	is hypocontinuous \gls{wrt} the bounded subsets of $\Dcomp{m}(\inputS)$ and $\Dcomp{m}(\mathcal Y)$.	
\end{lemma}
\begin{proof}{[Lemma~\ref{lem:MHypocontinuity}]}
	Let $\mathcal K$ be a compact subset of $\Cont{(m,m)}{}(\inputS \times \mathcal Y)$ and let $\mathcal B$ be a bounded subset of $\Dcomp{m}(\mathcal Y)$. We need to show that, when $D_x$ converges to $0$ uniformly over the compacts of $\Cont{m}{}(\inputS)$ while $T_y$ varies in $\mathcal B$, then $D_x \otimes T_y$ converges to $0$ uniformly over $\mathcal K$. Because $\Cont{m}{}(\inputS)$ is barrelled, bounded sets of its dual, such as $\mathcal B$, are relatively weakly compact and equicontinuous. But the bilinear map 
	\begin{equation*}
		\function{F}{(\Dcomp{m}(\inputS))_c \times \Cont{m}{}(\inputS ; \Cont{m}{}(\mathcal Y))}{\Cont{m}{}(\mathcal Y)}{(T_y, \, \vec \varphi)}{T_y(\varphi(\hat x, y))}
	\end{equation*}
	is hypocontinuous \gls{wrt} equicontinuous subsets of $\Dcomp{m}$ and relatively compact subsets of $\Cont{m}{}$ \citep[Corollaire~2 of Proposition~19]{schwartzFDVV}. Thus its restriction to $\mathcal B \times \mathcal K$ (where we identify $\Cont{m}{}(\inputS ; \Cont{m}{}(\mathcal Y))$ and $\Cont{(m,m)}{}(\inputS \times \mathcal Y)$ using Lemma~\ref{lem:CanIso}) is continuous. But seen that $\mathcal B$ is equicontinuous, the weak and the compact convergence topology coincide. Thus, being relatively weakly compact, $\mathcal K$ is also a compact subset of $(\Dcomp{m}(\mathcal Y))_c$. Thus the image of $\mathcal B \times \mathcal K$ via $F$ 
is a compact subset of $\Cont{m}{}(\mathcal Y)$. Thus $D_x(T_y(\varphi))$ converges uniformly to $0$ when $\varphi$ varies in $\mathcal K$ and $T_y$ varies in $\mathcal B$.
%
\end{proof}
\begin{proof}{[Lemma~\ref{lem:TensorM}]}
	First, note that the bounded subsets $\mathcal B$ and $\mathcal C$ of $\Dcomp{m}(\inputS)$ and $\Dcomp{m}(\mathcal Y)$ are also bounded for the compact convergence topology. But the map $\widetilde B$ being hypocontinuous \gls{wrt} bounded subsets of $\Dcomp{m}(\inputS)$ and $\Dcomp{m}(\mathcal Y)$, its restriction to $\mathcal B \times \mathcal C$ is continuous. And $\Cont{m}{}$ being barrelled, on bounded subsets of $\Dcomp{m}$, the weak and the compact convergence topology coincide \citep[Proposition~32.5 and Theorem~33.2]{treves67}.
\end{proof}
\begin{remark}
	Note that when $m=\infty$, then $\Cont{(\infty,\infty)}{} = \Cont{\infty}{}$ and $\Dcomp{(\infty, \infty)} = \Dcomp{\infty}$. In this case, \citet[Chapitre~IV, Théorème~VI]{schwartzTD} even proves that $B$ is continuous when $\Dcomp{\infty}(\inputS)$, $\Dcomp{\infty}(\mathcal Y)$ and $\Dcomp{\infty}(\inputS \times \mathcal Y)$ are equipped with their strong topology. This implies Lemma~\ref{lem:TensorM}, because on bounded sets of $\Dcomp{\infty}$, weak and strong topologies coincide.
\end{remark}

\subsection{Proposition~\ref{prop:WeakerThanNarrow} for Separable Metric Spaces\label{proof:WeakerThanNarrowPolish}}
\begin{proposition}[$(\Mf)_\sigma \overset{seq}{\csubset} \HK$]\label{prop:WeakerThanNarrowPolish}
	Let $\inputS$ be any topological space with a countable dense subset, and let $\K$ be a bounded continuous kernel on $\inputS$. The semi-inner product topology induced by $\K$ in $\Mf$ is sequentially weaker than the narrow topology.
\end{proposition}
In other words: for any $\mu, \mu_1, \mu_2, ... \in \Mf$, if $\mu_n(f) \rightarrow \mu(f)$ for any $f \in \Cont{}{b}$, then $\embK(\mu_n) \rightarrow \embK(\mu)$ in $\normK{.}$. Note that on bounded sets of positive measures ---such as $\Mprob$---, and when $\inputS$ is a separable metric space, then the narrow convergence topology is metrisable.  In that case, ``sequentially weaker''  implies ``weaker''.

\begin{proof}
	Let $B_\K := \{ f \in \HK \, | \, \normK{f} \leq 1 \}$ and $C := \sup_{x \in \inputS} \K(x,x)^{1/2}$. Let $\mu_n,\mu \in \Mf$ such that $\mu_n \rightarrow \mu$ narrowly.
	Define the possibly degenerate metric $d(\hat x,\hat y) := \normK{\K(.,\hat x) - \K(.,\hat y)}$.
	First, note that the underlying topology $\topo$ is stronger than the topology induced by $d$. (Use the fact that $\K$ being continuous, so is $\K(.,\hat x)$.) 
	Now, for any $f \in B_\K$,
	\begin{equation*}
		|f(x) - f(y)| \leq \normK{f} \normK{\K(.,x) - \K(.,y)} \leq d(x,y) \, .
	\end{equation*}
	Thus $B_\K$ is a $d$-equicontinuous set of functions. But $\topo$ being stronger than $d$, $B_\K$ is also a $\topo$-equicontinuous. And it is uniformly bounded, since $|f(x)| \leq C$ for any $x \in \inputS$ and any $f \in B_\K$. Thus Corollary~11.3.4. of \citet{dudley02} applies and gives: $\normK{\mu_n - \mu} = \sup_{f \in B_\K} \int f \diff (\mu_n - \mu) \rightarrow 0$.
\end{proof}

\subsection{Proof of Proposition~\ref{prop:NonMetrization}\label{proof:NonMetrization}}

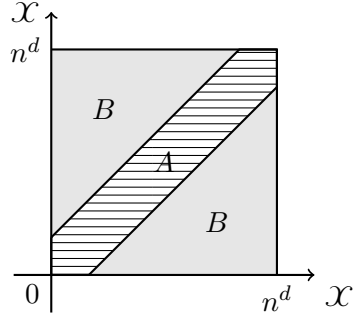
\begin{figure}[h]
	\centering
	\begin{tikzpicture}
		\path[fill = gray!20] (0,0.5)--(2.5,3)--(0,3)--(0,0.5) node at (0.7,2.2) {$B$};
		\path[fill = gray!20] (0.5,0)--(3,0)--(3,2.5)--(0.5,0) node at (2.2,0.7) {$B$};
		    
		\draw[thick,->,black] (-.5,0)--(3.5,0) node[below right] {$\inputS$}; 
		\draw[thick,->,black] (0,-.5)--(0,3.5) node[left] {$\inputS$}; 
		\draw[thick,-,black] (3,3)--(0,3) node[left] {$n^d$};
		\draw[thick,-,black] (3,3)--(3,0) node[below] {$n^d$};
		\draw[thick,-,black, pattern=horizontal lines] (0,0.5)--(2.5,3)--(3,3)--(3,2.5)--(0.5,0)--(0,0)--(0,0.5) node at (1.5,1.5) {$A$};
		\node at (-.25,-.25) {$0$};
	\end{tikzpicture}
	\caption{The $2d$-dimensional hypercube $\mathcal{C}_n$ and its subparts $A$ and $B$ (see Proof of Proposition~\ref{prop:NonMetrization})\label{fig:Hypercube}}
\end{figure}

\begin{proof}
Consider the sequence of probability measures $P_n := \frac{1}{n^d} \lambda_{[0,n]^d}$, where $[0,n]^d$ denotes the $d$-dimensional hypercube with edge of length $n$. The sequence $P_n$ does not converge narrowly (as it is not uniformly tight). But we will show that $P_n$ converges to the null measure in RKHS-norm, which concludes.
Indeed, let $\epsilon > 0$ and $\alpha >0$ such that, for any $h \in \real^d$ with $|h| > \alpha$, $|\psi(h)| \leq \epsilon$. Consider  the $2d$-dimensional hypercube $\mathcal{C}_n$ in $\real^d \times \real^d$, with edges of length $n \in \nat$.
Divide the hypercube into 2 parts $A$ and $B$ as schematised on Figure~\ref{fig:Hypercube}. More precisely, $A$ are all points of $\mathcal{C}_n$ at distance $\leq \alpha$ from the ``diagonal'' $\{(x,x) \in (\inputS \times \inputS)\cap \mathcal{C}_n\}$; $B$ is the rest of $\mathcal{C}_n$. Note that there exist constants $a$ and $b$, independent of $\epsilon$, $\alpha$ and $n$, such that the volumes $V_A$, $V_B$ of $A$ and $B$ verify: $V_A \leq a \alpha \, n^d$ and $V_B \leq b n^{2d}$. Furthermore, by construction $\sup_{(x,y) \in B} |\K(x,y)| \leq \epsilon$. Thus
\begin{align*}
	\normK{P_n}^2 &= \frac{1}{n^{2d}} \iint_{A \, \cup \, B} \K(x,y) \diff x \diff y \\
		&\leq \frac{1}{n^{2d}} \Big( V_A \norm{\K}_\infty + \epsilon \, V_B)\\
		&\leq \frac{a \alpha \norm{\K}_\infty}{n^d} + \epsilon b \, .
\end{align*}
For $n$ sufficiently large, $\normK{P_n}^2 \leq 2 b \epsilon$, thus $P_n$ converges to $0$ in the RKHS-norm. 
\end{proof}

\subsection{Statement and Proof of Lemma~\ref{lem:WandW}\label{proof:WandW}}

The proof of Theorem~\ref{theo:MetrizationOfNarrow2} is based on the following lemma.
\needspace{5\baselineskip}
\begin{lemma}\label{lem:WandW}
	Let $\K$ be a continuous, \gls{ispd} kernel. Let $c>0$ and  $\M^{\scriptscriptstyle{\leq c}}_{\scriptscriptstyle{+}}:= {\{ \mu \in \M_{\scriptscriptstyle{+}} \, | \, \mu(\inputS) \leq c \}}$. The following topologies coincide on $\M^{\scriptscriptstyle{\leq c}}_{\scriptscriptstyle{+}}$:
	\begin{enumerate}
		\item the weak topology $w(\M^{\scriptscriptstyle{\leq c}}_{\scriptscriptstyle{+}}, \HK)$ induced by $\K$, or topology of pointwise convergence in $\HK$~; \label{ww1}
		\item the vague topology $w(\M^{\scriptscriptstyle{\leq c}}_{\scriptscriptstyle{+}}, \Cont{}{c})$, or topology of pointwise convergence in $\Cont{}{c}$~. \label{ww2}
		\item the topology $w(\M^{\scriptscriptstyle{\leq c}}_{\scriptscriptstyle{+}}, \Cont{}{0})$, or topology of pointwise convergence in $\Cont{}{c}$~. \label{ww3}
	\end{enumerate}
\end{lemma}

\begin{remark}\label{rem:WandW}
	When $\K$ is continuous and $c_0$-universal, then the equality $w(\Mf, \Cont{}{0}) \cap \M^{\scriptscriptstyle{\leq c}}_{\scriptscriptstyle{+}} = w(\HK',\HK) \cap \M^{\scriptscriptstyle{\leq c}}_{\scriptscriptstyle{+}}$ from Lemma~\ref{lem:WandW} is a direct consequence of Theorem~\ref{theo:Q2}.
\end{remark}

\begin{proof}{[Lemma~\ref{lem:WandW}]}
	The topologies of \ref{ww2} and \ref{ww3} are equivalent because $\Cont{}{c}$ is dense in $\Cont{}{0}$ and $\M^{\scriptscriptstyle{\leq c}}_{\scriptscriptstyle{+}}$ is a bounded subset of their dual $\Mf$. We now show the equality between the topologies \ref{ww1} and \ref{ww2}.
	
	 The kernel $\K$ being continuous, its induced metric on $\M^{\scriptscriptstyle{\leq c}}_{\scriptscriptstyle{+}}$ is weaker than the narrow topology (Theorem~\ref{theo:Ws0}). Thus it suffices to show that $w(\Mf, \HK)$ is stronger than the narrow topology over $\M^{\scriptscriptstyle{\leq c}}_{\scriptscriptstyle{+}}$. As any filter is the intersection of its finer ultrafilters \citep[Chapter~I, §6, n\textdegree5, Proposition~7]{bourbakiTG}, it suffices to show that any ultrafilter $\mathfrak U$ on $\M^{\scriptscriptstyle{\leq c}}_{\scriptscriptstyle{+}}$ that converges to $\mu_0 \in \Mf$ in $w(\Mf, \HK)$ also converges to $\mu_0$ in $w(\Mf, \Cont{}{c})$. Thus, let $\mathfrak U$ be such an ultrafilter and let $K$ be a compact subset of $\inputS$. Because $\K$ is \gls{ispd}, $\HK$ is dense in $(\Cont{}{b})_c$ (Theorem~\ref{theo:Q2}). In particular, there exists $h \in \HK$ such that $h>0$ on $K$. For any $f \in \Cont{}{c}(\inputS)$ and $\mu \in \M^{\scriptscriptstyle{\leq c}}_{\scriptscriptstyle{+}}$, $|[h.\mu](f)| = |\mu(hf)|\leq c \norm{h}_\infty \norm{f}_\infty$. Thus $h.\M^{\scriptscriptstyle{\leq c}}_{\scriptscriptstyle{+}}$ is bounded in $\Mf$, thus relatively compact in $w(\Mf, \Cont{}{c})$. Thus the ultrafilter generated by $\mathfrak U_h := h . \mathfrak U$ converges in $w(\Mf, \Cont{}{c})$ to a measure $\mu_1 \in \M^{\scriptscriptstyle{\leq c}}_{\scriptscriptstyle{+}}$. We have now shown in particular that, for any compact $K \subset \inputS$,
	\begin{equation}\label{eq:VagueConvergence}
		\forall g \in \Cont{}{c}(K), \qquad \lim_{\mu \in \mathfrak U} \mu(gh) = \mu_1(gh) \, .
	\end{equation}
	But the map $\function{}{\Cont{}{c}(K)}{\Cont{}{c}(K)}{g}{gh}$ is bijective (because $h>0$ on $K$). Thus Equation~\eqref{eq:VagueConvergence} actually reads: for all $g \in \Cont{}{c}(K)$, $\lim_{\mu \in \mathfrak U} \mu(g) = \mu_1(g)$. Thus $\mathfrak U$ converges to $\mu_0$ in $w(\Mf,\Cont{}{c})$. But $\HK$ and $\Cont{}{c}$ are both dense in $(\Cont{}{b})_c$. Thus $\mu_0 = \mu_1$.
\end{proof}

\subsection{Alternative Proof for a Weaker Version of Theorem~\ref{theo:MetrizationOfNarrow2}\label{proof:MetrizationOfNarrow}}

In the proof of Theorem~\ref{theo:MetrizationOfNarrow2}, a key equality was $w(\Mf,\Cont{}{0}) \cap \Mprob = w(\HK',\HK) \cap \Mprob$, which was given by Lemma~\ref{lem:WandW}. As noted in Remark~\ref{rem:WandW}, this equality can also be obtained from Theorem~\ref{theo:Q2}, when $\HK \csubset \Cont{}{0}$. Here, we show a weaker version of Theorem~\ref{theo:MetrizationOfNarrow2} where we add the constrain $\HK \subset \Cont{}{0}$, so that its proof relies now only on Theorem~\ref{theo:Q2}.
\begin{proposition}\label{prop:MetrizationOfNarrow}
	Let $\inputS$ be a locally compact Hausdorff space and suppose that $\HK \subset \Cont{}{0}(\inputS)$. The kernel $\K$ metrises the narrow topology over $\Mprob$ \gls{iff} it is continuous and characteristic (to $\Mprob$).
	In particular, any $c_0$-universal kernel metrises the narrow topology over $\Mprob$.
\end{proposition}
\begin{proof}
	The only if part is clear. Conversely, if $\K$ is continuous and characteristic to $\Mprob$, then Proposition~\ref{prop:WeakerThanNarrow} shows that the metric induced by $\K$ is weaker than the narrow topology. And if, additionally to being characteristic, we suppose that $\K$ be even $c_0$-universal, then Theorem~\ref{theo:Q2}, applied to the bounded set $\mathcal B = \Mprob$ shows that the metric induced by $\K$ on $\Mprob$ is stronger than the weak topology $w(\Mf, \Cont{}{0}) \cap \Mprob$. As $\Mprob$ is bounded in $\Mf$, this weak topology coincides with the vague topology $w(\Mf, \Cont{}{c})\cap \Mprob$, which is known to coincide with the narrow topology \citep[Chapter~2, Corollary~4.3]{berg84}. Thus narrow and kernel induced topology coincide on $\Mprob$. The remaining of the proof now consists in treating the case where $\K$ is characteristic, but not $c_0$-universal. The idea is to use $\K$ to construct a new kernel $\K_\varphi$ that be $c_0$-universal, and such that, on $\Mprob$, the topologies induced by $\K$ and by $\K_\varphi$ coincide.

	Suppose now that $\HK \csubset \Cont{}{0}$, but without $\K$ being $c_0$-universal. Let $\nu \not = 0$ be a finite measure such that $\embK(\nu) = 0$. Let $\varphi \in \Cont{}{0}$ such that $\nu(\varphi) = 1$ and define the kernel $\K_\varphi(x,y) := \varphi(x) \bar \varphi(y) + \K(x,y)$. The kernel $\K_\varphi$ is continuous and $c_0$-universal. Thus, the first part of the proof shows that it metrises the vague topology $w(\Mf, \Cont{}{c}) \cap \Mprob$ over $\Mprob$.
	
	Let now $\Mf^\varphi$ and $\Mf^0$ be the hyperplanes of $\Mf$ defined respectively by $\mu(\varphi) = 0$ and $\mu(\One) = 0$.  On $\Mf^\varphi$, $\K$ and $\K_\varphi$ induce the same metric. Thus $\K$ metrises the vague topology $w(\Mf, \Cont{}{c}) \cap \Mf^\varphi$ over $\Mf^\varphi$. But $\Mf^\varphi$ and $\Mf^0$ are two non-identical hyperplanes in $\Mf$. Thus
	\begin{align*}
		\Mf^0 &= (\Mf^0 \cap \Mf^\varphi) \oplus \lin \nu_0 \\
		\Mf^\varphi &= (\Mf^0 \cap \Mf^\varphi) \oplus \lin \nu \, , 
	\end{align*}
	where $\nu_0 \in \Mf^0 \backslash \Mf^\varphi$ such that $\nu_0(\varphi) =1$.
	Thus $\Mf^0$ and $\Mf^\varphi$ are isomorphic, algebraically and topologically, via
	\begin{equation*}
		\function{\imath}{}{}{\mu}{ \left \{
			\begin{array}{ll}
				\mu \qquad &\mathrm{if} \ \mu \in \Mf^0 \cap \Mf^\varphi \\
				\nu \qquad &\mathrm{if} \ \mu = \nu_0
			\end{array}
			\right.} \, .
	\end{equation*}
	Let now $P, P_1, P_2, \ldots \in \Mprob$. We decompose $P_n$ (and $P$) on $(\Mf^0 \cap \Mf^\varphi) \oplus \lin \nu_0$: $P_n = P_n(\varphi) \nu_0 + p_n$. Suppose that $\normK{P_n - P} \rightarrow 0$. Then, on one side, both $P_n(\varphi) \rightarrow P(\varphi)$ and $\normK{p_n - p} \rightarrow 0$. On the other, $\normK{\imath(P_n - P)} \rightarrow 0$, which shows that $\imath(P_n)$ converges vaguely to $\imath(P)$. Said differently: both $p_n \rightarrow p$ and $P_n(\varphi) \nu_0 \rightarrow P(\varphi) \nu_0$ vaguely. Thus both $p_n \rightarrow p$ and $P_n(\varphi) \nu \rightarrow P(\varphi) \nu$ vaguely. Thus $P_n \rightarrow P$ vaguely. And being probability measures, they also converge narrowly, which concludes.	
\end{proof}
\begin{remark}
	Note that, following Theorem~\ref{theo:CKCharacterisation}, one is tempted to choose $\K_\varphi = \K + \One$. $\K_\varphi$ is then known to be characteristic to $\Mf$, and the metrics induced by $\K$ and $\K_\varphi$ would obviously coincide on $\Mprob$. However, this approach fails because, when $\HK \subset \Cont{}{0}$, then $\HH_{\K + \One} \not \subset \Cont{}{0}$.
\end{remark}

\subsection{Proof of Proposition~\ref{prop:TrivialEmbedding}\label{proof:TrivialEmbedding}}

\begin{proof}
	Indeed, if $\Cont{\infty}{c}$ embeds continuously and densely into $\HK$ via the canonical embedding $\function{\jmath}{\Cont{\infty}{c}}{\HK}{f}{f}$, then $\HK'$ embeds continuously and injectively into $\Dall{\infty}$ via the transpose $\jmath^\star$ of $\jmath$. Thus $\overline \HK' \equiv \HK$ embeds continuously and injectively into $\Dall{\infty}$ via the complex conjugate map of $\jmath^\star$: said differently, it identifies with a topological space of distributions. 
\end{proof}

\subsection{Proof of Lemma~\ref{lem:HKnotBigger}\label{proof:HKnotBigger}}

\begin{proof}
	Let $(x_n)_{n \in \nat}$ be a sequence in $\inputS$ that is not included in any compact. Consider then the sequence $f_n := \sum_k^n \frac{a_k}{2^{2k}}\frac{\K(.,x_k)}{\normK{\K(.,x_k)}}$ where $a_n=1$ if $n=0$ or if $f_{n-1}(x_n) \geq 0$, and $a_n=-1$ otherwise. Then $(f_n)_n$ obviously converges to a function $f$ in $\HK$, and $f$ has non-compact support, because 
	\begin{multline*}
		|f(x_n)| \geq |f_n(x_n)| - \left |\sum_{k >n} \frac{a_k}{2^{2k}}\frac{\K(.,x_k)}{\normK{\K(.,x_k)}} \right | \geq \\
			\geq \frac{1}{2^{2n}}\frac{\K(x_n,x_n)}{\normK{\K(.,x_n)}} - \sum_{k >n} \frac{\normK{\K(.,x_n)}}{2^{2k}} 
			= \frac{4}{3} \frac{\normK{\K(.,x_n)}}{2^{2n}} > 0 \, . 
	\end{multline*}
\end{proof}

\section{The Parametric Integral is not Enough\label{sec:Parametric}}

In this section, we give examples of kernels $\K$ and measures $\mu$ over $\inputS = \real^d$, such that the function $\function{}{}{}{y}{\int \K(y,x) \diff \mu(x)}$, although being well-defined for any $x \in \inputS$, is not contained in $\HK$.

First, remember that if $\K \in \Cont{}{0}$, then all functions in $\HK$ vanish at infinity. Let now $\K(\hat x, \hat y) = \psi(\hat x - \hat y)$ be any continuous stationary kernel over $\inputS = \real^d$ such that $\psi$ is Lebesgue-integrable ($\psi \in \Lp{1}$) and $\int \psi(x) \diff x \not = 0$. (Take for example a gaussian kernel.) Then, for any $y \in \real^d$, the integral $\int \K(y,x) \diff x$ is well-defined. But the function $\function{}{}{}{y}{\int \K(y,x) \diff x}$ is non-zero and constant. So it does not vanish at infinity and cannot be contained in $\HK$. Thus the Lebesgue measure does not embed into $\HK$, although it can parametrically integrate the kernel $\K$. This is why we used weak-integration, and not parametric integration.

\section{\texorpdfstring{\GLS{cpd}}{C.P.D.\ } Kernels and Characteristic Kernels to \texorpdfstring{$\Mprob$}{P}\label{sec:CPD}}

The two most relevant criteria we gave for characteristicness (Corollary~\ref{cor:Existence} and Theorem~\ref{theo:CompactCharacteristic}) apply only to stationary kernels, because their proofs rely on Bochner's theorem. But there is an equivalent of Bochner's theorem that holds for a larger class of stationary kernels, namely the stationary \emph{conditionally} positive definite kernels (see Definition~\ref{def:CPDKernels}). Using this generalised Bochner theorem (Theorem~\ref{theo:GeneralisedBochner}), we get an analogue of Theorem~\ref{theo:CompactCharacteristic} that holds for this more general class of kernels. In particular, these results show that the Brownian motion kernel $\K(\hat x, \hat y) = \min (|\hat x|, |\hat y|)$ is characteristic to the probability measures with compact support $\Mprob \cap \Mc$. We start with reminders on \gls{cpd} kernels and refer to the geostatistics literature \citep{matheron73, chiles12} or \citet{auffray09} for many more details.
\begin{definition} \label{def:CPDKernels}
	A \glsreset{cpd}\gls{cpd} kernel is a function ${\function{\Kc}{\inputS \times \inputS}{\complex}{}{}}$ such that, for any $n \in \nat$, $x_1, \ldots, x_n \in \inputS$, $\lambda_1, \ldots, \lambda_n \in \complex$,
	\begin{equation*}
		\sum_{i = 1}^n \lambda_i = 0 \quad \Rightarrow \quad \sum_{i,j =1}^n \lambda_i \overline{\lambda_j} \Kc(x_i, x_j) \geq 0 \, .
	\end{equation*}
\end{definition}

\begin{remark}\label{rem:EquivClassCPD}
	In the geostatistics literature, one likes to identify all \gls{cpd} kernels $\Kc$ that cannot be distinguished by the values $\sum_{i,j =1}^n \lambda_i \overline{\lambda_j} \Kc(x_i, x_j)$ when the $\lambda_i$ vary in the hyperplane defined by $\sum_{i = 1}^n \lambda_i = 0$. All such functions define an equivalence class of \gls{cpd} kernels. Note that for any \gls{cpd} kernel $\Kc$, and any function $\function{\psi}{\inputS}{\complex}{}{}$, the \gls{cpd} kernels $\Kc$ and $\Kc(\hat x, \hat y) + \psi(\hat x) + \psi(\hat y)$ are in the same equivalence class.
\end{remark}

Similarly to p.d.\ kernels, we may associate to each \gls{cpd} kernel $\Kc$ a map
\begin{equation*}
	\function{\embC}{\Mfs}{\Func}{\delta_x}{\Kc(.,x) = \int \Kc(.,s) \diff \delta_x(s)} \, .
\end{equation*}
This time however, we focus on 
\begin{align*}
	\Mfs^0 :&= \left \{ \sum_{i =1}^n \lambda_i \delta_{x_i} \, | \, n \in \nat, \, \lambda_i \in \complex, \, x_i \in \inputS, \, \sum_{i=1}^n \lambda_i = 0 \right \} \\
		&= \left \{ \mu \in \Mfs \, | \int \diff \mu = 0 \right \} = \Mfs \cap \Mf^0
\end{align*}
and the induced space of functions $\HCpre := \embC(\Mfs^0)$.  $\HCpre$ is equipped with the inner product defined, for any $f = \embC(\mu)$ and $g = \embC(\nu)$ in $\HCpre$ ($\mu = \sum_i \mu_i \delta_{x_i}, \nu = \sum_j \mu_j \delta_{y_j} \in \Mfs^0$), as\footnote{As for p.d.\ kernels, one easily verifies that if $f = \embC(\mu) = \embC(\lambda)$, with $\mu, \lambda \in \Mfs^0$, then the right hand side does not change if we replace $\mu$ by $\lambda$.}
\begin{equation*}
	\ipdC{f}{g} := \sum_{i,j} \mu_i \K(x_i,x_j) \overline{\nu_j} = \int \Kc(x,y) \diff \mu(x) \diff \bar{\nu}(y) \, ,
\end{equation*}
which we use to define the induced semi-inner product on $\Mfs^0$:
\begin{equation*}
	\forall \, \mu, \nu \in \Mfs^0, \qquad \ipdC{\mu}{\nu} := \ipdC{\embC(\mu)}{\embC(\nu)} \, .
\end{equation*}
It is a (proper) inner product over $\Mfs^0$ iff, for any $\mu \in \Mfs^0$:
\[
	\normC{\mu}^2 = \sum_{i,j = 1}^n \mu_i \K(x_i, x_j) \overline{\mu_j} = 0 \quad \Rightarrow \quad \mu = 0,
\]
that is \gls{iff} $\Kc$ is characteristic to $\Mfs^0$. In that case, $\Kc$ is said conditionally \emph{strictly} positive definite.

	Next we complete $\HCpre$ to a Hilbert space of functions $\HC$ and extend the embedding $\embC$ to larger sets of measures or distributions. We again focus only on sets $\D^0$ that are contained in the (closed) hyperplane of $\DLone{\infty}$ defined by the equation $\int \diff D = 0$: for example $\Mc^0$, $\Mf^0$, $(\DLone{m})^0$. We proceed as for the p.d.\ case: if $\HC \csubset \Cont{m}{}$ (resp.\ $\Cont{m}{0}$), then any distribution in $D \in (\Dcomp{m})^0$ (resp.\ $(\DLone{m})^0$) defines a continuous linear form over $\HC$. Thus it defines an element in the dual $(\HC)'$, whose Riesz representer $\int \Kc(.,x) \diff D(x)$ we call the KME of $D$ in $\HC$.

Obviously, each p.d.\ kernel is also \gls{cpd} Conversely, given any \gls{cpd} kernel $\Kc$, the function defined for a given $z_0 \in \inputS$ and any $x,y \in \inputS$ as $\K(x, y) := \ipdC{\delta_x - \delta_{z_0}}{\delta_y - \delta_{z_0}}$, is a p.d.\ kernel ``that coincides with $\Kc$ on $\Mfs^0$ '' (read: ``such that $\embC = \embK$ on $\Mfs^0$ ''). Thus any metric we induce in $\Mfs^0$ (or $\Mc^0$, $\Mf^0$, ...) using a \gls{cpd} kernel could also have been induced by a p.d.\ kernel. So why bother about \gls{cpd} kernels? Because there exist p.d.\ kernels $\K$ that are not stationary, but coincide on $\Mfs^0$ with a stationary \gls{cpd} kernel. 
The class of stationary \gls{cpd} kernels is thus \emph{strictly} bigger than the class of \gls{cpd} kernels that coincide with a stationary p.d.\ kernel on $\Mfs^0$. 

\begin{example}\label{ex:BM}
	Let $z_0 \in \inputS = \real$ and $\K_1(x,y) := e^{-|x-y|} - e^{-|x-z_0|} - e^{-|z_0 - y|} + 1$ and $\K_2(x,y) := \min(|x| , |y|)$ (Brownian motion (or Wiener) kernel). Both $\K_1$ and $\K_2$ are non stationary p.d.\ kernels, but on $\Mfs^0$, they coincide with the \gls{cpd} kernels $\K_{c,1}(x,y) = e^{-|x-y|}$ and $\K_{c,2}(x,y) = - |x-y|$ respectively. As it happens, $\K_{c,1}$ is p.d., but $\K_{c,2}$ is not. And one may prove \citep{matheron73,chiles12} that there exists no stationary p.d.\ kernel that coincides with $\K_2$ on $\Mfs^0$.
\end{example}

	For ease of presentation, we now restrict ourselves to real-valued measures and functions. The following theorem extends Bochner's theorem to \gls{cpd} kernels \citep[Theorem~2.1]{matheron73}: 
\begin{theorem}[Generalised Bochner Theorem]\label{theo:GeneralisedBochner}
	A continuous symmetrical function $\psi_c(\hat x - \hat y)$ over $\real^d \times \real^d$ is a real-valued stationary \gls{cpd} kernel \gls{iff} there exists an even \gls{cpd} polynomial $P_0$ of order $\leq 2$ and a positive (necessarily unique and symmetric) Radon measure $\chi$ without atom at the origin such that $\int \frac{\diff \chi (\xi)}{1 + |\xi|^2} < +\infty$ and, for any $h \in \real^d$,
	\begin{equation*}
		\psi_c(h) = \int \frac{\cos \ipd{h}{\xi}_{\real^d} - 1}{|\xi|^2} \diff \chi(\xi) + P_0(h) \, .
	\end{equation*}
\end{theorem}
Now remember: Corollary~\ref{cor:Existence} and Theorem~\ref{theo:CompactCharacteristic} characterised characteristic stationary p.d.\ kernels $\psi(\hat x, \hat y)$ using the support of $\Four \psi$, where $\Four \psi$ is precisely the positive measure from the usual Bochner theorem. It is only natural to try to extend these critera to stationary \gls{cpd} kernels, using measure $\chi$ from the generalised Bochner theorem instead of $\Four \psi$. 

	We will however only address characteristicness to $\Mc^0$, because $\Mf^0$ is actually already covered by Corollary~\ref{cor:Existence}. Indeed, if we do not impose that the \gls{cpd} kernel $\Kc$ be bounded, $\Mf^0$ need not embed into $\HK$. But if we do impose boundedness, then \citet{matheron73} showed that there exists a stationary p.d.\ kernel $\K$ that differs from $\Kc$ only by an additive constant, so induces the same norm than $\Kc$ in $\Mf^0$. And for stationary p.d.\ kernels, being characteristic to $\Mf^0$ is equivalent to being characteristic to $\Mf$. So to study whether $\Kc$ is characteristic to $\Mf^0$, it suffices to apply Corollary~\ref{cor:Existence} to $\K$.
\begin{proposition}\label{prop:CPDcharacteristic}
	Let $\Kc(\hat x, \hat y) = \psi_c(\hat x-\hat y)$ be a real-valued continuous stationary \gls{cpd} kernel over $\real^d \times \real^d$ in $\Cont{}{}$. 
	Let $\chi$ be the measure defined in Theorem~\ref{theo:GeneralisedBochner}. Then $\Mc^0$ embeds into $\HC$. If the support of $\chi$ has an accumulation point in $\real^d$, then $\Kc$ embeds and is characteristic to $\Mc^0$.
\end{proposition}
\begin{proof}
	For any $\mu \in \Mc$, $\int \normC{\Kc(.,x)} \diff |\mu|(x) < \infty$. Thus $\Kc$ is Bochner-integrable \gls{wrt} to any measure in $\Mc$, so $\Mc^0$ embeds into $\HC$. For any $\mu \in \Mc^0$ with $\mu \not = 0$,
\begin{align*}
	\normC{\mu}^2 &= \int \Kc(x,y) \diff \mu(x) \diff \mu(y) \\
		&= \int_{(x,y)} \left (\int_\xi \frac{\cos \ipd{x-y}{\xi}_{\real^d} - 1}{|\xi|^2} \diff \chi(\xi) + P_0(x-y) \right )\diff \mu(x) \diff \mu(y) \\
		&= \int_\xi \frac{|\Four\mu(\xi)|^2}{|\xi|^2} \diff \chi(\xi) + |C_\mu|^2 \, 
\end{align*}
where we noted $\Four \mu(\xi) := \int e^{-i\ipd{x}{\xi}_{\real^d}} \diff \mu(x)$ the Fourier transform of $\mu$ and $|C_\mu|^2 := \int_{(x,y)} P_0(x-y) \diff \mu(x) \diff \mu(y)$ (the right hand-side is positive because $P_0$ is \gls{cpd}). The switch of integral follows from Fubini's theorem, because $(\cos \ipd{x-y}{\xi}_{\real^d} - 1)/{|\xi|^2}$ is integrable \gls{wrt} $\chi \otimes |\mu| \otimes |\mu|$. As $\mu$ has compact support, its Fourier-transform is a holomorphic function (Paley-Wiener-Schwartz Theorem, \citealp[see][Theorem~29.2]{treves67}). Thus its zeroes $(z_n)_n$ have no accumulation point in $\real^d$.
	
	Suppose now that the support of $\chi$ has an accumulation point in $\real^d$. Let then $C$ be a compact neighbourhood of an accumulation such that $C$ contains no zero $z_n$ of $\Four \mu$.  As $\mu(\real^d) = 0$, $\Four \mu(0) = 0$. But $\Four \mu$ being holomorphic, the function $|\Four \mu(\hat \xi)|^2/|\hat \xi|^2$ is continuous, even in $0$. Thus
\begin{equation*}
	\norm{\mu}_{\Kc}^2 \geq \int_C \frac{|\Four \mu(\xi)|}{|\xi|^2} \diff \chi(\xi) \geq \mu(C) \min_{\xi \in C} \frac{|\Four \mu(\xi)|}{|\xi|^2} > 0 \, .
\end{equation*}
Thus any non-null measure in $\Mc^0$ has a strictly positive norm. So $\Kc$ is characteristic.
\end{proof}

Applying Proposition~\ref{prop:CPDcharacteristic} and Lemma~\ref{lem:CharacteristicAndM0} yields: \emph{the Brownian motion kernel (see Example~\ref{ex:BM}) is characteristic to $\Pc$\,}, the probability measures with compact support.\footnote{This result cannot be obtained directly from Corollary~\ref{cor:Existence}. However, for this particular kernel, it is actually quicker to simply notice that the RKHS $\HH_{BM}$ associated to a Brownian motion kernel consists of the functions $\varphi \in \Sobo{2,2}{0}$ such that $\varphi(0) = 0$. Thus the Brownian motion kernel is characteristic to $\Pc$.} And being continuous, it metrises the narrow convergence topology over $\Pc$ (Theorem~\ref{theo:Ws0}).

\section{Background Material\label{sec:Reminders}}

Let us start with the definition of a \emph{barrelled} set. In a normed space $(\mathcal E, \norm .)$, the sets $T := \{ f \in \mathcal E \, | \, \norm{f} \leq C \}$ where $C>0$ are the closed balls centred on the origin of $\mathcal E$. A normed space is a particular case of a \glsreset{lcv} topological vector space (TVS). In a general \gls{lcv} TVS $\mathcal E$, the topology might not be given by a single norm, but by a family of semi-norms $(\norm{.}_\alpha)_{\alpha \in \mathcal I}$ (where the index set $\mathcal{I}$ can be uncountable). A so-called \emph{barrel} of $\mathcal E$ is then any closed ball centred on the origin and associated to a norm $\norm{.}_\alpha$, $\alpha \in \mathcal I$. More abstractly, a barrel can be defined as follows.
\begin{definition}[Barrel]\label{def:Barrel}
	A subset $T$ of a TVS $\mathcal E$ is called a \emph{barrel} if it is
	\begin{enumerate}
		\item \emph{absorbing}: for any $f \in \mathcal E$, there exists $c_f > 0$ such that $f \in c_f T$;
		\item \emph{balanced}: for any $f \in \mathcal E$, if $f \in T$ then $\lambda f \in T$ for any $\lambda \in \complex$ with $|\lambda| \leq 1$ ;
		\item \emph{convex} ;
		\item \emph{closed}.
	\end{enumerate}
\end{definition}	

In any \gls{lcv} space, there exists a basis of neighbourhoods of the origin consisting only of barrels. However, in general, there may be barrels that are not a neighbourhood of $0$. This leads to

\begin{definition}[Barrelled spaces]\label{def:Barrelled}
	A TVS is \emph{barrelled} if any barrel is a neighbourhood of the origin.
\end{definition}

Although many authors include local convexity in the definition, in general, a barrelled space need not be \gls{lcv} Barrelled spaces were introduced by Bourbaki, because they were well-suited for the following generalisation of the celebrated \emph{Banach-Steinhaus} theorem.

\begin{theorem}[Banach-Steinhaus]\label{theo:BS}
	Let $\mathcal E$ be a barrelled TVS, $\mathcal F$ be a \gls{lcv} TVS, and let $L(\mathcal E, \mathcal F)$ be the set of continuous linear maps form $\mathcal E$ to $\mathcal F$. For any $H \subset L(\mathcal E, \mathcal F)$ the following properties are equivalent:
	\begin{enumerate}
		\item $H$ is equicontinuous.
		\item $H$ is bounded for the topology of pointwise convergence. \label{BS2}
		\item $H$ is bounded for the topology of bounded convergence. \label{BS3}
	\end{enumerate}
\end{theorem}
When $\mathcal E$ is a normed space and $\mathcal F = \complex$, then $L(\mathcal E, \mathcal F)$ is by definition $\mathcal E'$. With $\norm{.}_{\mathcal E'}$ being the dual norm in $\mathcal E'$, the equivalence of \ref{BS2} and \ref{BS3} states that
\begin{equation*}
	\forall f \in \mathcal E, \ \sup_{h \in H} |h(f)| < \infty \quad \Longrightarrow \quad \sup_{h \in H} \norm{h}_{\mathcal{E}'} < \infty \, .
\end{equation*}

Obviously, to understand the content of the Banach-Steinhaus theorem, one needs the definition of a bounded set. Let us define them now.

When $\mathcal{E}$ is a normed space, then a subset $B$ of $\mathcal E$ is called \emph{bounded} if ${\sup_{f \in B} \norm{f}_{\mathcal E} < \infty}$.
In a more general \gls{lcv} TVS $\mathcal E$, where the topology is given by a family of semi-norms $(\norm{.}_\alpha)_{\alpha \in \mathcal I}$, a subset $B$ of $\mathcal E$ is called \emph{bounded} if, for any $\alpha \in \mathcal I$, ${\sup_{f \in B} \norm{f}_{\alpha} < \infty}$. This can be shown equivalent to the following, more usual definition.
\begin{definition}[Bounded Sets in a TVS]
A subset $B$ of a TVS $\mathcal{E}$ is \emph{bounded}, if, for any open set $U \subset \mathcal{E}$, there exists a real $c_{\scriptscriptstyle B} > 0$ such that $B \subset c_{\scriptscriptstyle B} U$.
\end{definition}

We now move on to an unrelated topic: the Riesz Representation theorem for Hilbert spaces. Most of this paper relies on this one theorem.

\begin{theorem}[Riesz Representation Theorem for Hilbert Spaces] \label{theo:RieszRepresentationHilbert}
A Hilbert space $\HH$ and its topological dual $\HH'$ are isometrically (anti-) isomorphic via the Riesz representer map
\begin{equation*}
	\function{\imath}{\HH}{\HH'}{f}{D_f := \left \{ 
		\begin{array}{ccc}
			\HH &\longrightarrow &\complex \\
			g &\longmapsto &\ipd{g}{f}
		\end{array} \right .
		} \, .
\end{equation*}
In particular, for any continuous linear form $D \in \HH'$, there exists a unique element $f \in \HH$, called the \emph{Riesz representer of $D$}, such that 
\begin{equation*}
	\forall g \in \HH, \qquad D(g) = \ipd{g}{f} \, .
\end{equation*}
\end{theorem}
Note that ``anti'' in ``anti-isomorphic'' simply means that, instead of being linear, $\imath$ is anti-linear: for any $\lambda \in \complex$ and $f \in \HH$, $\imath(\lambda f) = \bar \lambda \, \imath(f)$. Often, we prefer to say that $\HH$ is isometrically isomorphic to $\overline{\HH}'$, where $\overline{\HH}'$ denotes the conjugate of $\HH$, where the scalar multiplication is replaced by $\function{}{}{}{(\lambda, f)}{\bar \lambda f}$. $\HK'$ and $\overline \HK'$ are obviously isomorphic via the complex conjugation map $\function{}{}{}{D}{\bar D}$.

The Riesz representation theorem for Hilbert spaces is not to be confounded with the following theorem, also known as the Riesz ---or Riesz-Markov-Kakutani--- representation theorem. In this paper, we always refer to the latter as the Riesz-Markov-Kakutani representation theorem. This theorem has numerous variants, depending on which dual pair $(\mathcal E, \mathcal E')$ one uses. Here we state it for $\mathcal E = \Cont{}{0}$.

\begin{theorem}[Riesz-Markov-Kakutani]\label{theo:RMK}
	Let $\inputS$ be a locally compact Hausdorff space. The spaces $\Mf(\inputS)$ and $(\Cont{}{0}(\inputS))'$ are isomorphic, both algebraically and topologically via the map
	\begin{equation*}
		\function{\imath}{\Mf(\inputS)}{(\Cont{}{0}(\inputS))'}{\mu}{D_\mu := \left \{ 
		\begin{array}{ccc}
			\Cont{}{0} &\longrightarrow &\complex \\
			\varphi &\longmapsto &\int \varphi \diff \mu
		\end{array} \right .
		} \, .
	\end{equation*}
\end{theorem}
In other words, for any continuous linear form $D$ over $\Cont{}{0}(\inputS)$, there exists a unique finite Borel measure $\mu \in \Mf$ such that, for any test function $\varphi \in \Cont{}{0}(\inputS)$, $D(\varphi) = \int \varphi \diff \mu$. Moreover, $\sup_{\norm{\varphi}_\infty \leq 1} D(\varphi) = |\mu|(\inputS)$, or in short: $\norm{D}_{(\Cont{}{0})'} = \norm{\mu}_{TV}$, where $\norm{\mu}_{TV}$ denotes the total variation norm of $\mu$. This is why, in this paper, we identify $\Mf$ ---a space of $\sigma$-additive set functions--- with $\Mf$ ---a space of linear functionals. 

In this paper, to embed a space of measures into an RKHS $\HK$ we successively apply both Riesz representation theorems: If $\HK$ embeds continuously into $\Cont{}{0}$, then $(\Cont{}{0})'$ embeds continuously into $\overline{\HK}'$, via the embedding map $\embK$. But $(\Cont{}{0})' = \Mf$ (Riesz-Markov-Kakutani Representation) and $\overline{\HK}' = \HK$ (Riesz Representation). Thus $\embK$ may also be seen as an embedding of $\Mf$ into $\HK$.


%
%

\section{Function Spaces and their Duals\label{sec:InclusionsDiagram}}

The following two diagrams depict how some function spaces (first diagram) and their corresponding duals (second diagram) embed one into another. The shaded line in the first diagram highlights those function spaces, for which we do not know any RKHS that they continuously and densely contain. Likewise, we do not know whether their duals can be continuously embedded via a KME. Lemma~\ref{lem:HKnotBigger} might hint, that such KMEs do not exist.

\begin{equation*}
	\newcommand{\col}{\cellcolor{gray!20}}
	\renewcommand{\arraystretch}{1.2}
	\begin{array}{p{.6em}cccccccc}
		&\multirow{4}{*}{\rotatebox[origin=c]{55}{$\xhookleftarrow{\quad \quad}$}}&\Lp{q} &\csupset &\Sobo{m,q}{0} &\csupset & \Sobo{\infty,q}{0} &\csupset & \HH_{\mathrm{gauss}}\\
		&&\cupset & &\cupset & &\cupset & &\\
		&&\col \Cont{}{c} &\col \csupset &\col \Cont{m}{c} &\col \csupset &\col \Cont{\infty}{c} & &\scalebox{1.3}{\vequal} \\
		&&\cdownset & &\cdownset & &\cdownset & &\\
		$\complex^\inputS$&&\Cont{}{0} &\csupset &\Cont{m}{0} &\csupset &\Cont{\infty}{0} & \csupset & \HH_{\mathrm{gauss}}\\
		&\multirow{4}{*}{\rotatebox[origin=c]{-55}{$\xhookleftarrow{\quad \quad}$}}&\cdownset & &\cdownset & &\cdownset & &\\
		&&(\Cont{}{b})_c &\csupset & (\Cont{m}{b})_c & \csupset & (\Cont{\infty}{b})_c & & \\
		&&\cdownset & &\cdownset & &\cdownset & &\\
		&&\Cont{}{} &\csupset &\Cont{m}{} & \csupset & \Cont{\infty}{} & &\\
	\end{array}
\end{equation*}

\begin{equation*}
	\newcommand{\col}{\cellcolor{gray!20}}
	\renewcommand{\arraystretch}{1.2}
	\begin{array}{p{.6em}cccccccc}
		&\multirow{4}{*}{\rotatebox[origin=c]{55}{$\xhookrightarrow{\quad \quad}$}}&\Lp{q'} &\csubset &\Sobo{-m,q'}{0} &\csubset &\Sobo{-\infty, q'}{0} &\csubset &\HH'_{\mathrm{gauss}}\\
		&&\cdownset & &\cdownset & &\cdownset & &\\
		&&\col \Mr & \col \csubset &\col \Dall{m} & \col \csubset &\col \Dall{\infty} & &\scalebox{1.3}{\vequal} \\
		& &\cupset & & \cupset & & \cupset & &\\
		$\Mfs$&&\Mf &\csubset &\DLone{m} & \csubset &\DLone{\infty} & \csubset & \HH_{\mathrm{gauss}}' \\
		&\multirow{4}{*}{\rotatebox[origin=c]{-55}{$\xhookrightarrow{\quad \quad}$}}&\cupset & & \cupset & & \cupset & &\\
		&&\Mf &\csubset &\DLone{m} & \csubset &\DLone{\infty} & & \\
		&&\cupset & & \cupset & & \cupset & &\\
		&&\Mc &\csubset &\Dcomp{m} & \csubset &\Dcomp{\infty} & & \\
	\end{array}
\end{equation*}

\bibliography{DistributionEmbeddings}

\end{document}